%% file: main.tex
\documentclass[10pt,journal,compsoc]{IEEEtran}

\ifCLASSOPTIONcompsoc
\usepackage[nocompress]{cite}
\else
\usepackage{cite}
\fi

\ifCLASSINFOpdf
\else
\fi
\usepackage[dvipsnames,table,xcdraw]{xcolor}
\usepackage[pagebackref=true,breaklinks=true,colorlinks,citecolor=ForestGreen]{hyperref}
\usepackage{ragged2e}

\usepackage{amssymb}
\usepackage{makecell}
\usepackage{multirow}
\usepackage{rotating}
\usepackage{array}
\usepackage{times}
\usepackage{graphicx}
\usepackage{psfrag}
\usepackage{subfigure}
\usepackage{enumitem}
\usepackage{url}
\usepackage{amsmath}
\usepackage{booktabs}
\usepackage{cuted}
\usepackage{capt-of}
\usepackage{lscape}
\usepackage{verbatim}
\usepackage{overpic}
\usepackage{bbding}

\usepackage{amsthm}

\usepackage{tablefootnote}

\usepackage{array}
\usepackage{makecell}

\usepackage{orcidlink} 

\usepackage{algpseudocode}
\usepackage[ruled,linesnumbered]{algorithm2e}
\SetKwComment{Comment}{/* }{ */}

\newtheorem{definition}{Definition}

\newtheorem{hypothesis}{Hypothesis}

\usepackage[misc]{ifsym}

\usepackage{enumitem}
\usepackage{soul}

\newcommand{\revision}[1]{\textcolor{black}{#1}}

\begin{document}

    \title{Less is More: Efficient Black-box Attribution via Minimal Interpretable Subset Selection}

    \author{Ruoyu Chen$^{\orcidlink{0000-0001-7630-7154}}$,~\IEEEmembership{Member,~IEEE},
        Siyuan Liang$^{\orcidlink{0000-0002-6154-0233}}$,
        Jingzhi Li$^{\orcidlink{0000-0001-7054-9267}}$,
        Shiming Liu$^{\orcidlink{0009-0000-7933-6591}}$,\\
        Li Liu$^{\orcidlink{0000-0002-2011-2873}}$,~\IEEEmembership{Senior Member,~IEEE},
        Hua Zhang$^{\orcidlink{0000-0002-7627-4142}}$,
        and Xiaochun Cao$^{\orcidlink{0000-0001-7141-708X}}$,~\IEEEmembership{Senior Member,~IEEE}
        \thanks{
            Ruoyu Chen and Hua Zhang are with the Institute of Information Engineering, Chinese Academy of Sciences, Beijing 100093, China, and also with the  University of Chinese Academy of Sciences, Beijing 100049, China (Email: \href{mailto:cryexplorer@gmail.com}{cryexplorer@gmail.com},  \href{mailto:zhanghua@iie.ac.cn}{zhanghua@iie.ac.cn}).\\
            Siyuan Liang is with the College of Computing and Data Science, Nanyang Technological University, 639798, Singapore (Email: \href{mailto:pandaliang521@gmail.com}{pandaliang521@gmail.com}).\\
            Jingzhi Li is with the Beijing Key Laboratory of Intelligence and Communication Integration, University of Science and Technology Beijing, Beijing 100083, China (Email: \href{mailto:lijingzhi@ustb.edu.cn}{lijingzhi@ustb.edu.cn}).\\
            Shiming Liu is with the Department of Mechanical Engineering, Imperial College London, UK (Email: \href{mailto:852074479@qq.com}{852074479@qq.com}). \\
            Li Liu is with the Center for Machine Vision and Signal Analysis (CMVS), University of Oulu, Finland (Email: \href{mailto:li.liu@oulu.fi}{li.liu@oulu.fi}).\\
            Xiaochun Cao is with the School of Cyber Science and Technology, Shenzhen Campus of Sun Yat-sen University, Shenzhen 518107, China (Email: \href{mailto:caoxiaochun@mail.sysu.edu.cn}{caoxiaochun@mail.sysu.edu.cn}).}
    }


    \IEEEtitleabstractindextext{%
        \begin{abstract}
            \justifying
            To develop a trustworthy AI system, it is essential to understand its behavior using attribution methods, which aim to identify the input regions that most influence the model’s decisions.
            \revision{For discrete inputs such as images, however, accurately identifying input--prediction interactions is particularly challenging because of combinatorial explosion.}
            \revision{We formulate black-box attribution as an ordered interpretable subset-selection problem. Inspired by the marginal-gain principle of submodular optimization, the formulation permits redundancy, positive complementarity, and negative interference among input sub-regions. We propose \textsc{LiMA} (\emph{Less input is More faithful for Attribution}), which ranks sub-regions by their contributions along a sequence of selected subsets. \textsc{LiMA} employs a composite objective that evaluates subsets through semantic consistency, collaboration, prediction confidence, and semantic effectiveness. We further design a bidirectional greedy search to efficiently identify important and non-critical sub-regions.}
            \revision{Experiments on eight foundation models (CLIP, ImageBind, LanguageBind, QuiltNet, \textit{etc.}) and six datasets (ImageNet, VGG-Sound, CUB-200-2011, \textit{etc.}) demonstrate strong generalization and faithful interpretations with fewer sub-regions. \textsc{LiMA} achieves state-of-the-art Insertion, Deletion, and average highest confidence, with average improvements of 36.3\% in Insertion and 39.6\% in Deletion; it is also 1.6 times faster than naive greedy search.}
            \revision{For failure attribution on incorrect predictions, \textsc{LiMA} improves average highest confidence by 86.1\% over state-of-the-art methods.}
            The code is available at \url{https://github.com/RuoyuChen10/LIMA}.
        \end{abstract}

        }

    \maketitle

    \IEEEdisplaynontitleabstractindextext

    %
    \IEEEpeerreviewmaketitle

    \input{revision_sec/1-introduction}

    \input{revision_sec/2-related_work}
    \input{revision_sec/3-preliminaries}

    \input{revision_sec/4-method}
    \input{revision_sec/5-method_analysis}

    \input{revision_sec/6-experiments}

    \section{Conclusion} \label{conclusion}

    \revision{In this paper, we introduce \textsc{LiMA}, a novel and efficient black-box attribution method based on ordered interpretable subset selection. Inspired by the marginal-gain-based search principle of submodular optimization, LiMA uses a composite objective to accommodate both redundant and complementary evidence among input sub-regions, without assuming that the objective is globally submodular or monotone. The proposed bidirectional greedy search reduces the inference cost of ranking important and non-critical sub-regions. We validate the effectiveness of \textsc{LiMA} across multiple datasets and multimodal models, achieving state-of-the-art faithfulness. The results show that the method provides faithful explanations for correctly predicted samples and useful counterfactual evidence for analyzing incorrect predictions.}
    

    \ifCLASSOPTIONcaptionsoff
    \newpage
    \fi



    %
    \bibliographystyle{IEEEtran}      
    \footnotesize
    \bibliography{egbib}

    %

    \begin{IEEEbiography}[{\includegraphics[width=0.8in]{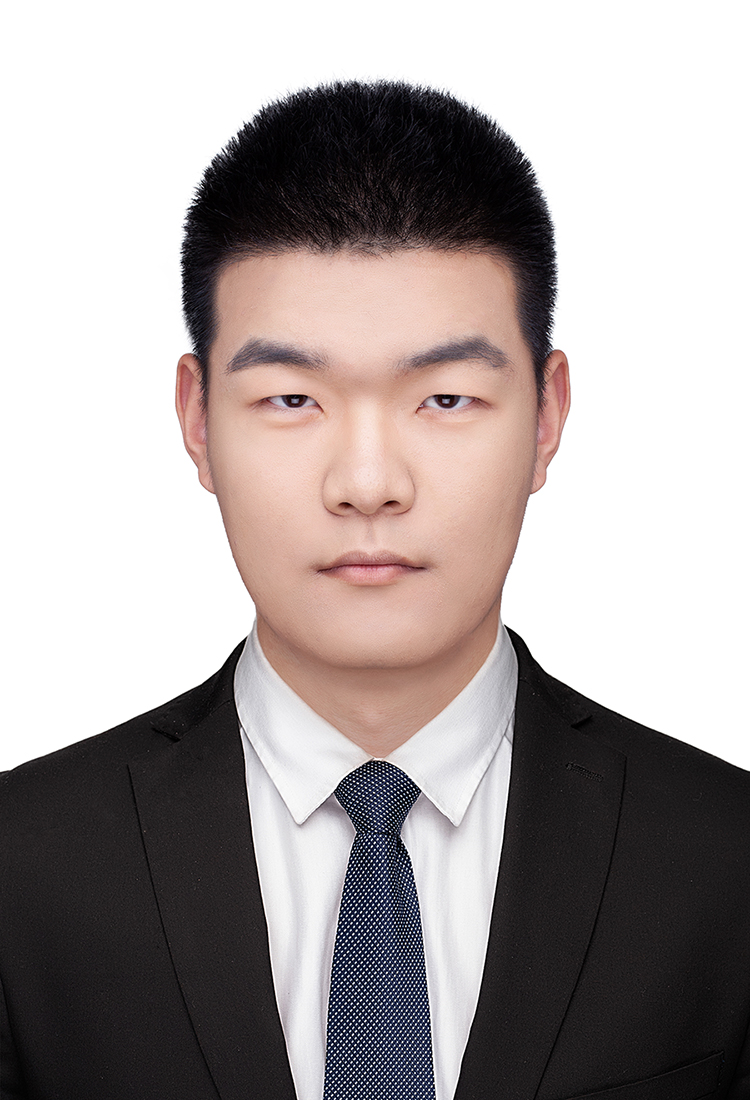}}]{Ruoyu Chen}
        received the Ph.D. degree in computer applied technology from the University of Chinese Academy of Sciences, China, in 2026, and the B.E. degree from Northeastern University, China, in 2021. He has published papers in leading journals and conferences, including IEEE T-PAMI, CVPR Highlight, and ICLR Oral. He has served as a reviewer for IEEE T-PAMI, IEEE TSP, IEEE/ACM TASLP, IEEE TBIOM, ACM TIST, journals in the npj series, CVPR, ICCV, ECCV, ICML, ICLR, NeurIPS, and other venues. His research interests include computer vision, machine learning, explainable AI, and multimodal foundation models.
    \end{IEEEbiography}

    \vspace{-20pt}

    \begin{IEEEbiography}[{\includegraphics[width=0.8in]{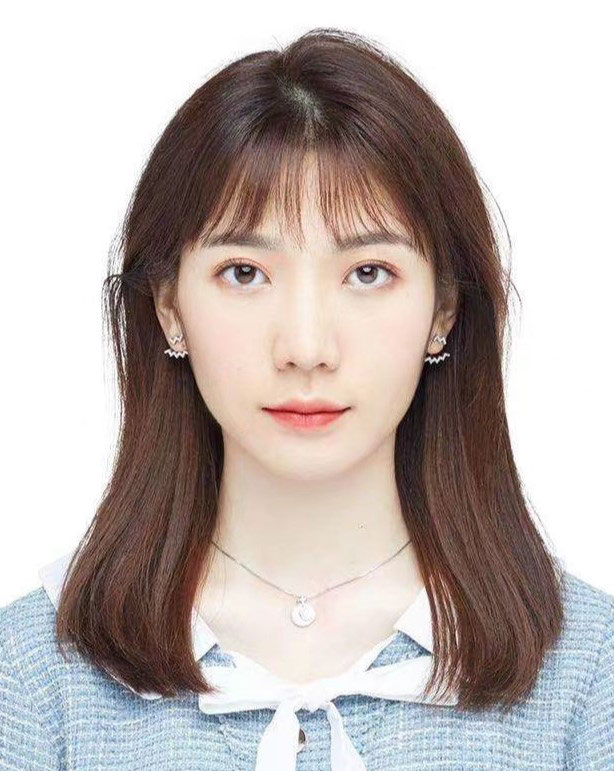}}]{Siyuan Liang} is a Research Fellow at the School of Computing, National University of Singapore, specializing in adversarial machine learning and computer vision. She obtained her Ph.D. in computer applied technology from the University of the Chinese Academy of Sciences and a Bachelor's degree in Software Engineering from Sichuan University. Her work primarily focuses on developing robust AI models to enhance computer vision security.
    \end{IEEEbiography}

    \vspace{-20pt}

    \begin{IEEEbiography}[{\includegraphics[width=0.8in]{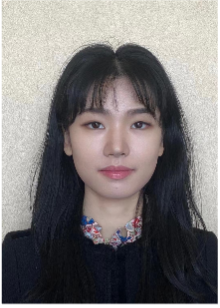}}]{Jingzhi Li} is currently an associate professor with the Beijing Key Laboratory of Intelligence and Communication Integration, University of Science and Technology Beijing. From 2023 to 2025, she was an associate professor with the Institute of Information Engineering, Chinese Academy of Sciences, Beijing, China. She received the Ph.D. degree in cyberspace security from the University of Chinese Academy of Sciences, Beijing, China. Her current research interests include image processing, face recognition security, and multimedia privacy.
    \end{IEEEbiography}

    \vspace{-20pt}
    \begin{IEEEbiography}[{\includegraphics[width=0.8in]{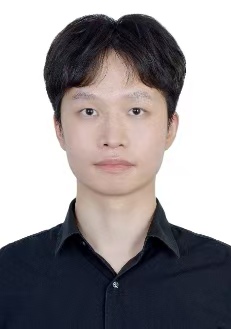}}]{Shiming Liu}
         received the Ph.D. degree in intelligent manufacturing from the Department of Mechanical Engineering, Imperial College London, UK, in 2022. His current research interests include advanced safety-guard for autonomous driving system (ADS), SOTIF, AI interpretability and AI failure detection.
    \end{IEEEbiography}

    \vspace{-20pt}

    \begin{IEEEbiography}[{\includegraphics[width=0.8in]{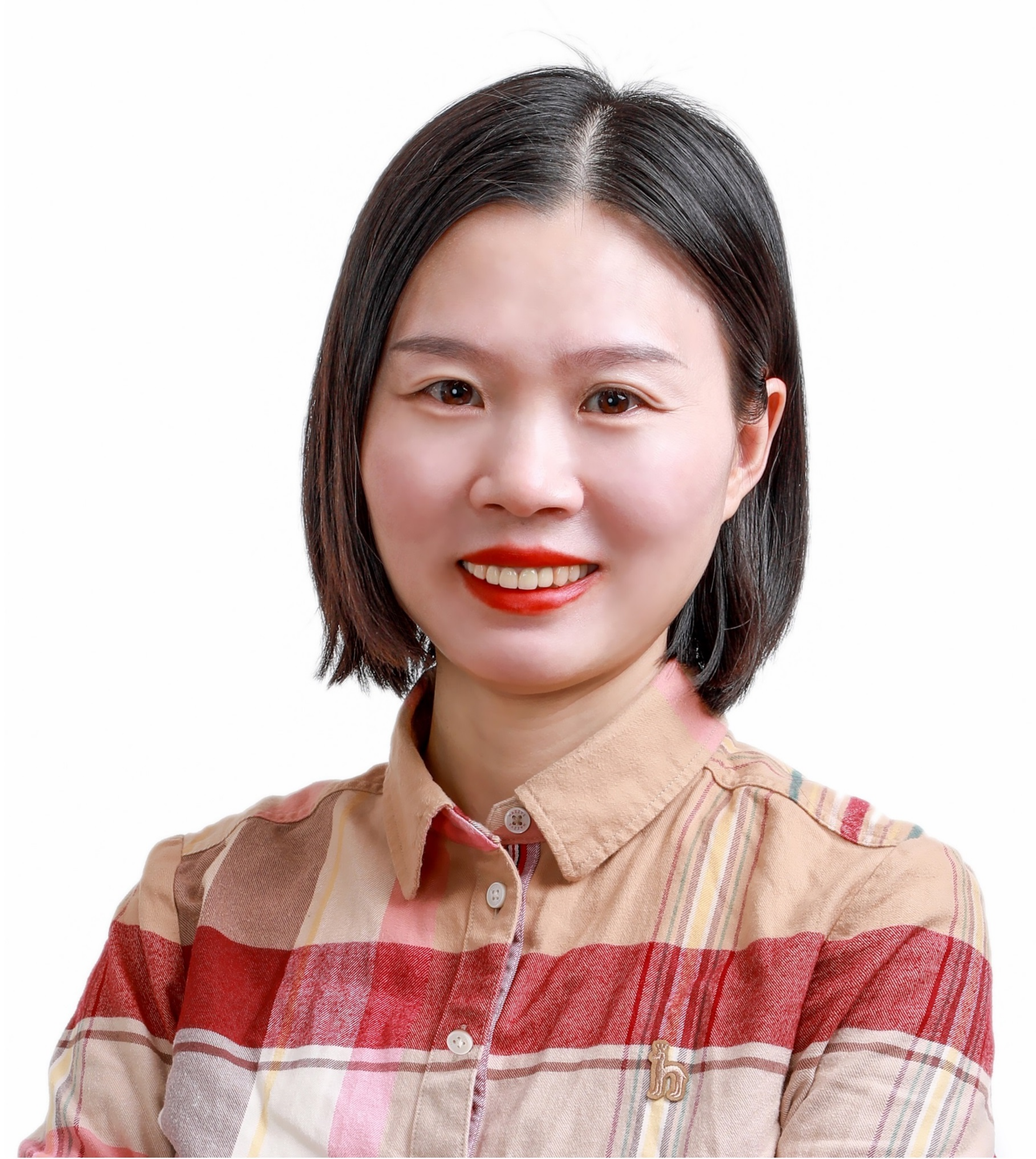}}]{Li Liu}(Senior Member, IEEE) 
        spent more than two years as a Visiting Student at the University of Waterloo, Canada, from 2008 to 2010. From 2015 to 2016, she spent ten months visiting the Multimedia Laboratory at the Chinese University of Hong Kong. From 2016.12 to 2018.11, she worked as a senior researcher at the Machine Vision Group at the University of Oulu, Finland. Her current research interests include Computer Vision, Machine Learning, Artificial Intelligence, Trustworthy AI, and Synthetic Aperture Radar. Her papers have currently over 25,000+ citations in Google Scholar.
    \end{IEEEbiography}

    \vspace{-20pt}

    \begin{IEEEbiography}[{\includegraphics[width=0.8in]{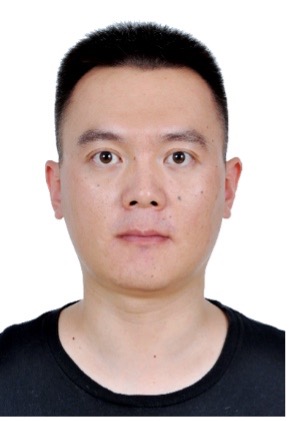}}]{Hua Zhang} is a Professor with the Institute of Information Engineering, Chinese Academy of Sciences. He received the Ph.D. degree in computer science from the School of Computer Science and Technology, Tianjin University, Tianjin, China in 2015. His research interests include computer vision, multimedia, and machine learning.
    \end{IEEEbiography}

    \vspace{-20pt}

    \begin{IEEEbiography}[{\includegraphics[width=0.8in]{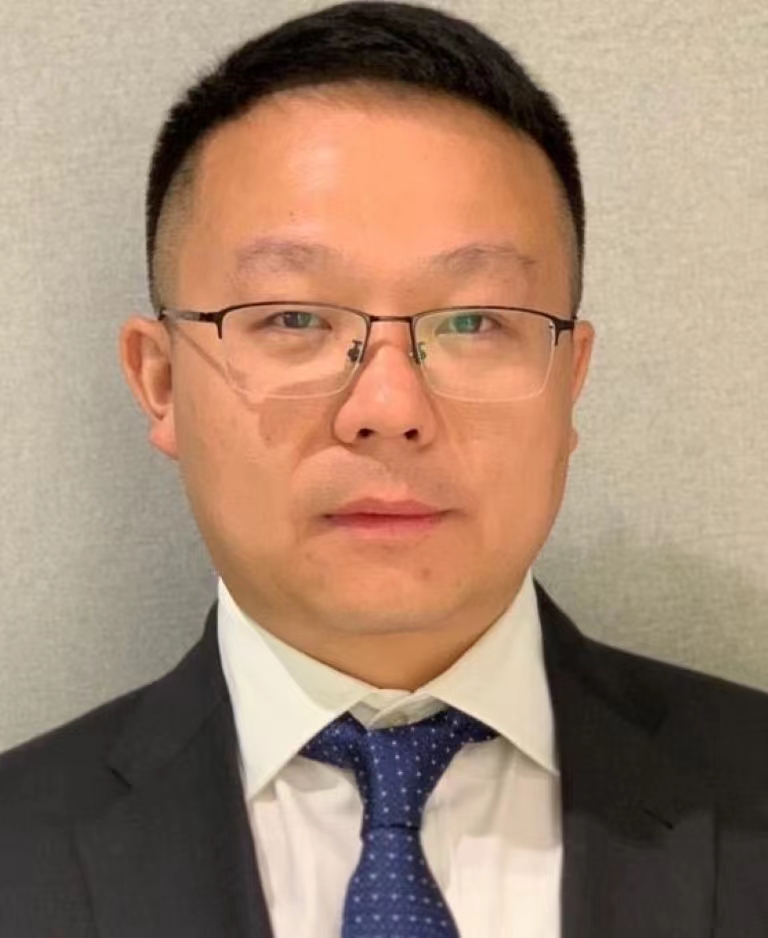}}]{Xiaochun Cao}(Senior Member, IEEE)
        is a Professor and Dean of the School of Cyber Science and Technology, Shenzhen Campus of Sun Yat-sen University. He received the B.E. and M.E. degrees both in computer science from Beihang University (BUAA), China, and the Ph.D. degree in computer science from the University of Central Florida, USA, with his dissertation nominated for the university level Outstanding Dissertation Award. After graduation, he spent about three years at ObjectVideo Inc. as a Research Scientist. From 2008 to 2012, he was a professor at Tianjin University. Before joining SYSU, he was a professor at the Institute of Information Engineering, Chinese Academy of Sciences. He has authored and co-authored over 200 journal and conference papers. In 2004 and 2010, he was the recipients of the Piero Zamperoni best student paper award at the International Conference on Pattern Recognition. He is on the editorial boards of IEEE \textsc{Transactions on Pattern Analysis and Machine Intelligence} and IEEE \textsc{Transactions on Image Processing}, and was on the editorial boards of IEEE \textsc{Transactions on Circuits and Systems for Video Technology} and IEEE \textsc{Transactions on Multimedia}.
    \end{IEEEbiography}


    \clearpage
    \appendices
    \input{revision_sec/7-appendix}
    \clearpage

\end{document}

%% file: revision_sec/1-introduction.tex
\section{Introduction}\label{sec:introduction}

\IEEEPARstart{E}{xplainable AI} (XAI) techniques have gained significant attention for enabling the transparent and reliable deployment of trustworthy AI systems in real-world applications~\cite{troncoso2025new,rong2024towards,patricio2023explainable,arrieta2020explainable,chen2023sim2word,gao2024going,chen2026not,jiao2026consistent}, particularly in high-stakes domains such as healthcare~\cite{patricio2023explainable} and autonomous driving~\cite{chen2024end,yang2026can}. The primary objective of XAI is to enhance our understanding of intelligent models, particularly by uncovering the relationships between predictions and input data~\cite{chen2024less,novello2022making,deng2024unifying,chen2024interpreting,chen2026mllms}. To elucidate these associations, attribution-based methods~\cite{nam2025illuminating,luo2024local,mandler2024review,dwivedi2023explainable} have been developed to interpret black-box deep neural networks by identifying the contribution of each input feature to the model's predictions. \revision{In attribution evaluation, faithfulness refers to the extent to which the importance assigned to input regions is consistent with their actual influence on the model output. Accordingly, a faithful attribution should rank influential regions highly, such that progressively inserting these regions rapidly restores the corresponding model response, whereas removing them causes a pronounced decrease in that response~\cite{novello2022making,petsiuk2018rise}.} However, faithful attribution for black-box AI systems remains an open problem.

Several advanced attribution methods have been developed, including propagation-based~\cite{bach2015pixel}, gradient-based~\cite{sundararajan2017axiomatic, selvaraju2020grad, khorram2021igos++, zhao2024gradient}, Shapley value-based~\cite{lundberg2017unified, sun2023explain}, and perturbation-based~\cite{petsiuk2018rise, novello2022making} mechanisms. The primary challenge in achieving perfect fidelity in attribution methods lies in accurately capturing the interactive relationships between inputs and predictions~\cite{chen2024defining, liu2023towards}. While existing attribution methods have yielded notable results, they still face several limitations that restrict their applicability in specific scenarios: 1) Computational efficiency, evaluating input-prediction interactions is a combinatorial explosion problem, making it challenging to achieve both efficient and faithful attribution. 2) Inaccurate contribution estimation, many attribution methods use sampling-based approximation strategies to estimate input contributions, which can lead to noisy attribution maps. 3) Region redundancy, when certain sub-regions are misattributed, either contributing less to, or even negatively affecting, the model’s response to the correct category. This issue becomes more pronounced in the case of erroneous predictions.

\begin{figure*}[!t]
    \centering
    \includegraphics[width = \textwidth]{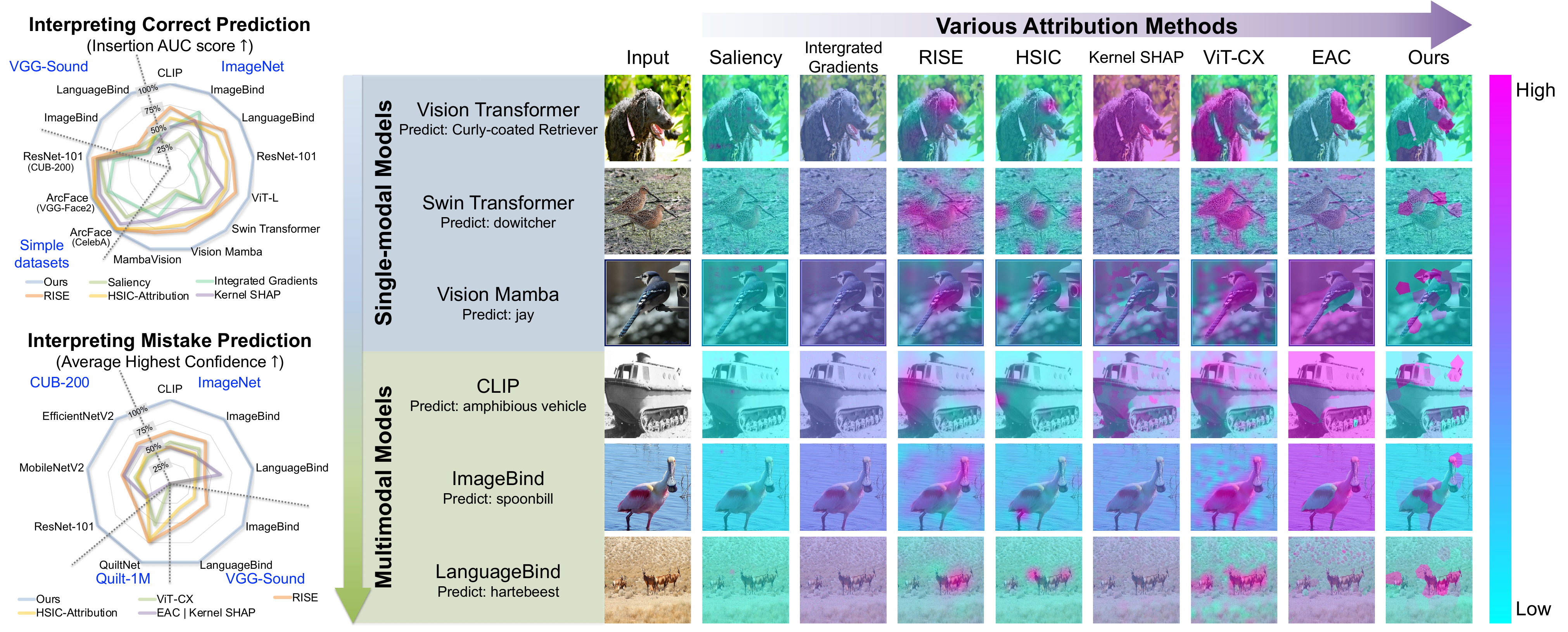}
    \caption{
    The left panel illustrates the Insertion and Average Highest Confidence metrics for various attribution mechanisms when attributing the model’s correct and incorrect predictions. Our method shows significant improvements across different datasets and models. The right panel shows the attribution maps of different methods, where our approach avoids noise and unnecessary region redundancy.
    }
    \label{motivation}
\end{figure*}

The core of the attribution problem is to \revision{determine the importance of individual sub-regions and identify the subset of sub-regions} that has the most significant impact on the final decision.
\revision{Inspired by submodular optimization~\cite{fujishige2005submodular}, which evaluates candidate elements through their marginal gains to a selected subset, we seek to enhance model interpretability by selecting and ordering fewer sub-regions. However, attribution objectives induced by modern neural networks need not satisfy the classical diminishing-returns assumption. Redundant sub-regions may exhibit overlapping utility, whereas complementary sub-regions may produce positive synergy, such that their joint response is greater than the sum of their individual responses. Background or distractor sub-regions may also produce negative marginal gains. We therefore formulate attribution as an ordered subset-selection problem and use submodular optimization as a design inspiration for marginal-gain-based search, without assuming that the resulting attribution objective is globally submodular or monotone. Fig.~\ref{motivation} shows the attribution results of several popular methods across different models and images, highlighting issues such as background noise and region redundancy.}

\revision{In this paper, we propose \textsc{LiMA} (\emph{Less input is More faithful for Attribution}), a novel and efficient black-box attribution mechanism based on ordered interpretable subset selection. LiMA first sparsifies an input into a limited set of fine-grained sub-regions using a semantic- or patch-based division method. It then uses a composite objective function to evaluate candidate subsets from four complementary perspectives: semantic consistency, collective effect, prediction confidence, and semantic effectiveness. A bidirectional greedy search algorithm simultaneously identifies important and non-critical sub-regions, producing an ordered sequence whose prefixes provide explanations at different region budgets.}

Extensive experiments on eight foundation models (CLIP~\cite{radford2021learning}, ImageBind~\cite{girdhar2023imagebind}, LanguageBind~\cite{zhu2024languagebind}, QuiltNet~\cite{ikezogwo2024quilt}, \textit{etc.}) and six datasets (ImageNet~\cite{deng2009imagenet}, VGG-Sound~\cite{chen2020vggsound}, CUB-200-2011~\cite{welindercaltech}, Celeb-A~\cite{liu2015deep}, VGG-Face2~\cite{cao2018vggface2}, and LC25000~\cite{borkowski2019lung}) demonstrate that our method achieves state-of-the-art attribution faithfulness evaluation metrics (Insertion, Deletion, and average highest confidence) with fewer regions and exhibits strong generalization, shows an average improvement of 36.3\% in Insertion and 39.6\% in Deletion. As shown in the radar chart in Fig.~\ref{motivation}, compared to the most advanced attribution methods, our approach shows an average improvement of 36.3\% in Insertion. \revision{For failure attribution on incorrect predictions, the average highest confidence achieved by our method is, on average, 86.1\% higher than that of state-of-the-art attribution algorithms.} Furthermore, our attribution efficiency is 1.6 times greater than that of the naive greedy search.

Our contributions can be summarized as follows:
\begin{itemize}
    \item \revision{We formulate black-box attribution as an ordered interpretable subset-selection problem inspired by the marginal-gain-based search principle of submodular optimization, while allowing both complementary and interfering interactions among sub-regions.}
    \item \revision{We introduce a composite objective function to evaluate subset importance and a bidirectional greedy search algorithm that balances attribution efficiency and empirical faithfulness.}
    \item Our analysis shows that as the pre-training scale and model parameters increase, or when models make incorrect predictions, the interaction between input elements becomes more complex, which makes accurate attribution more challenging.
    \item \revision{The proposed method demonstrates strong faithfulness across various models and datasets. It provides faithful attribution for correctly predicted samples and model-level counterfactual evidence for analyzing incorrect predictions.}
\end{itemize}

A preliminary version of this work~\cite{chen2024less} has been accepted for oral presentation at ICLR 2024. The current version presents several significant advancements over the previous one:
\begin{itemize}
    \item \textbf{Method optimization:} (1) We eliminate the reliance on a prior saliency map by adopting super-pixel segmentation or the Segment Anything approach. This produces semantically richer sub-regions and enhances attribution quality. (2) We introduce a new strategy for assessing the importance scores of ranked sub-regions, which enhances visualization and enables more accurate quantification of their importance. (3) We propose a bidirectional greedy search method to optimize attribution efficiency, which is particularly beneficial for models with a large number of parameters. 
    \item \textbf{Analytically:} (4) We observe that as the model’s pre-training scale and the number of parameters increase—or when the model produces incorrect predictions—the interaction between input \revision{sub-regions} becomes more complex, reducing the effectiveness of traditional attribution methods. (5) We analyze this problem from an interaction perspective and conclude that our proposed mechanism is better suited for such challenging scenarios. A series of experiments are conducted to validate the effectiveness of our approach. 
    \item \textbf{Exhaustive experiments:} (6) We evaluate our attribution approach on models with multiple modalities and architectures, including several multimodal foundation models, as well as ViT and Mamba architectures. The results demonstrate that our approach generalizes well across both multimodal models and various model architectures. (7) We further validate the effectiveness of our approach on multiple datasets, including the large-scale image dataset ImageNet, the medical image dataset LC25000, and the audio dataset VGG-Sound. These results illustrate the versatility of our interpretable attribution method across different data modalities.
\end{itemize}

\revision{The rest of the paper is organized as follows. In Section~\ref{related_work}, we review related work. Section~\ref{preliminaries} introduces ordered subset selection and its connection to submodular optimization. Section~\ref{method} presents the proposed \textsc{LiMA} method, including its objective design and query-complexity analysis. Section~\ref{method_analysis} analyzes the context-dependent response changes recorded by the method. Section~\ref{experiment} presents the experimental and visualization results, and Section~\ref{conclusion} concludes the paper.}

%% file: revision_sec/2-related_work.tex
\section{Related Work}\label{related_work}

\subsection{Attribution Methods}


\textbf{Inner propagation, activation, and gradient-based methods} analyze the internal responses of the network to identify the most important regions of the input. These methods are often referred to as white-box attribution methods because they require access to the model’s internals. Some methods attribute importance by propagating scores back through the layers until reaching the input layer~\cite{bach2015pixel}. 
Other methods rely solely on decision gradients for attribution~\cite{simonyan2014deep}, but these approaches evaluate only the individual effects of pixels or features, without accounting for their collective impact~\cite{chendefining}. 
Some methods combine network activations and gradients for attribution, including Grad-CAM~\cite{selvaraju2020grad}, Grad-CAM++~\cite{chattopadhay2018grad}, Score-CAM~\cite{wang2020score}, ViT-CX~\cite{xie2023vit}, and Grad-ECLIP~\cite{zhao2024gradient}. However, the effectiveness of these methods depends heavily on the selection of network layers, which significantly influences the quality of interpretation~\cite{novello2022making}. Moreover, these methods are largely heuristic and lack theoretical guarantees. Path-based methods~\cite{sundararajan2017axiomatic,khorram2021igos++} achieve attribution by selecting a specific integral path. However, this approach is highly sensitive to parameters such as baseline and path choice. Additionally, since the gradient must be integrated, applying this method to large models may be limited by computational resources.

\textbf{Shapley value estimation-based methods} attribute model predictions by estimating the Shapley value~\cite{shapley1953value} of each input region or feature. This is accomplished by calculating the marginal contribution of each region or feature and combining these contributions linearly. However, calculating the Shapley value is an exponential complexity problem. To address this, some model-agnostic methods estimate SHAP values by sampling subsets of features. For example, Kernel SHAP~\cite{lundberg2017unified} determines feature importance through sampling and the application of specially weighted linear regression. EAC~\cite{sun2023explain} estimates Shapley values by sampling sub-regions defined by the Segment Anything model and using a linear surrogate model. In addition, model-specific methods have been developed to reduce redundant subsets by leveraging unique model structures. HarsanyiNet~\cite{chen2023harsanyinet}, for instance, employs a specialized network based on a CNN architecture to accurately estimate Shapley values in a single propagation, though its scalability is limited. 
Despite these innovations, the exact computation of Shapley values remains generally impractical due to the exponential complexity associated with increasing data dimensions
, making it nearly impossible to apply these methods to large models. Additionally, Kumar \textit{et al.}~\cite{kumar2020problems} observed that feature importance evaluation methods based on the Shapley value estimation may risk underestimating the importance of relevant features, resulting in all such features being assigned lower importance values.

\textbf{Peturbation-based methods} operate under the assumption that the internal parameters of the model are unknowable. They assess the importance of input regions by perturbing the input and observing the resulting changes in the output. LIME~\cite{ribeiro2016should} locally approximates the predictions of a black-box model with a linear model, by only slightly perturbing the input. RISE~\cite{petsiuk2018rise} perturbs the model by inputting multiple random masks and weighting the masks to get the final saliency map. HSIC-Attribution method~\cite{novello2022making} measures the dependence between the input image and the model output through Hilbert Schmidt Independence Criterion (HSIC) based on RISE. 
Although these methods are highly portable and applicable to various network architectures, their attribution performance is affected by the grid size used for perturbation~\cite{novello2022making}. If the grid size is too large, the attribution performance may decline, resulting in insufficient granularity in the attribution regions~\cite{chen2024less}.


\textbf{Search-based methods} attribute the importance of different regions by ranking the sub-regions in order. Shitole \textit{et al.}~\cite{shitole2021one} defined minimal sufficient explanation, which involves identifying a limited region whose confidence response to the model is at least 90\% of the response of the entire image. They utilized a heuristic beam search to find these regions. \revision{Accordingly, their objective is to identify a compact set of regions that is sufficient to preserve at least 90\% of the model confidence obtained from the entire image.} However, this method is only effective when the model exhibits a high prediction confidence, and searching for sub-regions based solely on changes in confidence is often unstable~\cite{chen2024less}. \revision{The objective of MoXI~\cite{sumiyasu2024identifying} is to identify a group of pixels that yields high prediction confidence by jointly modeling the individual contributions of pixels and their cooperative effects using Shapley values and Shapley interactions.} Chen \textit{et al.}~\cite{chen2024less} addressed the challenge of searching for high-confidence regions by modeling the image attribution problem as a submodular subset selection problem. However, the performance of this method is influenced by the quality of a prior saliency map.

\textbf{Multimodal Attribution:} Traditional explanation methods mainly focus on single-modal models like DNNs and CNNs, with emerging approaches for multi-modal models. Some studies have investigated attention mechanisms in multimodal ViT models for attributing input or observation influence on outputs~\cite{chefer2021transformer}, but these methods often lack transferability as they rely on access to internal parameters. Darcet \textit{et al.}~\cite{darcet2024vision} leveraged attention maps in ViT as guidance and addressed artifacts by incorporating additional tokens. However, attention maps merely reflect intermediate response strengths without explicitly capturing input-output relationships. Zhao \textit{et al.}~\cite{zhao2024gradient} attributed the CLIP model using gradients and feature maps, achieving better explanation results than self-attention. Gandelsman \textit{et al.}~\cite{gandelsman2024interpreting} approximated CLIP token attribution via MLP and multi-head attention decomposition, but the results were limited by omitting layer normalization.



\revision{In this paper, we propose an ordered subset-selection approach inspired by the marginal-gain principle of submodular optimization. By sparsifying the input, we refine it into fine-grained sub-regions and identify important evidence through a context-dependent composite objective. Unlike classical submodular formulations, our objective permits complementary interactions and negative marginal contributions, which commonly arise in neural attribution.}

\subsection{Error Explanation} Identifying and explaining a model’s errors allows us to understand its flaws better and address them more efficiently. 
Wu \textit{et al.}~\cite{wu2023discover} identified unstable concepts as false attributes and intervened in these concepts to correct model decisions. However, this process requires precise concept annotations, which is cumbersome. Abid \textit{et al.}~\cite{abid2022meaningfully} evaluated the impact of specific concepts on misclassified images using concept activation vectors, thereby explaining model prediction errors from a conceptual perspective. However, concerns remain about whether traditional unimodal models can truly learn explicit concepts and accurately assess their importance~\cite{ramaswamy2023overlooked}. 
\revision{In this paper, our method analyzes an individual misclassified sample at the input level by identifying a compact region whose revelation reverses the model prediction to the ground-truth class. This intervention provides model-level counterfactual evidence associated with prediction reversal.}

\subsection{Submodular Optimization} Submodular optimization~\cite{fujishige2005submodular} has been successfully studied in multiple application scenarios~\cite{kothawade2022talisman}, for example, He \textit{et al.}~\cite{he2024efficient} combined submodular subset selection with a loss function and Shapley value to evaluate and select modality importance in multimodal learning. A small number of studies have also applied this theory to explore model interpretability. Catav \textit{et al.}~\cite{catav2021marginal} utilized the maximum marginal contribution from all possible subsets to explain the data’s contribution. Elenberg \textit{et al.}~\cite{elenberg2017streaming} frame the interpretability of black-box classifiers as a combinatorial maximization problem, it achieves similar results to LIME and is more efficient. Chen \textit{et al.}~\cite{chen2018learning} propose a learnable network for instance-wise feature selection to explain deep models, but its scalability is limited. Pervez \textit{et al.}~\cite{pervez2022scalable} proposed a simple subset sampling alternative based on conditional Poisson sampling, which they applied to interpret both image and text recognition tasks. 
However, these methods only retained the selected important pixels and observed the recognition accuracy~\cite{chen2018learning}. Chhabra \textit{et al.}~\cite{chhabra2024data} interpreted the feature space through submodular subset selection to determine which data is helpful to improve model performance. 
\revision{In this paper, we propose an attribution method based on ordered interpretable subset selection, achieving state-of-the-art empirical performance according to standard attribution metrics. The method is evaluated across diverse datasets, modalities, and network architectures, including analyses of both correct and incorrect predictions.}


\subsection{Greedy Search}

The greedy search algorithm~\cite{nemhauser1978analysis}, the most commonly used approach, optimizes submodular functions by iterating through unselected examples and calculating the gain in function value from adding each to the selected set. Some accelerated greedy search algorithms have been proposed. Mirzasoleiman \textit{et al.}~\cite{mirzasoleiman2015lazier} proposed a stochastic-greedy algorithm, which randomly samples a portion of the unselected examples and performs greedy selection on that subset. Joseph \textit{et al.}~\cite{joseph2019submodular} partition the samples into equal-sized, non-overlapping sets, perform greedy selection within each, then merge and further select elements. 
Note that most of the above acceleration algorithms focus on subset selection rather than the order of samples. Consequently, they are unsuitable for attribution tasks requiring ordered subset search~\cite{chen2024less}. 
\revision{In this paper, we propose a bidirectional greedy search strategy to reduce the practical candidate-evaluation overhead and measured runtime of attribution for large models while identifying both important and non-critical sub-regions. During the forward search, the method generates candidates with low marginal objective values and uses them to construct the reverse sequence. Combining the two sequences produces a complete ordered sequence.}

%% file: revision_sec/3-preliminaries.tex
\section{Preliminaries}\label{preliminaries}

\revision{We first introduce the marginal-gain principle that motivates our ordered subset-selection formulation. Let $V$ be a finite set of sub-regions and let $\mathcal{F}:2^{V}\rightarrow\mathbb{R}$ be a set-dependent objective. The marginal gain of a candidate sub-region $\alpha\in V\setminus S$ is $\Delta_{\mathcal{F}}(\alpha\mid S)=\mathcal{F}(S\cup\{\alpha\})-\mathcal{F}(S)$. Classical submodular optimization formalizes a diminishing-returns relation between such marginal gains:}

\begin{definition}[Submodular function~\cite{edmonds1970submodular}]
    \revision{For any $S_a\subseteq S_b\subseteq V$ and $\alpha\in V\setminus S_b$, a set function $\mathcal{F}$ is submodular if}
    \begin{equation}\label{eq:sub}
        \mathcal{F}\left( S_{a} \cup \{\alpha\} \right) - \mathcal{F}\left( S_{a} \right) \ge \mathcal{F}\left(S_{b} \cup \{ \alpha \}\right) - \mathcal{F}\left(S_{b} \right).
    \end{equation}
\end{definition}

\revision{Monotonicity is a separate property requiring $\Delta_{\mathcal{F}}(\alpha\mid S)\ge 0$. We do not impose either global submodularity or monotonicity on the proposed attribution objective. Neural model responses can contain redundant evidence that resembles diminishing returns, complementary evidence that yields increasing marginal gains, and distractors that yield negative marginal gains. LiMA therefore borrows the marginal-gain-based selection principle from submodular optimization, while optimizing a general, potentially non-submodular and non-monotone objective.}

\begin{figure*} 
    \centering
    \begin{overpic}[width=\textwidth,tics=8]{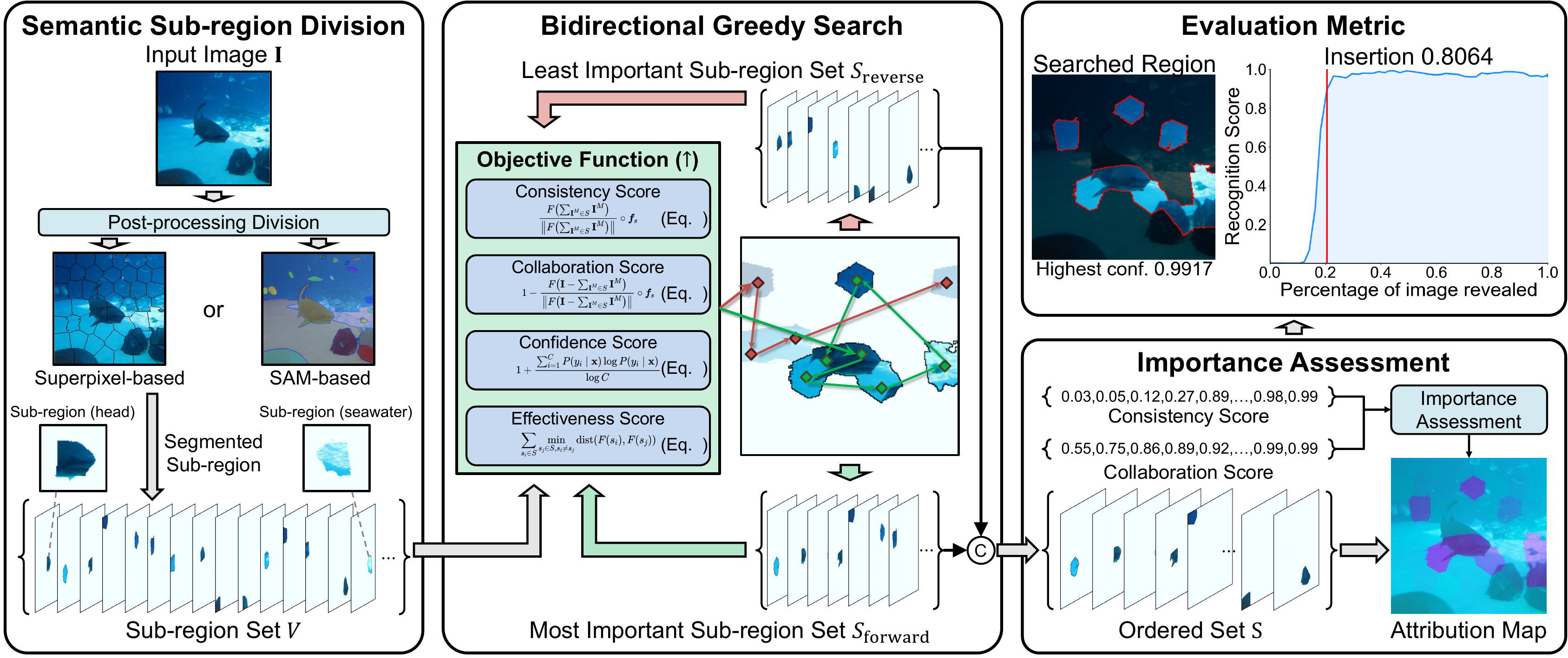}
        \put (44.2,27.55) {\scriptsize \ref{consistency_function}}
        \put (44.2,22.8) {\scriptsize \ref{collaboration_function}}
        \put (44.2,18.0) {\scriptsize \ref{confidence_function}}
        \put (44.2,13.1) {\scriptsize \ref{r_function}}
    \end{overpic}  
    \caption{\revision{The framework of the proposed \textsc{LiMA} method. We first divide the input into interpretable sub-regions using a superpixel- or semantic-based strategy. A composite objective and bidirectional greedy search then identify important and non-critical sub-regions and arrange them into an ordered sequence. Finally, consistency- and collaboration-based marginal changes are used to assign attribution scores. The resulting prefixes provide explanations at different region budgets and are evaluated using faithfulness metrics.}}  
    \label{framework}
\end{figure*}

\noindent\textbf{Problem formulation:} We divide an input data $\mathbf{I}$ into a finite number of sub-regions with a division algorithm $\texttt{Div}(\cdot)$, denoted as $V = \texttt{Div}(\mathbf{I}) = \left\{ \mathbf{I}^{M}_{1}, \mathbf{I}^{M}_{2}, \cdots, \mathbf{I}^{M}_{\revision{N}} \right\}$, where \revision{$N=|V|$ denotes the number of divided sub-regions, and} $M$ indicates a sub-region $\mathbf{I}^{M}$ formed by masking part of $\mathbf{I}$. 
\revision{Throughout this paper, a \emph{sub-region} denotes an atomic interpretable unit generated by $\operatorname{Div}(\cdot)$ and contained in the set $V$, whereas a \emph{region} refers to a spatial area formed by one or more sub-regions. The term \emph{element} is reserved for generic set-theoretic definitions and proofs.}
\revision{Attribution seeks an ordering $\boldsymbol{\pi}=(\pi_1,\ldots,\pi_N)$ of the sub-regions. Its prefix $S_t=\{\pi_1,\ldots,\pi_t\}$ represents the explanation available under a budget of $t$ sub-regions. Given the composite objective $\mathcal{F}$, a forward marginal-gain order is constructed by}
{
\begin{equation}\label{eq:object}
\begin{aligned}
    \pi_t &\in \arg\max_{\alpha\in V\setminus S_{t-1}}
    \Delta_{\mathcal{F}}(\alpha\mid S_{t-1}),\\
    S_t &= S_{t-1}\cup\{\pi_t\}.
\end{aligned}
\end{equation}
}
\revision{A standard forward greedy ranking performs $\mathcal{O}(kN)$ candidate-objective evaluations for a prefix of length $k$. Specifically, it performs $kN-\frac{1}{2}k(k-1)$ candidate evaluations, and ranking all $N$ sub-regions performs $\frac{1}{2}N^2+\frac{1}{2}N$ evaluations, i.e., $\mathcal{O}(N^2)$. 
}

%% file: revision_sec/4-method.tex
\section{Proposed Method}\label{method}

\revision{In this section, we present an ordered subset-selection method for black-box attribution.}
Section \ref{subsec_subregion_division} provides a detailed explanation of the sub-region division process. In Section \ref{submodular_function_design}, we present our custom-designed \revision{composite objective function}. Section \ref{alg_fast} outlines the attribution algorithm, which utilizes a bidirectional greedy search approach, while Section \ref{importance_assess} describes the procedure for assigning importance scores to the ordered set. Fig. \ref{framework} illustrates the overall framework of our approach.

\subsection{Sub-region Division}\label{subsec_subregion_division}

To obtain the interpretable region in an image, we partition the image $\mathbf{I} \in \mathbb{R}^{w \times h \times 3}$ into $m$ sub-regions $\mathbf{I}^M$ with a division algorithm $\texttt{Div}(\cdot)$, where $M$ indicates a sub-region $\mathbf{I}^{M}$ formed by masking part of the image $\mathbf{I}$. The division algorithm determines the quality of the search space, thereby influencing the effectiveness of the attribution~\cite{chen2024less}. Our recommended division strategies $\texttt{Div}(\cdot)$ are as follows:
\begin{itemize}
    \item \textbf{Superpixel-based division} involves clustering pixels with similar characteristics. For image modality, methods like SLICO~\cite{achanta2012slic} can be employed to perform this division, thereby improving attribution efficiency.
    \item \textbf{Semantic-based division} can more effectively distinguish between different concepts and enhance human understanding. Segment Anything (SAM)~\cite{kirillov2023segment} can be used for zero-shot division. However, since the regions segmented by SAM may overlap, we propose a SAM-based division method. Details can be found in Appendix~\ref{SAM_based_division}, Algorithm~\ref{alg:sam_division}.
    \item \textbf{Patch-based division} has been studied in traditional methods~\cite{noroozi2016unsupervised,petsiuk2018rise}, which can divide the input into regular patch regions.
\end{itemize}

We can select the appropriate division method for attribution based on the specific input modality.

\subsection{\revision{Composite Objective Function Design}}\label{submodular_function_design}

In this section, \revision{we construct an objective function to evaluate candidate subsets and determine an importance order over interpretable sub-regions. The objective combines four complementary criteria: consistency, collaboration, confidence, and effectiveness. These criteria capture different aspects of subset quality.}


\textbf{Consistency Score:} 
We aim to make the representation of the identified image region consistent with the original semantics. Given a target semantic feature or mapping function, $\boldsymbol{f}_{s}$, we make the semantic features of the searched image region close to the target semantic features. We introduce the consistency score:
\begin{equation}\label{consistency_function}
    s_{\mathrm{cons.}} \left ( S, \boldsymbol{f}_{s} \right ) = \frac{ F\left ( \sum_{\mathbf{I}^{M} \in S}\mathbf{I}^{M} \right )} {\left \| F \left ( \sum_{\mathbf{I}^{M} \in S}\mathbf{I}^{M} \right ) \right \| } \circ \boldsymbol{f}_{s},
\end{equation}
where $F(\cdot)$ denotes a pre-trained feature extractor. The target semantic feature $\boldsymbol{f}_{s}$, can either adopt the features computed from the original image using the pre-trained feature extractor, expressed as $\boldsymbol{f}_{s} = F \left ( \mathbf{I} \right )$, or directly implement the fully connected layer of the classifier for a specified class.
By incorporating the $s_{\mathrm{cons.}}$, our method targets regions that reinforce the desired semantic response. This approach ensures a precise selection that aligns closely with our specific semantic goals.

\textbf{Collaboration Score:} Some individual \revision{sub-regions} may lack significant individual effects in isolation, but when placed within the context of a group or system, they exhibit an indispensable collective effect. Therefore, we introduce the collaboration score, which is defined as:
\begin{equation}\label{collaboration_function}
    s_{\mathrm{colla.}} \left ( S, \mathbf{I}, \boldsymbol{f}_{s} \right ) = 1 - \frac{ F \left (\mathbf{I} - \sum_{\mathbf{I}^{M} \in S}\mathbf{I}^{M} \right )}{\left \| F \left (\mathbf{I} - \sum_{\mathbf{I}^{M} \in S}\mathbf{I}^{M} \right ) \right \|} \circ \boldsymbol{f}_{s},
\end{equation}
where $F(\cdot)$ denotes a pre-trained feature extractor, $\boldsymbol{f}_{s}$ is the target semantic feature. By introducing the collaboration score, we can judge the collective effect of the \revision{sub-region}. By incorporating the $s_{\mathrm{colla.}}$, our method pinpoints regions whose exclusion markedly affects the model's predictive confidence. This effect underscores the pivotal role of these regions, indicative of a significant collective impact. Such a metric is particularly valuable in the initial stages of the search, highlighting regions essential for sustaining the model's accuracy and reliability.

\textbf{Confidence Score:} The model’s confidence in its own predictions is also closely related to interpretability~\cite{guo2020learn} and the entropy of the model output is a key indicator for measuring the uncertainty of the model’s predictions~\cite{shu2022test,guo2023handling}. Therefore, we use the entropy of the network output to calculate the confidence score. 
In the inference, the predictive uncertainty can be calculated as $u = -\frac{\sum_{i=1}^{C}P(y_i \mid \mathbf{x})\log{P(y_i \mid \mathbf{x})}}{\log{C}}$, where $u \in [0,1]$, $P(y_i \mid \mathbf{x})$ is the probability of a specific class output, $C$ is the total number of classes. Thus, the confidence score of the sample $\mathbf{x}$ predicted by the network can be expressed as:
\begin{equation}\label{confidence_function}
    s_{\mathrm{conf.}} \left ( \mathbf{x} \right ) = 1 - u = 1 + \frac{\sum_{i=1}^{C}P(y_i \mid \mathbf{x})\log{P(y_i \mid \mathbf{x})}}{\log{C}}.
\end{equation}

By incorporating the $s_{\mathrm{conf.}}$, we can ensure that the selected regions align closely with the In-Distribution (InD). This score acts as a reliable metric to distinguish regions from out-of-distribution, ensuring alignment with the InD.

\textbf{Effectiveness Score:} Adding different sub-regions may exhibit OR interactions, meaning they have similar marginal contribution scores~\cite{chen2024defining}. We expect to maximize the response of valuable information with fewer regions since some image regions have the same semantic representation. Given \revision{a candidate sub-region} $\alpha$, and a subset $S$, we measure the distance between the \revision{sub-region} $\alpha$ and all \revision{candidate sub-regions} in the set, and calculate the smallest distance as the effectiveness score of the judgment \revision{candidate sub-region} $\alpha$ for the subset $S$:
\begin{equation}\label{r_function_single}
    s_{e} \left ( \alpha \mid S \right ) = \min_{s_i \in S}{\mathrm{dist}\left( 
 F\left(\alpha\right),   F\left(s_i\right)\right)},
\end{equation}
where $\mathrm{dist}(\cdot, \cdot)$ denotes the equation to calculate the distance between two \revision{candidate sub-regions}. 
\revision{When $S=\varnothing$, we define $s_e(\alpha\mid S)=0$, since there is no selected sub-region against which to measure redundancy.} Traditional distance measurement methods~\cite{wang2018cosface, deng2019arcface} are tailored to maximize the decision margins between classes during model training, involving operations like feature scaling and increasing angle margins. In contrast, our method focuses solely on calculating the relative distance between features, for which we utilize the general cosine distance.
$F(\cdot)$ denotes a pre-trained feature extractor. To calculate the effectiveness score of a set, we can compute the sum of the effectiveness scores for each \revision{candidate sub-region}:
\begin{equation}\label{r_function}
    s_{\mathrm{eff.}} \left ( S \right ) = \sum_{s_{i} \in S } \min_{s_j \in S, s_i \ne s_j}{\mathrm{dist}\left(  F\left(s_i\right), F\left(s_j\right)\right)}.
\end{equation}

By maximizing the $s_{\mathrm{eff.}}$, we aim to limit the selection of regions with similar semantic representations, thereby increasing the diversity and improving the overall quality of region selection.

\revision{\textbf{Composite Objective Function:} We combine the above scores into the following objective for selecting sub-regions:}
\begin{equation}\label{submodular_function}
    \begin{aligned}
        \mathcal{F}(S) = & \; \lambda_{1} s_{\mathrm{cons.}}\left ( S, \boldsymbol{f}_{s} \right ) + \lambda_{2} s_{\mathrm{colla.}}\left ( S, \mathbf{I}, \boldsymbol{f}_{s} \right ) \\
        &  + \lambda_{3} s_{\mathrm{conf.}} \left ( \sum_{\mathbf{I}^{M} \in S}\mathbf{I}^{M} \right ) + \lambda_{4} s_{\mathrm{eff.}} \left ( S \right ),
    \end{aligned}
\end{equation}
where $\lambda_{1}, \lambda_{2}, \lambda_{3}$, and $\lambda_{4}$ represent the weighting factors used to balance each score. According to the importance and experience of the scores, we set these parameters to $\lambda_{1}=20$, $\lambda_{2}=5$, $\lambda_{3}=0.05$, and $\lambda_{4}=0.01$, respectively.

\revision{The raw scales of the four criteria are not inherently aligned. Consistency and collaboration inherit the scale of representation similarity, confidence is based on normalized predictive entropy, and effectiveness aggregates representation distances and therefore also varies with the number of selected sub-regions. The coefficients consequently provide empirical scale and trade-off control across these complementary criteria.}
\revision{The coefficients $\lambda_1,\lambda_2,\lambda_3$, and $\lambda_4$ balance terms with different numerical scales and practical roles. We use $(\lambda_1,\lambda_2,\lambda_3,\lambda_4)=(20,5,0.05,0.01)$ as an empirically selected default and keep it fixed throughout the main comparative evaluations.}

\revision{The four criteria play complementary roles. Consistency rewards subsets that preserve the target semantics, whereas collaboration rewards subsets whose removal disrupts those semantics and thus captures removal-based necessity. Confidence discourages highly uncertain predictions on masked inputs, and effectiveness promotes diversity among selected sub-regions by penalizing semantic redundancy. Accordingly, the coefficients balance the different numerical scales and practical contributions of these criteria.}

\revision{\textbf{Relation to Submodular Optimization:} The construction of ordered subsets through marginal gains is inspired by submodular optimization. Nevertheless, Eq.~\ref{submodular_function} is a general model-dependent objective rather than a submodular function. Redundant sub-regions can exhibit overlapping utility and diminishing marginal effects, but complementary evidence can produce positive synergy---informally, a ``$1+1>2$'' response---and distractors can produce negative marginal gains. Consequently, neither global diminishing returns nor monotonicity is assumed. The bidirectional greedy procedure below is therefore used as a practical ranking strategy, and its effectiveness is evaluated empirically rather than justified by a submodular approximation guarantee. The quality of the sub-region division still determines the available search space and influences attribution performance.}

\subsection{Bidirectional Greedy Search Algorithm}\label{alg_fast}

\revision{Given $V=\{\mathbf{I}^{M}_{1},\mathbf{I}^{M}_{2},\ldots,\mathbf{I}^{M}_{N}\}$, we construct an ordered sequence by repeatedly evaluating Eq.~\ref{submodular_function} and selecting the candidate with the largest marginal objective value for the current prefix. This greedy construction is motivated by marginal-gain-based subset selection~\cite{mirzasoleiman2015lazier,wei2015submodularity}, but it is applied here to the general objective in Eq.~\ref{submodular_function}. A forward ranking must traverse all remaining sub-regions in $V\setminus S$ at every iteration, which leads to a high inference cost for large black-box models. Randomly sampling candidates may miss important evidence and is therefore undesirable when the goal is to rank all sub-regions.}

\revision{We propose a bidirectional greedy search algorithm that simultaneously identifies important and non-critical sub-regions, as illustrated in Algorithm~\ref{alg:bidirectional_greedy}. The forward sequence $S_{\mathrm{forward}}$ incrementally collects candidates with high marginal objective values, whereas the reverse sequence $S_{\mathrm{reverse}}$ accumulates candidates with low values. At each forward iteration, the algorithm evaluates all sub-regions in $V\setminus(\operatorname{set}(S_{\mathrm{forward}})\cup\operatorname{set}(S_{\mathrm{reverse}}))$ and selects the largest-gain candidate. The same evaluations also produce a pool $S_p$ containing at most $n_p$ low-value candidates. The reverse step selects from this pool according to}
{
\begin{equation}
\min_{\alpha \in S_p}{\mathcal{F} \left( \{ \alpha \} \cup \operatorname{set}(S_{\mathrm{reverse}}) \right) }.
\end{equation}
}

\revision{This reverse selection is a ranking heuristic: low objective values indicate sub-regions that are non-critical or potentially counterproductive in the current context. The candidate-pool size $n_p$ controls the empirical trade-off between ranking faithfulness and inference cost. Algorithm~\ref{alg:bidirectional_greedy} retains the $\mathcal{O}(N^2)$ asymptotic complexity of standard greedy ranking. It reuses the forward-sweep scores to construct a low-value candidate pool and restricts each reverse search to at most $n_p$ candidates, rather than performing a second exhaustive sweep over all remaining sub-regions. For the default $N=49$ and $n_p=8$, the pseudocode performs 805 candidate-objective evaluations, compared with 1,225 for a complete standard forward ranking. Under the evaluated multimodal objective, the corresponding query upper bounds are 1,660 and 2,500, respectively. Thus, the method improves practical inference efficiency without changing the asymptotic complexity order. We report measured wall-clock times in Section~\ref{experiment}.}

\begin{algorithm}[]
    \scriptsize
    \caption{Bidirectional greedy ranking for interpretable region discovery}\label{alg:bidirectional_greedy}
    \KwIn{Image $\mathbf{I} \in \mathbb{R}^{w \times h \times 3}$, a division algorithm $\texttt{Div}(\cdot)$, reverse candidate-pool size $n_p$.}
    \KwOut{An ordered sequence $\boldsymbol{\pi}$ containing all sub-regions in $V$.}
    $V \gets \texttt{Div}(\mathbf{I})$ \Comment*[r]{Sub-region division}
    {$S_{\mathrm{forward}} \gets \langle\,\rangle$\;}
    {$S_{\mathrm{reverse}} \gets \langle\,\rangle$\;}
    {\While{$|S_{\mathrm{forward}}|+|S_{\mathrm{reverse}}|<|V|$}{
        $S_d \gets V \backslash \left(\operatorname{set}(S_{\mathrm{forward}}) \cup \operatorname{set}(S_{\mathrm{reverse}})\right)$\;
        \uIf {$S_d = \varnothing$}{
            \textbf{break}}
        $\alpha \gets  \arg{\max_{\alpha \in S_d}{\mathcal{F} \left( \operatorname{set}(S_{\mathrm{forward}}) \cup \{ \alpha \} \right) }}$ \Comment*[r]{Select the largest objective value}
        $S_{\mathrm{forward}} \gets \operatorname{append}(S_{\mathrm{forward}},\alpha)$ \Comment*[r]{Build the sequence front}
        $S_d \gets S_d \backslash \{\alpha\}$\;
        \uIf {$|S_d| > n_p$}{
            $S_p \gets \texttt{TOPK}^{\mathrm{min}}_{\alpha \in S_d}(\mathcal{F} \left( \operatorname{set}(S_{\mathrm{forward}}) \cup \{ \alpha \} \right), n_p)$\;
            $\alpha \gets  \arg{\min_{\alpha \in S_p}{\mathcal{F} \left( \{ \alpha \} \cup \operatorname{set}(S_{\mathrm{reverse}}) \right) }}$\;
            $S_{\mathrm{reverse}} \gets \operatorname{prepend}(\alpha,S_{\mathrm{reverse}})$ \Comment*[r]{Build the sequence back}
        }
    }}
    {$\boldsymbol{\pi} \gets S_{\mathrm{forward}} \mathbin{\|} S_{\mathrm{reverse}}$\;}
    {\Return $\boldsymbol{\pi}$}
\end{algorithm}


\subsection{Sub-region Importance Assessment}\label{importance_assess}

Although the sub-regions are ranked by importance, the ranking does not consider the degree of difference between them. The importance differences between adjacent sub-regions are unlikely to be uniform. To address this, we propose a simple first-order assignment method based on marginal contribution scores. Equations~\ref{consistency_function} and~\ref{collaboration_function} are closely related to the model’s decision-making process, allowing us to quantify the marginal effects by analyzing their variations to capture the differences between \revision{sub-regions}. The attribution score $a_i$ for each \revision{sub-region} $s_i$ in $S$ is given by:
\begin{equation}\label{attribution_score}
    a_i=\begin{cases}
     b_{\text{base}} & \text{if } i = 1, \\
     a_{i-1} - \left| \left( s_{\mathrm{cons.}}(S_{[i]}) + s_{\mathrm{colla.}}(S_{[i]}) \right) \right. & \\
     \quad \left. - \left( s_{\mathrm{cons.}}(S_{[i-1]}) + s_{\mathrm{colla.}}(S_{[i-1]}) \right) \right| & \text{if } i>1,
    \end{cases}
\end{equation}
where $b_{\text{base}}$ denotes a baseline value, and $S_{[i]}$ represents the set containing the top $i$ \revision{sub-regions} in the set $S$. When a new \revision{sub-region} is added, if the marginal effect does not increase significantly, it suggests that the importance of this \revision{sub-region} is comparable to that of the previous one. If the marginal effect is negative—indicating that the remaining \revision{sub-regions} have a counterproductive effect on the model’s decision—the negative impact of the \revision{sub-region} can be evaluated based on the absolute value of the marginal effect. By assessing the importance of all \revision{sub-regions} in this manner, we can better visualize which regions are most important and distinguish the differences in importance across sub-regions, enhancing human understandability.

%% file: revision_sec/5-method_analysis.tex
\section{Method Analysis}\label{method_analysis}

In this section, we evaluate the effectiveness of our approach, identify the scenarios in which it is most applicable, and compare it against other baseline attribution algorithms. 

\revision{\textbf{Context-dependent Response Analysis:} Let $G(S)=F(\sum_{\mathbf{I}^{M}\in S}\mathbf{I}^{M})/\|F(\sum_{\mathbf{I}^{M}\in S}\mathbf{I}^{M})\|\circ\boldsymbol{f}_{s}$ denote the target-semantic response of a subset, and let $H(S)=s_{\mathrm{cons.}}(S)+s_{\mathrm{colla.}}(S)$. For the ordered prefix $S_{[i]}$ and the newly added sub-region $\alpha_i=S_{[i]}\setminus S_{[i-1]}$, define insertion- and removal-side finite differences by $d_i^{\mathrm{ins}}=G(S_{[i]})-G(S_{[i-1]})$ and $d_i^{\mathrm{rem}}=G(V\setminus S_{[i-1]})-G(V\setminus S_{[i]})$. From the consistency and collaboration definitions, their combined change is exactly}
{\color{blue}
\begin{equation}\label{attribution_expansion}
    H(S_{[i]})-H(S_{[i-1]})
    =d_i^{\mathrm{ins}}+d_i^{\mathrm{rem}}.
\end{equation}
}
\revision{Thus, the importance assessment records the finite change in model response when a sub-region is evaluated under both the currently revealed and complementary contexts. Unlike an isolated-region score, this quantity is conditional on previously selected sub-regions. It can reflect synergy along the discovered ordering when a candidate becomes more influential after complementary evidence enters the prefix; redundant or distracting sub-regions may instead produce small or negative changes. This context dependence neither constitutes a unique decomposition of all higher-order interactions nor implies global submodularity.}

\revision{\textbf{Comparison with Existing Attribution Mechanisms:} Existing attribution methods estimate region importance through gradients, propagation, Shapley-value approximations, or sampled perturbations~\cite{deng2024unifying,kumar2020problems,chen2024defining}. Averaging over sampled contexts can obscure strongly context-dependent effects when complementary regions are rarely observed together or different contexts yield opposing responses. LiMA instead evaluates each candidate conditionally on an evolving subset and its complementary context. This design does not guarantee recovery of every interaction and remains dependent on the discovered search path, but it provides a direct mechanism for retaining context-dependent changes during ranking. We evaluate its practical benefit empirically in Section~\ref{experiment}.}

\begin{figure}[!t]
    \centering  
    \includegraphics[width=0.48\textwidth]{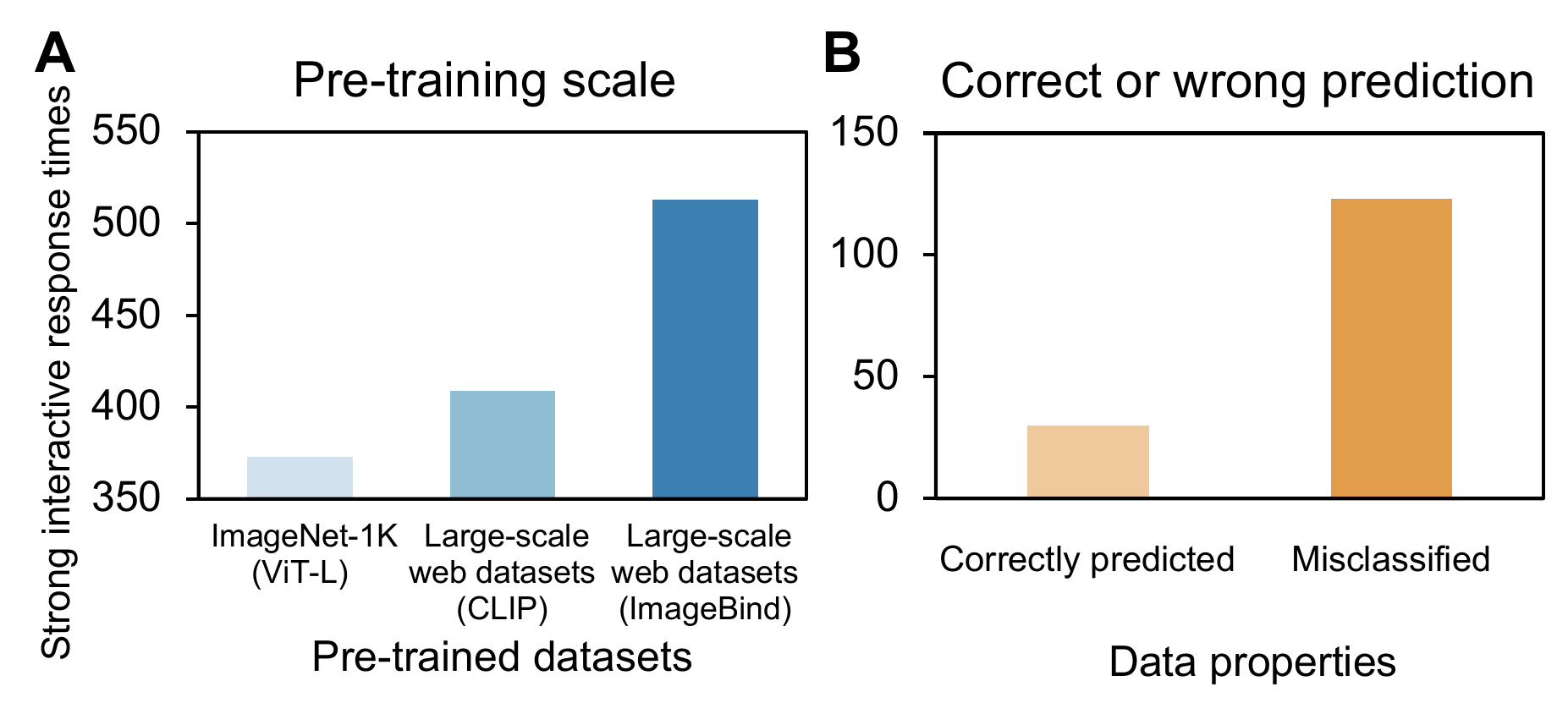} 
    \caption{\revision{Statistics of large response transitions along the insertion trajectory. \textbf{A.} Results for models with different pre-training scales and sizes. \textbf{B.} Results for correctly and incorrectly predicted samples.}} 
    \label{theory_ver}  
\end{figure}

\revision{\textbf{Empirical Response-transition Analysis:} We use abrupt changes along the insertion trajectory as an empirical indicator of context-dependent model behavior. This statistic is not a formal test of synergy, because a large transition may also be caused by a strong individual effect. We conduct two analyses. First, we attribute 1,000 correctly predicted samples using three models and count insertion steps whose absolute response change exceeds 0.5. Fig.~\ref{theory_ver}A shows more such transitions for the evaluated models with larger-scale pre-training or greater model capacity. Second, using ImageBind, we compare 1,000 correctly predicted and 1,000 misclassified samples and count insertion steps with a response decrease greater than 0.5. Fig.~\ref{theory_ver}B shows that these negative transitions occur more frequently among the evaluated misclassified samples. These observations motivate, but do not prove, the following empirical hypothesis:}

\begin{hypothesis}
\revision{Context-dependent subset evaluation may provide larger empirical benefits in settings that exhibit more frequent large response transitions, including the evaluated larger-scale models and misclassified samples.}
\end{hypothesis}

\revision{The comparative attribution experiments in the following section test the practical consequence of this hypothesis; they are not used to infer a universal interaction structure or set-function property.}

%% file: revision_sec/6-experiments.tex
\section{Experiments}\label{experiment}

In this section, we provide a validation of our \textsc{\textbf{LiMA}} mechanism. We begin by describing the datasets and experimental setup in Section~\ref{experimental_setup}, followed by a detailed explanation of the evaluation metrics in Section~\ref{evaluation_metric}. In Section~\ref{faithfulness_analysis}, we present the interpretation results of \textsc{LiMA} across multimodal and unimodal models, as well as various network architectures. Section~\ref{error_debug} demonstrates how \textsc{LiMA} explains the factors influencing model predictions across different modalities and network architectures. Finally, Section~\ref{ablation_study} verifies the effectiveness of each module in the proposed method through a series of ablation experiments.

\subsection{Experimental Setup}\label{experimental_setup}

\textbf{Datasets.} We evaluate the proposed \textsc{\textbf{LiMA}} on six datasets: a large-scale natural image dataset, ImageNet~\cite{deng2009imagenet}; a fine-grained bird dataset, CUB-200-2011~\cite{welindercaltech}; two face datasets, Celeb-A~\cite{liu2015deep} and VGG-Face2~\cite{cao2018vggface2}; a medical image dataset, LC25000 (Lung)~\cite{borkowski2019lung}; and an audio dataset, VGG-Sound~\cite{chen2020vggsound}. 

For the ImageNet dataset, which comprises 1,000 categories, we select five correctly predicted samples from each class in the validation set, resulting in 5,000 images per model to assess image attribution performance. Additionally, two incorrectly predicted samples per class are chosen, yielding a total of 2,000 images per model, \revision{to evaluate failure attribution for incorrect predictions.}

For the CUB-200-2011 dataset, which consists of 200 bird species, we select three correctly predicted samples from each class in the validation set, totaling 600 images, to evaluate image attribution performance. \revision{Additionally, two incorrectly predicted samples per class are chosen, resulting in 400 images, to evaluate failure attribution for incorrect predictions.}

The Celeb-A dataset includes 10,177 identities, from which we randomly select 2,000 identities from the validation set, using one test face image per identity to evaluate our method. Similarly, for the VGG-Face2 dataset, which contains 8,631 identities, we randomly select 2,000 identities from the validation set, with one test face image per identity used for evaluation.

For the LC25000 (Lung) dataset, which includes three lung cell classes, we selected 1,000 misclassified images to assess our method’s effectiveness in attributing errors in medical images.

For the VGG-Sound dataset, containing 309 audio classes, we randomly select two correctly predicted samples from each class in the validation set to assess attribution effectiveness, and one misclassified sample per class to evaluate error attribution.

\textbf{Implementation Details.} 
We primarily apply \textsc{\textbf{LiMA}} to interpret architectures based on CNN, ViT, and Mamba. For the ImageNet dataset, we evaluated our method on four multimodal foundation models: CLIP (ViT-L)~\cite{radford2021learning}, CLIP (ResNet-101), ImageBind (Huge)~\cite{girdhar2023imagebind}, and LanguageBind (Large)~\cite{zhu2024languagebind}, as well as five single-modal pre-trained models: ResNet-101~\cite{he2016deep}, ViT-Large~\cite{dosovitskiy2021image}, Swin Transformer (Large)~\cite{liu2021swin}, Vision Mamba (Base)~\cite{zhu2024vision}, and MambaVision (L2)~\cite{hatamizadeh2024mambavision}. All models utilized their official pre-trained parameters. For image classification, we used the video encoder of LanguageBind, while for the other models, we employed their respective image encoders.
For the face datasets, we evaluated recognition models trained with the ResNet-101 architecture~\cite{he2016deep} and the ArcFace loss function~\cite{deng2019arcface}, with an input size of $112 \times 112$ pixels. For the CUB-200-2011 dataset, we tested three recognition models built on the ResNet-101, MobileNetV2~\cite{sandler2018mobilenetv2}, and EfficientV2-M~\cite{tan2021efficientnetv2} architectures. In the case of the LC25000 (Lung) dataset, we employed QuiltNet (ViT/B-32)~\cite{ikezogwo2024quilt}, a vision-language medical foundation model trained on the Quilt-1M dataset. Finally, for the VGG-Sound dataset, we evaluated our method using two foundation models: ImageBind (Huge) and LanguageBind (Large).

For the target semantic feature $\boldsymbol{f}_{s}$, in the multimodal foundation model, we utilize the semantic vector encoded by the text encoder. In contrast, for the single-modal model, we directly use the fully connected layer following the backbone. Unless otherwise specified, the number of sub-regions segmented using the SLICO~\cite{achanta2012slic} algorithm is fixed at $|V|=49$, whereas the number of sub-regions generated by the SAM~\cite{kirillov2023segment} algorithm is not controllable. Our experiments were conducted using the Xplique repository\footnote{Xplique:~\url{https://github.com/deel-ai/xplique}}, which provides baseline methods and evaluation tools. All experiments were performed on an NVIDIA 4090 GPU.

\subsection{Evaluation Metrics}\label{evaluation_metric}

Since attribution methods aim to identify which inputs influenced a model’s decision, the resulting explanations should be evaluated not by humans but by objective faithfulness metrics. We employ four fidelity metrics to evaluate our attribution methods.

The first is the \textbf{Deletion AUC score}~\cite{petsiuk2018rise}, which quantifies the reduction in the model’s output when important regions are replaced with a baseline value. A sharp decline in performance indicates that the explanation method effectively identifies the key variables influencing the decision. Let $\mathbf{x}_{[x_{T}=x_{0}]}$ denote the input where the $T$ most important variables, according to the attribution map, are set to the baseline value $x_{0}=0$. Given a set $\mathcal{T} = \{T_0,T_1,\cdots, T_n\}$, where $T_0=0$ and $T_n$ is the input size of $\mathbf{x}$, this set represents the selected numbers of the most important regions. Then, the Deletion AUC score is given by:
\begin{equation}
    \text{Del.} = \sum_{i=1}^{n} \frac{\left (
    f(\mathbf{x}_{[\mathbf{x}_{T_i}=x_{0}]})+f(\mathbf{x}_{[\mathbf{x}_{T_{i-1}}=x_{0}]}) \right )\cdot \left ( T_{i}-T_{i-1}\right )}{2T_n},
\end{equation}
the lower this metric, the better the attribution performance.

The second metric is the \textbf{Insertion AUC score}~\cite{petsiuk2018rise}, which quantifies the increase in the model’s output as important regions are progressively revealed. This metric is defined as follows:
\begin{equation}
    \text{Ins.} = \sum_{i=1}^{n} \frac{\left (
    f(\mathbf{x}_{[\mathbf{x}_{\bar{T}_i}=x_{0}]})+f(\mathbf{x}_{[\mathbf{x}_{\bar{T}_{i-1}}=x_{0}]}) \right )\cdot \left ( T_{i}-T_{i-1}\right )}{2T_n},
\end{equation}
where $\mathbf{x}_{[x_{\bar{T}}=x_0]}$ denotes the input where elements not belonging to the set $T$ are set to the baseline value $x_0=0$. The higher this metric, the better the attribution performance.

\begin{table*}[!t]
    \caption{\revision{Faithfulness, absolute query counts, and wall-clock time for multimodal ViT backbone foundation models on ImageNet. F/B denotes forward queries/backward passes; time is mean $\pm$ sample standard deviation in seconds.}}
    \label{faithfulness_on_correct_imagenet}
    \begin{center}
        \resizebox{\textwidth}{!}{
            \begin{tabular}{l|ccccc|ccccc|ccccc}
                \toprule
                \multirow{2}{*}{Methods} & \multicolumn{5}{c}{CLIP (ViT-L)~\cite{radford2021learning}} & \multicolumn{5}{c}{ImageBind (Huge)~\cite{girdhar2023imagebind}} & \multicolumn{5}{c}{LanguageBind (Large)~\cite{zhu2024languagebind}} \\ \cmidrule(r){2-6} \cmidrule(lr){7-11} \cmidrule(l){12-16}
                  & Del. ($\downarrow$) & Ins. ($\uparrow$) & $\mu$Fid. ($\uparrow$) & F/B & Time & Del. ($\downarrow$) & Ins. ($\uparrow$) & $\mu$Fid. ($\uparrow$) & F/B & Time & Del. ($\downarrow$) & Ins. ($\uparrow$) & $\mu$Fid. ($\uparrow$) & F/B & Time \\ \midrule
                Saliency~\cite{simonyan2014deep} & 0.3040 & 0.1938 & 0.1010 & 1/1 & $0.095{\pm}0.002$ & 0.3857 & 0.3473 & 0.1807 & 1/1 & $0.127{\pm}0.002$ & 0.2237 & 0.2379 & 0.2395 & 1/1 & $0.402{\pm}0.002$ \\
                Integrated Gradients~\cite{sundararajan2017axiomatic} & 0.1600 & 0.3560 & 0.0881 & 80/80 & $1.498{\pm}0.929$ & 0.1405 & 0.5577 & 0.1919 & 80/80 & $5.059{\pm}0.033$ & 0.0699 & 0.2733 & 0.1422 & 80/80 & $29.11{\pm}0.63$ \\
                iGOS++~\cite{khorram2021igos++} & 0.2934 & 0.4467 & 0.2313 & 16,426.7/1,600 & $85.48{\pm}11.13$ & 0.2658 & 0.4417 & 0.2371 & 17,253.3/1,600 & $447.3{\pm}22.6$ & -- & -- & -- & -- & -- \\
                RISE~\cite{petsiuk2018rise} & 0.1615 & 0.5345 & 0.3244 & 500/0 & $2.041{\pm}0.102$ & 0.1528 & 0.4699 & 0.3459 & 500/0 & $11.23{\pm}0.10$ & 0.1145  &  0.4750 & 0.2764 & 500/0 & $67.49{\pm}0.44$ \\
                HSIC-Attribution~\cite{novello2022making} & 0.1565 & 0.4397 & 0.3232 & 500/0 & $2.454{\pm}0.746$ & 0.1875 & 0.4407 & 0.3413 & 500/0 & $11.46{\pm}0.04$ & 0.1034 & 0.4070 & \textbf{0.4132} & 500/0 & $68.25{\pm}0.12$ \\
                Kernel SHAP~\cite{lundberg2017unified} & 0.2808 & 0.3588 & 0.0699 & 1,000/0 & $4.341{\pm}0.393$ & 0.3736 & 0.4714 & 0.0931 & 1,000/0 & $23.74{\pm}0.10$ & 0.2104 & 0.3109 & 0.0998 & 1,000/0 & $138.6{\pm}2.9$ \\
                ViT-CX~\cite{xie2023vit} & 0.1779 & 0.4413 & 0.2076 & 126/1 & $0.918{\pm}0.265$ & 0.1987 & 0.4277 & 0.3377 & 147.3/1 & $3.705{\pm}0.297$ & 0.1303 & 0.3486 & 0.1015 & 88/1 & $10.91{\pm}0.78$ \\
                EAC~\cite{sun2023explain} (w/ Segment Anything) &  0.3792 & 0.4349 & 0.1566 & -- & -- & 0.3888 & 0.4564 & 0.2740 & -- & -- & 0.1563 & 0.4340 & 0.2848 & -- & -- \\
                Grad-ECLIP~\cite{zhao2024gradient} & 0.1374 & 0.4783 & 0.3205 & 1/1 & $0.046{\pm}0.032$ & -- & -- & -- & -- & -- & -- & -- & -- & -- & -- \\
                Previous work~\cite{chen2024less} & 0.1351 & 0.6752 & -- & -- & -- & 0.1378 & 0.6652 & -- & -- & -- & 0.0885 & 0.5082 & -- & -- & -- \\
                \rowcolor{gray!20}
                \textsc{LiMA} (w/ Segment Anything) & 0.1905 & 0.7151 & 0.2953 & -- & -- & 0.2210 & 0.7341 & 0.3492 & -- & -- & 0.2228 & \textbf{0.6690} & 0.3648 & -- & -- \\ \rowcolor{gray!20}
                \textsc{LiMA} (w/ SLICO) & \textbf{0.0926} & \textbf{0.7441} & \textbf{0.3334} & 1,660/0 & $8.774{\pm}0.159$ & \textbf{0.1280} & \textbf{0.7463} & \textbf{0.3551} & 1,660/0 & $39.46{\pm}0.26$ & \textbf{0.0642} & 0.6058 & 0.4096 & 1,660/0 & $199.4{\pm}5.6$ \\
                \bottomrule
                \end{tabular}
        }
    \end{center}
\end{table*}

\begin{figure*}[!t]
    \centering
    \begin{overpic}[width=\textwidth,tics=8]{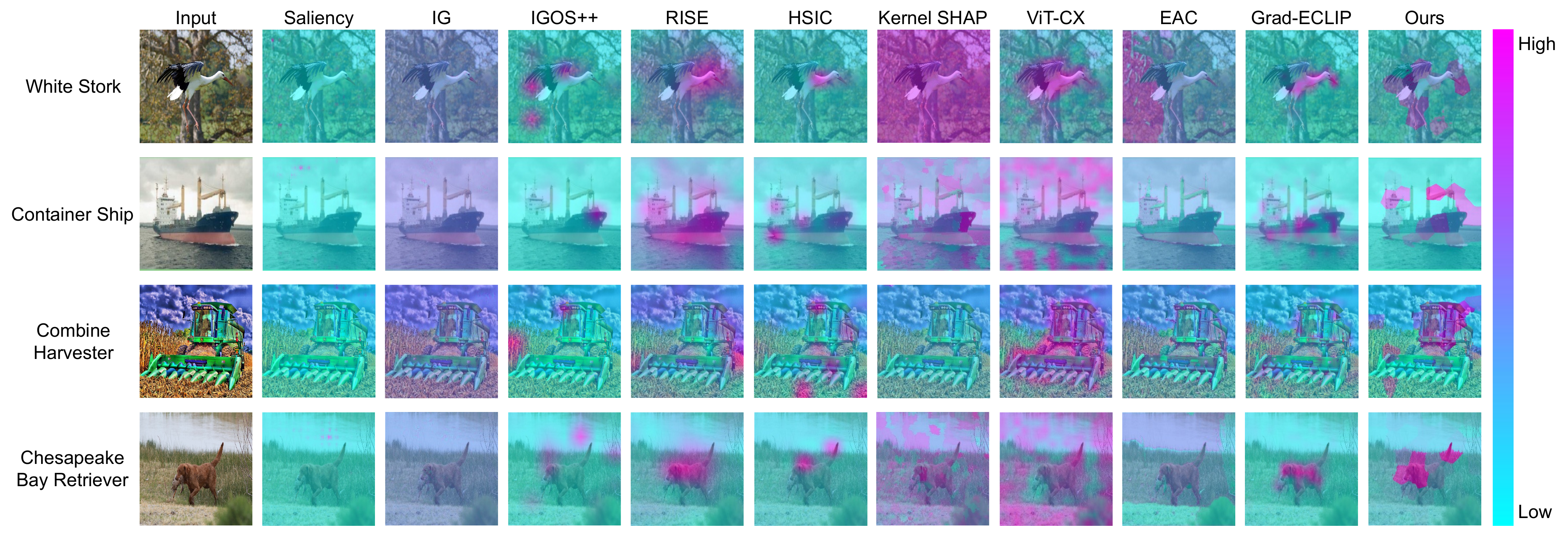} 
        \put (22.7,32.65) {\tiny \cite{simonyan2014deep}}
        \put (28.7,32.65) {\tiny \cite{sundararajan2017axiomatic}}
        \put (38.2,32.65) {\tiny \cite{khorram2021igos++}}
        \put (45.4,32.65) {\tiny \cite{petsiuk2018rise}}
        \put (53.0,32.65) {\tiny \cite{novello2022making}}
        \put (62.9,32.65) {\tiny \cite{lundberg2017unified}}
        \put (69.2,32.65) {\tiny \cite{xie2023vit}}
        \put (76.4,32.65) {\tiny \cite{sun2023explain}}
        \put (86.3,32.65) {\tiny \cite{zhao2024gradient}}
    \end{overpic}  
    \caption{Visual explanations of the CLIP model using various attribution mechanisms, with our approach effectively reducing noise and eliminating redundant regions, leading to more interpretable attribution results.}  
    \label{vis:multimodal_qualitative}  
\end{figure*}

\begin{figure*}[!t]
    \centering
    \begin{overpic}[width=\textwidth,tics=8]{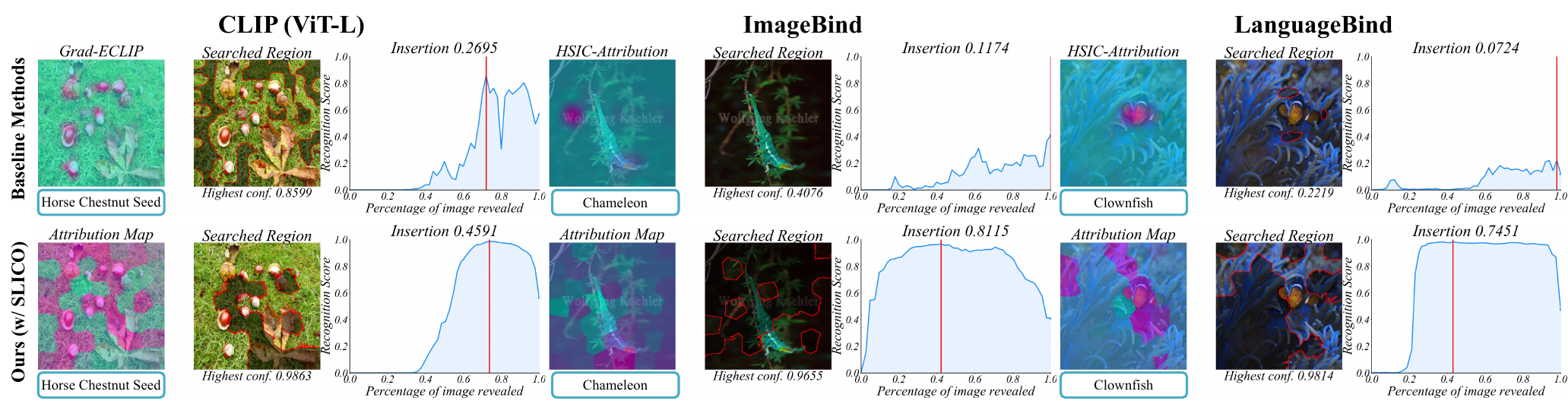} 
        \put (9.3,22.2) {\tiny \cite{zhao2024gradient}}
        \put (42.8,22.2) {\tiny \cite{novello2022making}}
        \put (75.4,22.2) {\tiny \cite{novello2022making}}
    \end{overpic}  
    \caption{Attribution visualizations of decision results for different multimodal foundation models on the ImageNet dataset. The first row shows the interpretation results from the state-of-the-art baseline attribution methods, while the second row displays the interpretation results from our method. Each interpretation result includes the saliency map, the highest confidence score, and its corresponding region, as well as the Insertion AUC curve. The red line in the curve represents the highest confidence of the model’s response during the search.}  
    \label{vis:multimodal}  
\end{figure*}

The third metric is the \textbf{average highest confidence}~\cite{chen2024less}, which measures the model’s highest response within a constrained search region. This metric is used to evaluate attribution results for samples that the model incorrectly predicts, assessing whether the method can identify the factors contributing to the model’s incorrect predictions. It is defined as follows:
\begin{equation}
    \text{highest conf.} = \max_{T\in\mathcal{T}_{[k]}}{(f(\mathbf{x}_{[\mathbf{x}_{\bar{T}}=x_{0}]}))},
\end{equation}
where $\mathcal{T}_{[k]}$ denotes the set of most important regions that can be searched, constrained by a search region that does not exceed $k$. The higher this metric, the better the attribution performance.

The final metric is \textbf{$\mu$Fidelity}~\cite{bhatt2021evaluating}, which measures the correlation between the reduction in the model’s score when variables are set to a baseline state and the importance of those variables. However, since the network’s output is not necessarily a linear combination of its inputs, this metric may not fully capture the quality of the attribution. Nonetheless, it is primarily used to evaluate how effectively our method identifies the importance of sub-regions.

\subsection{Faithfulness Analysis}\label{faithfulness_analysis}

To demonstrate the superiority of our method, we assess its faithfulness, which measures the alignment between the generated explanations and the deep model’s decision-making process~\cite{li2021instance}. We employ Deletion and Insertion AUC scores, along with $\mu$Fidelity, as evaluation metrics. Our method is evaluated on correctly predicted samples across various modalities and architectures.

\subsubsection{Faithfulness on Multimodal Foundation Models}

We first evaluate the interpretative performance of our attribution method on multimodal foundation models, including CLIP~\cite{radford2021learning}, ImageBind~\cite{girdhar2023imagebind}, and LanguageBind~\cite{zhu2024languagebind}. We select 5,000 correctly predicted ImageNet images for each model for validation. Table~\ref{faithfulness_on_correct_imagenet} presents the faithfulness results across various attribution methods, with our approach demonstrating strong attribution effectiveness. The attribution results on CLIP (ViT-L) show that, compared to the state-of-the-art Grad-ECLIP
~\cite{zhao2024gradient}, which is based on gradient and internal attention activation, our method achieves a 32.6\% improvement in Deletion AUC score, a 61.8\% improvement in Insertion AUC score, and a 4.0\% improvement in $\mu$Fidelity when using SLICO as the sub-region division algorithm. Fig.~\ref{vis:multimodal_qualitative} qualitatively shows the saliency maps generated by different attribution methods on the CLIP model, where our approach effectively reduces noise and eliminates redundant regions, resulting in more interpretable and clearer attribution maps. On the ImageBind model, the Integrated Gradients~\cite{sundararajan2017axiomatic} method produces better attribution results. In comparison, our method improves the Deletion AUC score by 9.8\%, the Insertion AUC score by 33.8\%, and the $\mu$Fidelity by 85.0\% when using the SLICO sub-region division algorithm. For LanguageBind, our approach outperforms the perturbation-based HSIC-Attribution~\cite{novello2022making} method by 43.9\% in the Deletion AUC score and 48.8\% in the Insertion AUC score when using SLICO as the sub-region division algorithm. Although our method is slightly lower by 0.9\% in the $\mu$Fidelity metric, it still demonstrates strong attribution capability. Compared to the previous version, our method improves the Insertion AUC scores by 10.2\%, 12.2\%, and 19.2\% and the Deletion AUC scores by 31.5\%, 7.1\%, and 27.5\% across the three models, respectively, demonstrating that attribution results based on semantic division are more effective. 

We also validated our attribution method on the CLIP model with ResNet-101 as the backbone, as shown in Table~\ref{faithfulness_on_correct_imagenet_clip_rn101}. Our approach outperforms the state-of-the-art perturbation-based HSIC-Attribution method by 31.9\% in the Deletion AUC score and 51.8\% in the Insertion AUC score when using SLICO as the sub-region division algorithm. Although it does not achieve the best results in the $\mu$Fidelity metric, the performance on the first two metrics demonstrates that our method more effectively identifies the most critical combination of decision regions. Compared with the previous version of the method, the Insertion and Deletion AUC scores increased by 21.3\% and 8.4\%, respectively.

\begin{table}[h]
    \caption{\revision{Faithfulness, absolute query counts, and wall-clock time for CLIP with a ResNet-101 backbone on ImageNet.}}
    \label{faithfulness_on_correct_imagenet_clip_rn101}
    \begin{center}
        \resizebox{0.48\textwidth}{!}{
            \begin{tabular}{l|ccccc}
                \toprule
                \multirow{2}{*}{Methods} & \multicolumn{5}{c}{CLIP (ResNet-101)~\cite{radford2021learning}} \\ \cmidrule(r){2-6}
                  & Deletion ($\downarrow$) & Insertion ($\uparrow$) & $\mu$Fidelity ($\uparrow$) & F/B & Time \\ \midrule
                Saliency~\cite{simonyan2014deep} & 0.0801 & 0.1540 & 0.3648 & 1/1 & $0.084{\pm}0.005$ \\
                Integrated Gradients~\cite{sundararajan2017axiomatic} & 0.0482 & 0.0804 & 0.1141 & 80/80 & $0.588{\pm}0.212$ \\
                Kernel SHAP~\cite{lundberg2017unified} & 0.1198 & 0.1766 & 0.0923 & 1,000/0 & $2.454{\pm}0.342$ \\
                RISE~\cite{petsiuk2018rise} & 0.0719 & 0.4034 & 0.3355 & 500/0 & $0.744{\pm}0.011$ \\
                HSIC-Attribution~\cite{novello2022making} & 0.0659 & 0.3574 & \textbf{0.4373} & 500/0 & $0.831{\pm}0.021$ \\
                EAC~\cite{sun2023explain} (w/ Segment Anything)   & 0.2332 & 0.2756 & 0.2463 & -- & -- \\
                Previous work~\cite{chen2024less} & 0.0490 & 0.4474 & -- & -- & -- \\
                \rowcolor{gray!20}
                \textsc{LiMA} (w/ SLICO) & \textbf{0.0449} & \textbf{0.5427} & 0.3827 & 1,660/0 & $5.020{\pm}0.816$ \\ 
                \bottomrule
                \end{tabular}
        }
    \end{center}
\end{table}

From Tables~\ref{faithfulness_on_correct_imagenet} and~\ref{faithfulness_on_correct_imagenet_clip_rn101}, we observe that comparable performance can be achieved when using SLICO as the division method. However, when using Segment Anything (SAM) for sub-region division, although it remains competitive compared to other attribution methods, there are noticeable performance differences relative to superpixel-based methods. This is mainly because SAM occasionally segments sub-regions that are too large, lacking fine-grained detail, and may include areas that negatively impact the results. Optimizing SAM for more fine-grained segmentation will be a focus of future work, which could lead to improved performance.

Fig.~\ref{vis:multimodal} shows the interpretation results for different multimodal foundation models on the ImageNet dataset. In the CLIP model attribution results, both Grad-ECLIP and our method accurately located the horse chestnut seed. However, Grad-ECLIP’s response jittered as key regions were revealed, while ours remained stable. Additionally, the region searched by our method achieved higher confidence, demonstrating its superiority. In the ImageBind model, the watermark affected the HSIC-Attribution method, causing it to misattribute the location and miss the chameleon, leading to poor faithfulness. In contrast, our method successfully identifies the area where ImageBind accurately recognizes the chameleon category. An interesting phenomenon occurs in LanguageBind: while the HSIC-Attribution method seems to more effectively locate the clownfish, the model shows no positive response as key regions are revealed. Our method, however, identifies both the clownfish and surrounding background regions, enabling the model to recognize the category more effectively. This highlights the gap between human cognition and machine learning, illustrating why humans alone cannot evaluate attribution effectiveness. Please refer to Appendix~\ref{supp:vis-multimodal} for more visualization results.

By reformulating the image attribution problem as a subset selection problem, we effectively address the issue of insufficient granularity in the attribution region, making our approach better suited for explaining models with complex interactive responses. As a result, our method significantly improves faithfulness compared to existing attribution algorithms.

\subsubsection{Faithfulness on Single-Modal Models}

We also validate the effectiveness of our explanation method on single-modal models, including CNN, Vision Transformer, and Vision Mamba architectures. For each model, we select 5,000 correctly predicted ImageNet images.

\textbf{Faithfulness on CNN:} We select the ResNet-101~\cite{he2016deep} pre-trained on ImageNet to verify the effectiveness of our method. Table~\ref{faithfulness_on_correct_imagenet_rn101} presents the attribution performance of various methods, where our approach outperforms the state-of-the-art perturbation-based HSIC-Attribution method by 34.5\% in the Deletion AUC score and 47.0\% in the Insertion AUC score when using SLICO as the sub-region division algorithm. Compared with the previous version of the method, the Insertion and Deletion AUC scores increased by 29.1\% and 11.8\%, respectively. This demonstrates that our method is highly competitive on the CNN models. Please refer to Appendix~\ref{supp:vis-resnet} for the visualization results.

\begin{table}[h]
    \caption{\revision{Faithfulness, absolute query counts, and wall-clock time for the single-modal ResNet-101 backbone on ImageNet.}}
    \label{faithfulness_on_correct_imagenet_rn101}
    \begin{center}
        \resizebox{0.48\textwidth}{!}{
            \begin{tabular}{l|ccccc}
                \toprule
                \multirow{2}{*}{Methods} & \multicolumn{5}{c}{ResNet-101~\cite{he2016deep}} \\ \cmidrule(r){2-6}
                  & Deletion ($\downarrow$) & Insertion ($\uparrow$) & $\mu$Fidelity ($\uparrow$) & F/B & Time \\ \midrule
                Saliency~\cite{simonyan2014deep} & 0.1148 & 0.2613 & 0.4556 & 1/1 & $0.087{\pm}0.004$ \\
                Integrated Gradients~\cite{sundararajan2017axiomatic} & \textbf{0.0521} & 0.1482 & 0.1822 & 80/80 & $0.602{\pm}0.056$ \\
                iGOS++~\cite{khorram2021igos++} & 0.1934 & 0.3779 & 0.3377 & 14,346.7/1,600 & $19.96{\pm}0.37$ \\
                Kernel SHAP~\cite{lundberg2017unified} & 0.1944 & 0.2878 & 0.1399 & 1,000/0 & $2.212{\pm}0.271$ \\
                RISE~\cite{petsiuk2018rise} & 0.1101 & 0.4920 & 0.3827 & 500/0 & $0.839{\pm}0.042$ \\
                HSIC-Attribution~\cite{novello2022making} & 0.0866 & 0.4448 & \textbf{0.5805} & 500/0 & $0.900{\pm}0.022$ \\
                Previous work~\cite{chen2024less} & 0.0643 & 0.5063 & -- & -- & -- \\
                \rowcolor{gray!20}
                \textsc{LiMA} (w/ SLICO) & 0.0567 & \textbf{0.6538} & 0.4203 & 1,660/0 & $5.012{\pm}0.581$ \\ 
                \bottomrule
                \end{tabular}
        }
    \end{center}
\end{table}

\textbf{Faithfulness on Vision Transformer:} To validate the effectiveness of our method on transformer backbones, we employed the Vision Transformer-Large (ViT-L)~\cite{dosovitskiy2021image} and Swin Transformer-Large (Swin-L)~\cite{liu2021swin}, both pre-trained on ImageNet. Table~\ref{faithfulness_on_correct_imagenet_vit} presents the attribution performance, where RISE~\cite{petsiuk2018rise} proves to be a strong baseline method for ViT backbones. Compared to RISE, our method outperforms both the ViT-L and Swin-L models, achieving a 48.9\% and 56.4\% improvement in the Deletion AUC score and a 19.9\% and 34.9\% improvement in the Insertion AUC score, respectively. Compared with the previous version of the method, the Insertion and Deletion AUC scores increased by 11.1\% and 33.2\% on average, respectively. Please refer to Appendix~\ref{supp:vis-vit} for the visualization results.

\begin{table}[h]
    \caption{\revision{Faithfulness, absolute query counts, and wall-clock time for single-modal vision transformer backbones on ImageNet.}}
    \label{faithfulness_on_correct_imagenet_vit}
    \begin{center}
        \resizebox{0.48\textwidth}{!}{
            \begin{tabular}{l|cccc|cccc}
                \toprule
                \multirow{2}{*}{Methods} & \multicolumn{4}{c}{ViT-L~\cite{dosovitskiy2021image}} & \multicolumn{4}{c}{Swin-L~\cite{liu2021swin}} \\ \cmidrule(r){2-5} \cmidrule(r){6-9}
                  & Del. ($\downarrow$) & Ins. ($\uparrow$) & F/B & Time & Del. ($\downarrow$) & Ins. ($\uparrow$) & F/B & Time \\ \midrule
                Saliency~\cite{simonyan2014deep} & 0.3499 & 0.3544 & 1/1 & $0.096{\pm}0.002$ & 0.2888 & 0.4178 & 1/1 & $0.213{\pm}0.088$ \\
                Integrated Gradients~\cite{sundararajan2017axiomatic} & 0.1530 & 0.2402 & 80/80 & $2.029{\pm}0.015$ & 0.3074 & 0.4198 & 80/80 & $1.506{\pm}0.032$ \\
                Kernel SHAP~\cite{lundberg2017unified} & 0.4155 & 0.4942 & 1,000/0 & $9.250{\pm}0.191$ & 0.4300 & 0.4794 & 1,000/0 & $6.788{\pm}0.110$ \\
                RISE~\cite{petsiuk2018rise} & 0.1947 & 0.6177 & 500/0 & $4.335{\pm}0.052$ & 0.3338 & 0.5655 & 500/0 & $3.273{\pm}0.030$ \\
                HSIC-Attribution~\cite{novello2022making} & 0.1898 & 0.5253 & 500/0 & $4.423{\pm}0.022$ & 0.3589 & 0.5531 & 500/0 & $3.377{\pm}0.043$ \\
                Previous work~\cite{chen2024less} & 0.1268 & 0.6715 & -- & -- & 0.2490 & 0.6811 & -- & -- \\
                \rowcolor{gray!20}
                \textsc{LiMA} (w/ SLICO) & \textbf{0.0994} & \textbf{0.7405} & 1,660/0 & $17.05{\pm}0.75$ & \textbf{0.1371} & \textbf{0.7627} & 1,660/0 & $14.24{\pm}0.66$ \\ 
                \bottomrule
                \end{tabular}
        }
    \end{center}
\end{table}

\textbf{Faithfulness on Vision Mamba:} To evaluate the effectiveness of our attribution method on the latest Mamba architecture, we tested it on Vision Mamba (base)~\cite{zhu2024vision} and MambaVision (L2)~\cite{hatamizadeh2024mambavision}, both pre-trained on ImageNet. As shown in Table~\ref{faithfulness_on_correct_imagenet_vim}, RISE remains a robust baseline method among the comparisons. However, our method demonstrated superior performance, surpassing RISE on both the Vision Mamba (base) and MambaVision (L2) models with a 58.9\% and 59.8\% increase in Deletion AUC score, and a 30.0\% and 25.3\% improvement in Insertion AUC score, respectively. Compared with the previous version of the method, the Insertion and Deletion AUC scores increased by 14.3\% and 37.6\%, respectively. Please see Appendix~\ref{supp:vis-vim} for the visualization results.

\begin{table}[h]
    \caption{\revision{Faithfulness, absolute query counts, and wall-clock time for single-modal Mamba backbones on ImageNet.}}
    \label{faithfulness_on_correct_imagenet_vim}
    \begin{center}
        \resizebox{0.48\textwidth}{!}{
            \begin{tabular}{l|cccc|cccc}
                \toprule
                \multirow{2}{*}{Methods} & \multicolumn{4}{c}{Vision Mamba (base)~\cite{zhu2024vision}} & \multicolumn{4}{c}{MambaVision (L2)~\cite{hatamizadeh2024mambavision}} \\ \cmidrule(r){2-5} \cmidrule(r){6-9}
                  & Del. ($\downarrow$) & Ins. ($\uparrow$) & F/B & Time & Del. ($\downarrow$) & Ins. ($\uparrow$) & F/B & Time \\ \midrule
                Saliency~\cite{simonyan2014deep} & 0.2793 & 0.1938 & 1/1 & $0.127{\pm}0.009$ & 0.1681 & 0.1494 & 1/1 & $0.099{\pm}0.005$ \\
                Integrated Gradients~\cite{sundararajan2017axiomatic} & 0.0914 & 0.2323 & 80/80 & $1.319{\pm}0.022$ & 0.1067 & 0.2202 & 80/80 & $1.236{\pm}0.015$ \\
                Kernel SHAP~\cite{lundberg2017unified} & 0.2947 & 0.3603 & 1,000/0 & $5.878{\pm}0.351$ & 0.2030 & 0.2345 & 1,000/0 & $5.827{\pm}0.149$ \\
                RISE~\cite{petsiuk2018rise} & 0.2008 & 0.5352 & 500/0 & $2.789{\pm}0.425$ & 0.1419 & 0.2937 & 500/0 & $2.511{\pm}0.027$ \\
                HSIC-Attribution~\cite{novello2022making} & 0.2067 & 0.4872 & 500/0 & $2.854{\pm}0.361$ & 0.1541 & 0.2753 & 500/0 & $2.725{\pm}0.048$ \\
                Previous work~\cite{chen2024less} & 0.1387 & 0.6281 & -- & -- & 0.0872 & 0.3121 & -- & -- \\
                \rowcolor{gray!20}
                \textsc{LiMA} (w/ SLICO) & \textbf{0.0825} & \textbf{0.6957} & 1,660/0 & $12.81{\pm}0.70$ & \textbf{0.0570} & \textbf{0.3680} & 1,660/0 & $11.88{\pm}0.35$ \\ 
                \bottomrule
                \end{tabular}
        }
    \end{center}
    \vspace{-10pt}
\end{table}

In summary, our interpretable approach demonstrates strong attribution performance and generalizes well across models with different architectures.

\textbf{\revision{Faithfulness on Face Recognition Models:}} Table~\ref{faithfulness_on_correct_face} shows the results on the Celeb-A and VGG-Face2 validation sets. 
On the Celeb-A dataset, our method surpasses the state-of-the-art HSIC-Attribution~\cite{novello2022making} approach, achieving a 42.0\% improvement in Deletion AUC, a 5.9\% increase in Insertion AUC, and a 9.0\% enhancement in $\mu$Fidelity. Similarly, on the VGG-Face2 dataset, our method outperforms HSIC-Attribution with gains of 46.3\% in Deletion AUC, 2.2\% in Insertion AUC, and 1.2\% in $\mu$Fidelity. In addition, compared to the previous version~\cite{chen2024less}, our method improves the Deletion AUC score by 36.4\% and the Insertion AUC score by 4.8\% on the Celeb-A dataset. On the VGG-Face2 dataset, it achieves improvements of 45.8\% in the Deletion AUC score and 2.0\% in the Insertion AUC score. Compared to the previous large-scale pre-trained models validated on the ImageNet dataset, the improvement is less pronounced. This is primarily due to the fact that face recognition images are aligned, resulting in fewer complex interaction effects than those encountered in large-scale pre-trained models. Nevertheless, our method shows consistent improvement over the baseline and achieves state-of-the-art performance.

\begin{table}[h]
    \caption{Deletion, Insertion AUC scores, and $\mu$Fidelity on the Celeb-A and VGG-Face2 validation sets.}
    \label{faithfulness_on_correct_face}
    \begin{center}
        \resizebox{0.48 \textwidth}{!}{
            \begin{tabular}{l|ccc|ccc}
            \toprule
             & \multicolumn{3}{c}{Celeb-A (ArcFace ResNet-101)}       & \multicolumn{3}{c}{VGGFace2 (ArcFace ResNet-101)} \\ \cmidrule(lr){2-4} \cmidrule(lr){5-7} 
            \multirow{-2}{*}{Methods} & Deletion ($\downarrow$)    & Insertion ($\uparrow$) & $\mu$Fidelity ($\uparrow$)  & Deletion ($\downarrow$) & Insertion ($\uparrow$) & $\mu$Fidelity ($\uparrow$)  \\ \midrule
            Saliency~\cite{simonyan2014deep} & 0.1453 & 0.4632 & 0.3258 & 0.1907 & 0.5612 & 0.5034 \\ 
            Grad-CAM~\cite{selvaraju2020grad} & 0.2865 & 0.3721 & 0.0672 & 0.3103 & 0.4733 & 0.2773 \\ 
            Integrated Gradients~\cite{sundararajan2017axiomatic} & 0.0680 & 0.3578 & 0.3352 & 0.0749 & 0.5399 & 0.3853 \\
            LIME~\cite{ribeiro2016should} & 0.1484 & 0.5246 & 0.1985 & 0.2034 & 0.6185 & 0.4856 \\ 
            Kernel SHAP~\cite{lundberg2017unified} & 0.1409 & 0.5246 & 0.2722 & 0.2119 & 0.6132 & 0.4905 \\ 
            RISE~\cite{petsiuk2018rise} & 0.1444 & 0.5703 &  0.4147 & 0.1375 & 0.6530 & 0.3822 \\ 
            HSIC-Attribution~\cite{novello2022making} & 0.1151 & 0.5692 & 0.5530  & 0.1317 &  0.6694 & 0.5295 \\ 
            Previous work~\cite{chen2024less}  & 0.1054 & 0.5752 & -  & 0.1304 & 0.6705 & -   \\ \rowcolor{gray!20}
            \textsc{LiMA} (w/ SLICO)   & \textbf{0.0668} & \textbf{0.5849} & \textbf{0.5743} & \textbf{0.0707} & \textbf{0.6841} & \textbf{0.5361} \\ 
            \bottomrule
            \end{tabular}
        }
    \end{center}
    \vspace{-10pt}
\end{table}

\textbf{\revision{Faithfulness on A Fine-grained Model:}} Table~\ref{faithfulness_on_correct_cub} shows the results on the CUB-200-2011 validation sets. Compared to the state-of-the-art HSIC-Attribution~\cite{novello2022making} method, our approach surpasses it by 13.44\% in Deletion AUC score, 13.22\% in Insertion AUC score, and achieves a 3.3\% improvement in $\mu$Fidelity. Additionally, compared to the previous version~\cite{chen2024less}, our method improves the Deletion AUC score by 8.6\% and the Insertion AUC score by 6.7\%. The CUB-200-2011 dataset is a single-object dataset, meaning the model correctly predicts with fewer sub-region complex interaction effects, but our method still delivers highly competitive results.

\begin{table}[h]
    \caption{Deletion, Insertion AUC scores, and $\mu$Fidelity on the CUB-200-2011 validation set.}
    \label{faithfulness_on_correct_cub}
    \begin{center}
        \resizebox{0.4 \textwidth}{!}{
            \begin{tabular}{l|ccc}
            \toprule
            & \multicolumn{3}{c}{CUB-200-2011}                                     \\ \cmidrule(lr){2-4} 
            \multirow{-2}{*}{Methods} & Deletion ($\downarrow$)    & Insertion ($\uparrow$) & $\mu$Fidelity ($\uparrow$)  \\ \midrule
            Saliency~\cite{simonyan2014deep} & 0.0682 & 0.6585 & 0.1409 \\ 
            Grad-CAM~\cite{selvaraju2020grad} & 0.0810 & 0.7224 & 0.1486 \\ 
            Integrated Gradients~\cite{sundararajan2017axiomatic} & 0.1693 & 0.5263 & 0.2832\\
            LIME~\cite{ribeiro2016should} & 0.1070  & 0.6812 & 0.0948 \\ 
            Kernel SHAP~\cite{lundberg2017unified} & 0.1016  & 0.6763 & 0.1452 \\ 
            RISE~\cite{petsiuk2018rise} & 0.0665 & 0.7193 & 0.2370 \\ 
            HSIC-Attribution~\cite{novello2022making} & 0.0647 & 0.6843 & 0.3435 \\ 
            Previous work~\cite{chen2024less}  & 0.0613 & 0.7262 & -  \\ 
            \rowcolor{gray!20}
            \textsc{LiMA} (w/ Segment Anything) & 0.0829 & 0.7106 & 0.3420 \\
            \rowcolor{gray!20}
            \textsc{LiMA} (w/ SLICO)   & \textbf{0.0560} & \textbf{0.7748} & \textbf{0.3547} \\ 
            \bottomrule
            \end{tabular}
        }
    \end{center}
    \vspace{-10pt}
\end{table}

\textbf{\revision{Faithfulness on Audio Recognition Models:}}
Explanation and transparency are important issues to ethical AI for music and audio~\cite{ma2024foundation}, yet there is limited existing research on explaining audio classification. In this section, we analyze audio classification explanations from the perspective of spectrograms, utilizing two multi-modal foundation models. For sub-region division, we adopt a patch-based approach.
Table~\ref{faithfulness_on_correct_vggsound_vit} presents the attribution results of our method compared to the baseline on the VGG-Sound~\cite{chen2020vggsound} dataset. On ImageBind, the Square-Grad~\cite{hooker2019benchmark} method is notably competitive, but our approach surpasses it by 56.0\% in Deletion AUC and 83.9\% in Insertion AUC when using a $10 \times 10$ patch sub-region division. On LanguageBind, the RISE~\cite{petsiuk2018rise} method performs well, yet our method exceeds it by 42.8\% in Deletion AUC and 79.6\% in Insertion AUC. Although Integrated Gradients~\cite{sundararajan2017axiomatic} achieves the best Deletion AUC score, our method delivers strong performance overall, with Integrated Gradients performing poorly on the Insertion AUC. We demonstrate the robust attribution capability of our method on the Audio modality, highlighting its adaptability across different modalities.
Fig.~\ref{vis:audio} presents the attribution visualization results on the spectrogram, highlighting the regions the model focuses on for classification. Future work will explore ways to enhance human understanding of these interpretation results.

\begin{table}[h]
    \caption{\revision{Faithfulness, absolute query counts, and wall-clock time for multimodal ViT backbones on VGG-Sound.}}
    \label{faithfulness_on_correct_vggsound_vit}
    \begin{center}
        \resizebox{0.48\textwidth}{!}{
            \begin{tabular}{l|cccc|cccc}
                \toprule
                \multirow{2}{*}{Methods} & \multicolumn{4}{c}{ImageBind~\cite{girdhar2023imagebind}} & \multicolumn{4}{c}{LanguageBind~\cite{zhu2024languagebind}} \\ \cmidrule(r){2-5} \cmidrule(r){6-9}
                  & Del. ($\downarrow$) & Ins. ($\uparrow$) & F/B & Time & Del. ($\downarrow$) & Ins. ($\uparrow$) & F/B & Time \\ \midrule
                Saliency~\cite{simonyan2014deep} & 0.0999 & 0.1823 & 1/1 & $0.089{\pm}0.001$ & 0.1024 & 0.2184 & 1/1 & $0.136{\pm}0.005$ \\
                Gradient-Input~\cite{shrikumar2016not} & 0.0326 & 0.1257 & 1/1 & $0.079{\pm}0.003$ & 0.0243 & 0.1388 & 1/1 & $0.134{\pm}0.004$ \\
                SmoothGrad~\cite{smilkov2017smoothgrad} & 0.0734 & 0.0955 & 80/80 & $2.549{\pm}0.167$ & 0.0535 & 0.1664 & 80/80 & $6.605{\pm}0.018$ \\
                Integrated Gradients~\cite{sundararajan2017axiomatic} & \textbf{0.0313} & 0.1292 & 80/80 & $2.180{\pm}0.012$ & \textbf{0.0165} & 0.2153 & 80/80 & $8.668{\pm}0.715$ \\
                VarGrad~\cite{seo2018noise} & 0.0755 & 0.2121 & 80/80 & $2.298{\pm}0.030$ & 0.0955 & 0.2643 & 80/80 & $6.548{\pm}0.052$ \\
                Square-Grad~\cite{hooker2019benchmark} & 0.0748 & 0.2120 & 80/80 & $2.309{\pm}0.008$ & 0.0888 & 0.2697 & 80/80 & $6.584{\pm}0.037$ \\
                LIME~\cite{ribeiro2016should} & 0.1375 & 0.1422 & 1,000/0 & $10.10{\pm}0.15$ & 0.1092 & 0.2164 & 1,000/0 & $31.52{\pm}0.51$ \\
                Kernel SHAP~\cite{lundberg2017unified} & 0.1316 & 0.1436 & 1,000/0 & $10.27{\pm}0.19$ & 0.1456 & 0.1785 & 1,000/0 & $29.93{\pm}0.33$ \\
                Occlusion~\cite{ancona2018towards} & 0.0707 & 0.1650 & 937/0 & $13.19{\pm}1.07$ & 0.0911 & 0.2052 & 1,051/0 & $33.96{\pm}0.37$ \\
                RISE~\cite{petsiuk2018rise} & 0.0782 & 0.1952 & 500/0 & $4.795{\pm}0.037$ & 0.0800 & 0.2879 & 500/0 & $15.22{\pm}0.16$ \\
                Sobol-Attribution~\cite{fel2021look} & 0.0881 & 0.1429 & 2,112/0 & $19.97{\pm}0.21$ & 0.0956 & 0.2235 & 2,112/0 & $61.69{\pm}0.42$ \\
                HSIC-Attribution~\cite{novello2022making} & 0.0891 & 0.1418 & 500/0 & $4.999{\pm}0.029$ & 0.0904 & 0.2270 & 500/0 & $14.42{\pm}0.09$ \\
                \rowcolor{gray!20}
                \textsc{LiMA} ($8 \times 8$) & 0.0425 & \textbf{0.4164} & 2,660/0 & $29.26{\pm}3.72$ & 0.0806 & \textbf{0.5362} & 2,660/0 & $103.05{\pm}5.61$ \\ 
                \rowcolor{gray!20}
                \textsc{LiMA} ($10 \times 10$) & 0.0329 & 0.3899 & 5,960/0 & $67.04{\pm}3.56$ & 0.0458 & 0.5172 & 5,960/0 & $276.29{\pm}5.96$ \\
                \bottomrule
                \end{tabular}
        }
    \end{center}
\end{table}

\begin{figure}[h]
    \centering  
    \begin{overpic}[width=0.48\textwidth,tics=8]{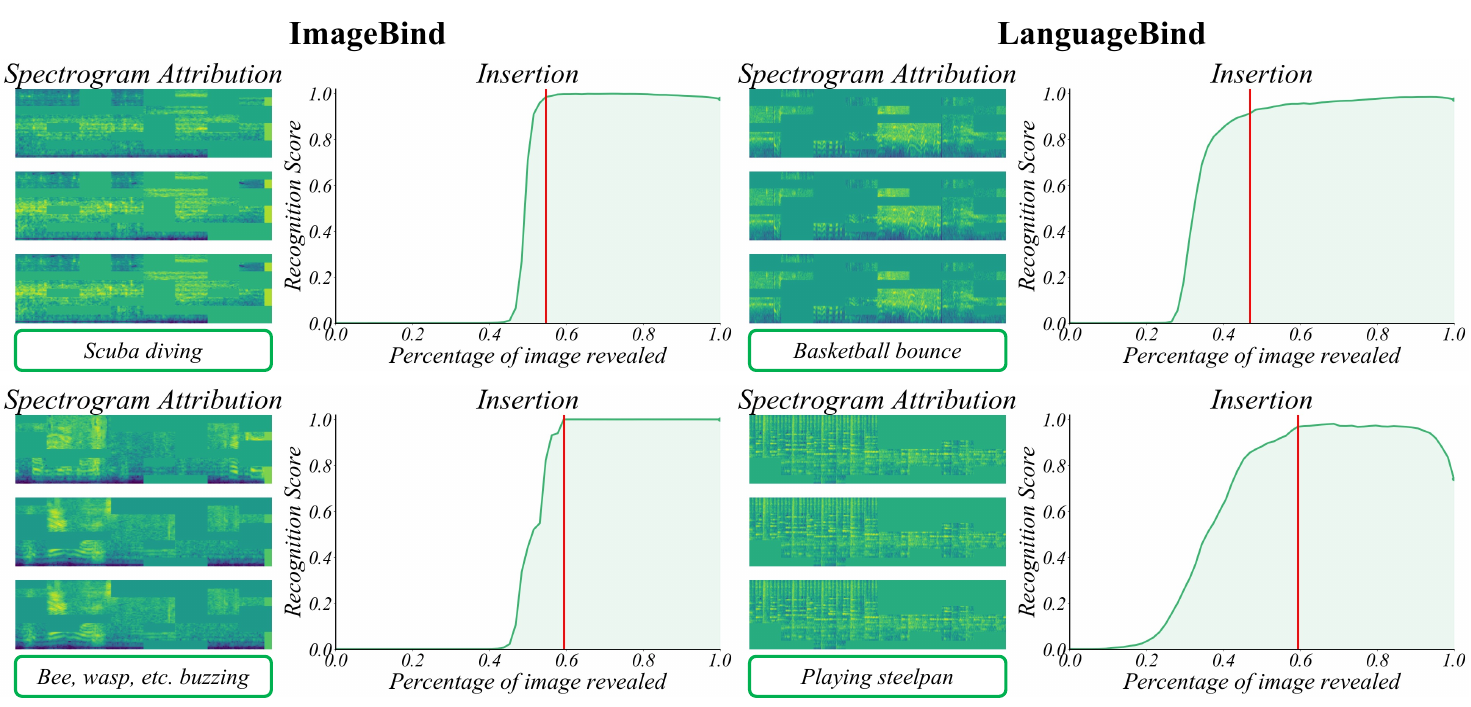} 
        \put (30.7,45.2) {\tiny \cite{girdhar2023imagebind}}
        \put (82.3,45.2) {\tiny \cite{zhu2024languagebind}}
    \end{overpic}  
    \caption{Attribution visualizations for audio classification on ImageBind and LanguageBind, highlighting the most critical spectrogram regions attributed to the model. Additionally, we provide the corresponding Insertion AUC curve.}  
    \label{vis:audio}  
\end{figure}

\begin{figure*}[!t]
    \centering
    \includegraphics[width=\textwidth]{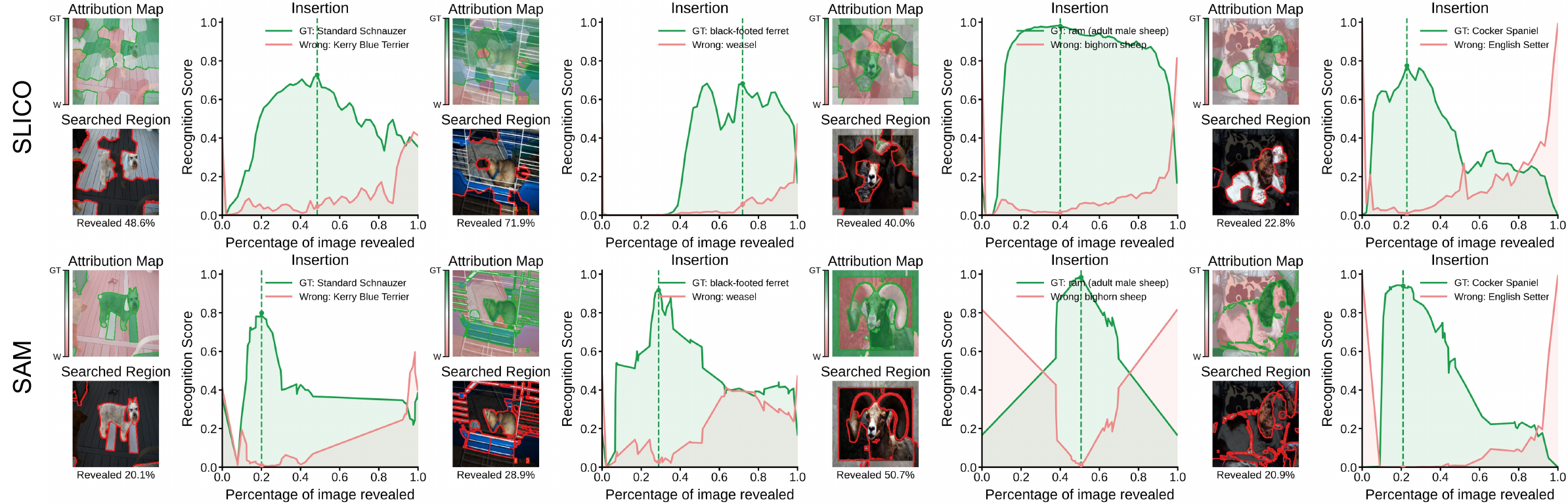} 
    \caption{\revision{Visualization of failure attribution for CLIP on representative misclassified samples. In the attribution maps, green regions tend to support the ground-truth (GT) class, whereas red regions tend to support the originally predicted incorrect class. In the insertion curves, the green and red lines denote the recognition scores of the GT and incorrect classes, respectively. Failure attribution aims to identify a compact region whose revelation changes the prediction from the incorrect class back to the GT class, i.e., the GT recognition score becomes higher than that of the incorrect class. The searched region corresponds to the image content revealed at the turning point (dashed vertical line), representing the minimal region required to recover the correct prediction.}}
    \label{vis:debug_multimodal}  
\end{figure*}

\revision{Tables~\ref{faithfulness_on_correct_imagenet}–\ref{faithfulness_on_correct_imagenet_vim} show that, under the default setting $N=49$ and $n_p=8$, \textsc{LiMA} requires 1,660 forward queries and no backward pass, with per-image runtime ranging from about 5.0 seconds on ResNet-101 to 199.4 seconds on LanguageBind. Its query cost is higher than RISE and HSIC-Attribution (500) and Kernel SHAP (1,000), but substantially lower than iGOS++ in the measured settings. Table~\ref{faithfulness_on_correct_vggsound_vit} also illustrates the granularity–efficiency trade-off: increasing the spectrogram partition from $8\times8$ to $10\times10$ raises the query count from approximately 2,660 to 5,960, together with the runtime.}

\revision{For $N=49$, standard greedy ranking requires 2,500 queries, whereas \textsc{LiMA} has an upper bound of 1,660, yielding a 33.6\% reduction while maintaining strong faithfulness. Therefore, \textsc{LiMA} should be viewed as a high-faithfulness, moderate-query method rather than the universally cheapest attribution method. Batching reduces local latency but not the number of model evaluations charged by proprietary providers. Further discussion is provided in Appendix~\ref{supp-supp:deployment_scope}.}

\subsection{\revision{Failure Attribution for Incorrect Predictions}}\label{error_debug}

\revision{In this section, we conduct failure attribution to analyze misclassified images and assess whether the attribution algorithm can accurately identify the causes of the model's prediction errors. Specifically, our goal is to perform a counterfactual analysis by masking the local regions responsible for the incorrect prediction, such that the model's output can be corrected from the original wrong class to the ground-truth class.} We use the Average Highest Confidence within a limited search region (\textit{e.g.}, reveal up to 25\% of the image region) and the Insertion AUC score as our evaluation metrics. The Deletion AUC score is not considered, as the initial prediction scores for the misclassified samples are already very low. We validate our method on the ImageNet~\cite{deng2009imagenet}, CUB-200-2011~\cite{welindercaltech}, VGG-Sound~\cite{chen2020vggsound}, and LC25000 (Lung)~\cite{borkowski2019lung} datasets. The ground truth labels are given. 

\subsubsection{\revision{Failure Attribution for Multimodal Foundation Models}}

We first analyzed the samples misclassified by the multimodal foundation models. Table~\ref{discover_mistake_multimodal} presents the attribution results for these misclassified samples, using the ground truth labels for interpretation. Our approach achieved stunning results. When utilizing Segment Anything (SAM)~\cite{kirillov2023segment} as the sub-region division algorithm, our method outperforms RISE~\cite{petsiuk2018rise} by 61.5\% on the CLIP (ViT-L) model and 52.4\% on ImageBind, while surpassing EAC~\cite{sun2023explain} by 63.8\% on LanguageBind in terms of the average highest confidence during the global search. For Insertion metrics, our method exceeds RISE~\cite{petsiuk2018rise} by 127.2\% on CLIP (ViT-L) and 89.4\% on ImageBind, and outperforms EAC~\cite{sun2023explain} by 101.9\% on LanguageBind. 

\begin{table}[!h]
    \caption{Evaluation \revision{of failure attribution} for different multimodal models on the ImageNet dataset.}
    \label{discover_mistake_multimodal}
    \begin{center}
    \resizebox{0.48 \textwidth}{!}{
        \begin{tabular}{cl|ccccc}
        \toprule
        & \multicolumn{1}{c|}{} & \multicolumn{4}{c}{Average highest confidence ($\uparrow$)} &                             \\
        \multirow{-2}{*}{Backbones} & \multicolumn{1}{c|}{\multirow{-2}{*}{Methods}}       & (0-25\%)                  & (0-50\%)                  & (0-75\%)                  & (0-100\%)                 & \multirow{-2}{*}{Insertion ($\uparrow$)} \\ \midrule
        & RISE~\cite{petsiuk2018rise} & 0.2710 & 0.3985 & 0.4558 & 0.4879 & 0.1827 \\
        & HSIC-Attribution~\cite{novello2022making} & 0.2148 & 0.2723 & 0.3051 & 0.3462 & 0.1062 \\
        & ViT-CX~\cite{xie2023vit} & 0.1268 & 0.2367 & 0.3137 & 0.3902 & 0.1236 \\
        & EAC~\cite{sun2023explain} (w/ Segment Anything) & 0.1004 & 0.1861 & 0.2561 & 0.3397 & 0.1086 \\
        & Grad-ECLIP~\cite{zhao2024gradient} & 0.2202 & 0.3404 & 0.3767 & 0.4123 & 0.1397 \\
        & Previous work~\cite{chen2024less} & 0.2012 & 0.3073 & 0.3582 & 0.3989 & 0.2017 \\
        & \cellcolor[HTML]{EFEFEF}\textsc{LiMA} (w/ SLICO) & \cellcolor[HTML]{EFEFEF}0.5677 & \cellcolor[HTML]{EFEFEF}0.7030 & \cellcolor[HTML]{EFEFEF}0.7479 & \cellcolor[HTML]{EFEFEF}0.7583 & \cellcolor[HTML]{EFEFEF}\textbf{0.4789}    \\
        \multirow{-8}{*}{CLIP (ViT-L)~\cite{radford2021learning}} & \cellcolor[HTML]{EFEFEF}\textsc{LiMA} (w/ Segment Anything) & \cellcolor[HTML]{EFEFEF}\textbf{0.6440} & \cellcolor[HTML]{EFEFEF}\textbf{0.7307} & \cellcolor[HTML]{EFEFEF}\textbf{0.7676} & \cellcolor[HTML]{EFEFEF}\textbf{0.7878} & \cellcolor[HTML]{EFEFEF}0.4151 \\ \midrule
        & RISE~\cite{petsiuk2018rise} & 0.2427 & 0.4045 & 0.4757 & 0.5164 & 0.2237 \\
        & HSIC-Attribution~\cite{novello2022making} & 0.1711 & 0.2726 & 0.3260 & 0.3636 & 0.2180 \\
        & ViT-CX~\cite{xie2023vit} & 0.1299 & 0.2593 & 0.3532 & 0.4261 & 0.1129 \\
        & EAC~\cite{sun2023explain} (w/ Segment Anything) & 0.1064 & 0.1901 & 0.2658 & 0.3463 & 0.1124 \\
        & Previous work~\cite{chen2024less} & 0.4436 & 0.5865 & 0.6157 & 0.6348 & 0.3929 \\
        & \cellcolor[HTML]{EFEFEF}\textsc{LiMA} (w/ SLICO) & \cellcolor[HTML]{EFEFEF}0.5351 & \cellcolor[HTML]{EFEFEF}0.6816 & \cellcolor[HTML]{EFEFEF}0.7445 & \cellcolor[HTML]{EFEFEF}0.7557 & \cellcolor[HTML]{EFEFEF}\textbf{0.4681}    \\
        \multirow{-7}{*}{ImageBind~\cite{girdhar2023imagebind}} & \cellcolor[HTML]{EFEFEF}\textsc{LiMA} (w/ Segment Anything) & \cellcolor[HTML]{EFEFEF}\textbf{0.6639} & \cellcolor[HTML]{EFEFEF}\textbf{0.7487} & \cellcolor[HTML]{EFEFEF}\textbf{0.7736} & \cellcolor[HTML]{EFEFEF}\textbf{0.7872} & \cellcolor[HTML]{EFEFEF}0.4237    \\ \midrule
        & RISE~\cite{petsiuk2018rise} & 0.0967 & 0.1902 & 0.2530 & 0.2920 & 0.1500 \\
        & HSIC-Attribution~\cite{novello2022making} & 0.0787 & 0.1380 & 0.1779 & 0.2112 & 0.0861 \\
        & ViT-CX~\cite{xie2023vit} & 0.0414 & 0.1123 & 0.1628 & 0.2348 & 0.0816 \\
        & EAC~\cite{sun2023explain} (w/ Segment Anything) & 0.1404 & 0.2802 & 0.3755 & 0.4245 & 0.1721 \\
        & Previous work~\cite{chen2024less} & 0.2050 & 0.2933 & 0.3360 & 0.3506 & 0.1812 \\
        & \cellcolor[HTML]{EFEFEF}\textsc{LiMA} (w/ SLICO) & \cellcolor[HTML]{EFEFEF}0.3843 & \cellcolor[HTML]{EFEFEF}0.5311 & \cellcolor[HTML]{EFEFEF}0.5824 & \cellcolor[HTML]{EFEFEF}0.5978 & \cellcolor[HTML]{EFEFEF}\textbf{0.3483}    \\
        \multirow{-7}{*}{LanguageBind~\cite{zhu2024languagebind}} & \cellcolor[HTML]{EFEFEF}\textsc{LiMA} (w/ Segment Anything) & \cellcolor[HTML]{EFEFEF}\textbf{0.5401} & \cellcolor[HTML]{EFEFEF}\textbf{0.6399} & \cellcolor[HTML]{EFEFEF}\textbf{0.6753} & \cellcolor[HTML]{EFEFEF}\textbf{0.6955} & \cellcolor[HTML]{EFEFEF}0.3475  \\  \bottomrule
        \end{tabular}
    }
    \end{center}
\end{table}

This is primarily due to the presence of potential regions in the samples where the model predicts incorrectly, which exerts a strong influence on the model’s decision, leading to complex interactions between features. Traditional attribution algorithms struggle to effectively identify these regions. In contrast, our method leverages marginal contribution scores, enabling it to attribute regions that are more relevant to the target class and pinpoint the causes of the model’s incorrect predictions. Therefore, our method achieves more competitive results than the baseline method on such samples.

\revision{Fig.~\ref{vis:debug_multimodal} shows some visualization of failure attribution. The insertion curves show how the ground-truth and originally predicted incorrect-class responses change as the ranked sub-regions are revealed. The highlighted prefix supports reversal to the ground-truth prediction under this intervention, whereas later revealed evidence is associated with recovery of the originally incorrect-class response. These regions are therefore interpreted as model-level counterfactual evidence, not as formally identified causes.}

Additionally, we observed that for the average highest confidence metric, using SAM as the sub-region division algorithm yields better results, while the superpixel-based division method performs better for the Insertion metric. This is mainly because the sub-regions generated by SAM have stronger semantic coherence, giving it an advantage when identifying regions of high model interest. However, SAM sometimes produces sub-regions that are too large and uncontrollable, leading to mixed positive and negative regions in one sub-region. As a result, it performs slightly worse than the superpixel-based method for the Insertion metric.

\begin{table}[]
    \caption{Evaluation \revision{of failure attribution} for different multimodal models on the VGG-Sound dataset.}
    \label{discover_mistake_multimodal_vggsound}
    \begin{center}
    \resizebox{0.48 \textwidth}{!}{
        \begin{tabular}{cl|ccccc}
        \toprule
        & \multicolumn{1}{c|}{} & \multicolumn{4}{c}{Average highest confidence ($\uparrow$)} &                             \\
        \multirow{-2}{*}{Backbones} & \multicolumn{1}{c|}{\multirow{-2}{*}{Methods}}       & (0-25\%)                  & (0-50\%)                  & (0-75\%)                  & (0-100\%)                 & \multirow{-2}{*}{Insertion ($\uparrow$)} \\ \midrule
        & Integrated Gradients~\cite{sundararajan2017axiomatic} & 0.0081 & 0.0318 & 0.0776 & 0.3904 & 0.0460 \\
        & RISE~\cite{petsiuk2018rise} & 0.0064 & 0.0533 & 0.1547 & 0.2546 & 0.0385 \\
        & HSIC-Attribution~\cite{novello2022making} & 0.0107 & 0.0441 & 0.0609 & 0.1657 & 0.0201 \\
        \multirow{-4}{*}{ImageBind~\cite{girdhar2023imagebind}} & \cellcolor[HTML]{EFEFEF}\textsc{LiMA} ($10 \times 10$) & \cellcolor[HTML]{EFEFEF}\textbf{0.0955} & \cellcolor[HTML]{EFEFEF}\textbf{0.1978} & \cellcolor[HTML]{EFEFEF}\textbf{0.3992} & \cellcolor[HTML]{EFEFEF}\textbf{0.5028} & \cellcolor[HTML]{EFEFEF}\textbf{0.1813} \\ \midrule
        & Integrated Gradients~\cite{sundararajan2017axiomatic} & 0.0574 & 0.1092 & 0.1747 & 0.4254 & 0.1153 \\
        & RISE~\cite{petsiuk2018rise} & 0.0478 & 0.1415 & 0.2283 & 0.2799 & 0.1001 \\
        & HSIC-Attribution~\cite{novello2022making} & 0.0365 & 0.0711 & 0.1213 & 0.1821 & 0.0532 \\
        \multirow{-4}{*}{LanguageBind~\cite{zhu2024languagebind}} & \cellcolor[HTML]{EFEFEF}\textsc{LiMA} ($10 \times 10$) & \cellcolor[HTML]{EFEFEF}\textbf{0.2656} & \cellcolor[HTML]{EFEFEF}\textbf{0.4195} & \cellcolor[HTML]{EFEFEF}\textbf{0.5543} & \cellcolor[HTML]{EFEFEF}\textbf{0.5767} & \cellcolor[HTML]{EFEFEF}\textbf{0.3022} \\  \bottomrule
        \end{tabular}
    }
    \end{center}
\end{table}

\begin{figure}[]
    \centering  
    \begin{overpic}[width=0.48\textwidth,tics=8]{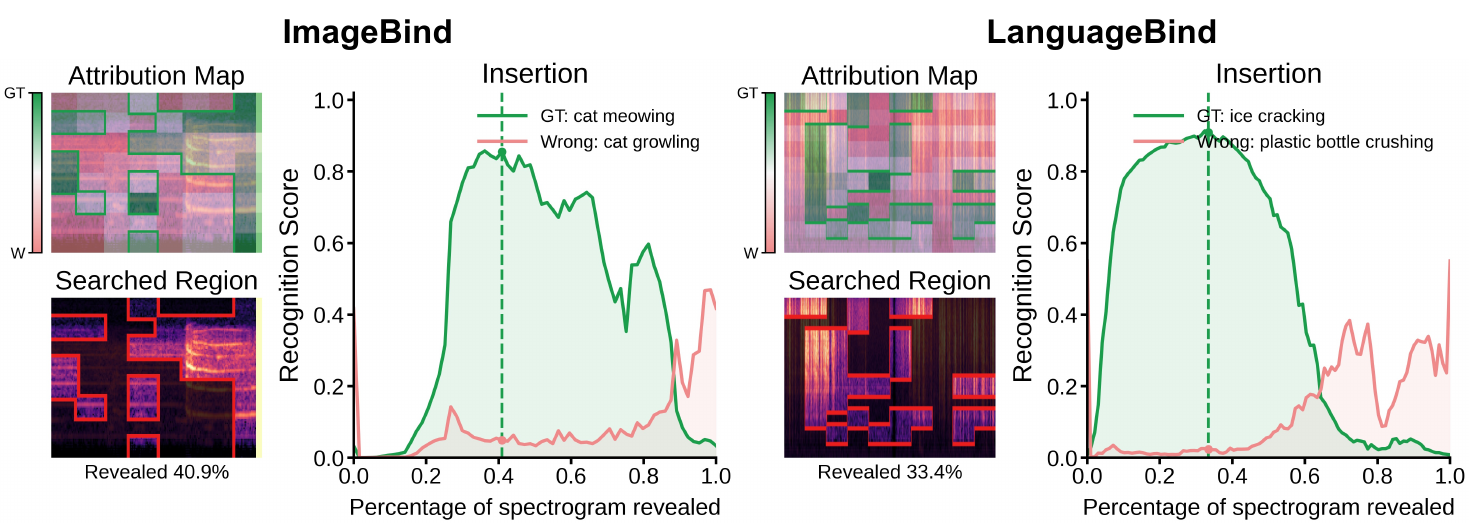} 
        \put (30.5,32.9) {\tiny \cite{girdhar2023imagebind}}
        \put (81.8,32.9) {\tiny \cite{zhu2024languagebind}}
    \end{overpic}
    \caption{Visualization of failure attribution for representative incorrect audio predictions.}  
    \label{vis:debug_audio}  
\end{figure}

\subsubsection{\revision{Failure Attribution for Audio Foundation Models}}

\revision{Next, we validate failure attribution in the audio modality by analyzing which revealed spectrogram regions are associated with prediction reversal for misclassified samples.}
Table~\ref{discover_mistake_multimodal_vggsound} presents the attribution results on ImageBind~\cite{girdhar2023imagebind} and LanguageBind~\cite{zhu2024languagebind}. Our method outperforms RISE by 28.8\% and 35.6\% in average highest confidence on the ImageBind and LanguageBind models, respectively. For the Insertion AUC score, it surpasses RISE by 294.1\% on ImageBind and 162.1\% on LanguageBind, demonstrating the strong generalizability of our method across different modalities. Fig.~\ref{vis:debug_audio} illustrates the failure attribution results of our method on misclassified audio samples. By correctly searching the relevant spectral region, our method enables the model to make accurate predictions and reveals the spectrogram region responsible for the corrected prediction.

\begin{table}[]
    \caption{Evaluation \revision{of failure attribution} for the medical multimodal model QuiltNet on the LC25000 (Lung) dataset.}
    \label{discover_mistake_multimodal_medical}
    \begin{center}
    \resizebox{0.48 \textwidth}{!}{
        \begin{tabular}{cl|ccccc}
        \toprule
        & \multicolumn{1}{c|}{} & \multicolumn{4}{c}{Average highest confidence ($\uparrow$)} &                             \\
        \multirow{-2}{*}{Backbone} & \multicolumn{1}{c|}{\multirow{-2}{*}{Methods}}       & (0-25\%)                  & (0-50\%)                  & (0-75\%)                  & (0-100\%)                 & \multirow{-2}{*}{Insertion ($\uparrow$)} \\ \midrule
        & Integrated Gradients~\cite{sundararajan2017axiomatic} & 0.2989 & 0.2989 & 0.2989 & 0.3012 & 0.1344 \\
        & Kernel SHAP~\cite{lundberg2017unified} & 0.1957 & 0.1962 & 0.1965 & 0.2024 & 0.1359 \\
        & RISE~\cite{petsiuk2018rise} & 0.4986 & 0.5109 & 0.5163 & 0.5171 & 0.1470 \\
        & HSIC-Attribution~\cite{novello2022making} & 0.4858 & 0.4927 & 0.5005 & 0.5013 & 0.1468 \\
        & ViT-CX~\cite{xie2023vit} & 0.3497 & 0.3595 & 0.3637 & 0.3658 & 0.1419 \\
        \multirow{-6}{*}{QuiltNet (ViT/B-32)~\cite{ikezogwo2024quilt}} & \cellcolor[HTML]{EFEFEF}\textsc{LiMA} (w/ SLICO) & \cellcolor[HTML]{EFEFEF}\textbf{0.6215} & \cellcolor[HTML]{EFEFEF}\textbf{0.6764} & \cellcolor[HTML]{EFEFEF}\textbf{0.6901} & \cellcolor[HTML]{EFEFEF}\textbf{0.6915} & \cellcolor[HTML]{EFEFEF}\textbf{0.3401}    \\ 
        \bottomrule
        \end{tabular}
    }
    \end{center}
\end{table}

\begin{figure}[]
    \centering  
    \begin{overpic}[width=0.48\textwidth,tics=8]{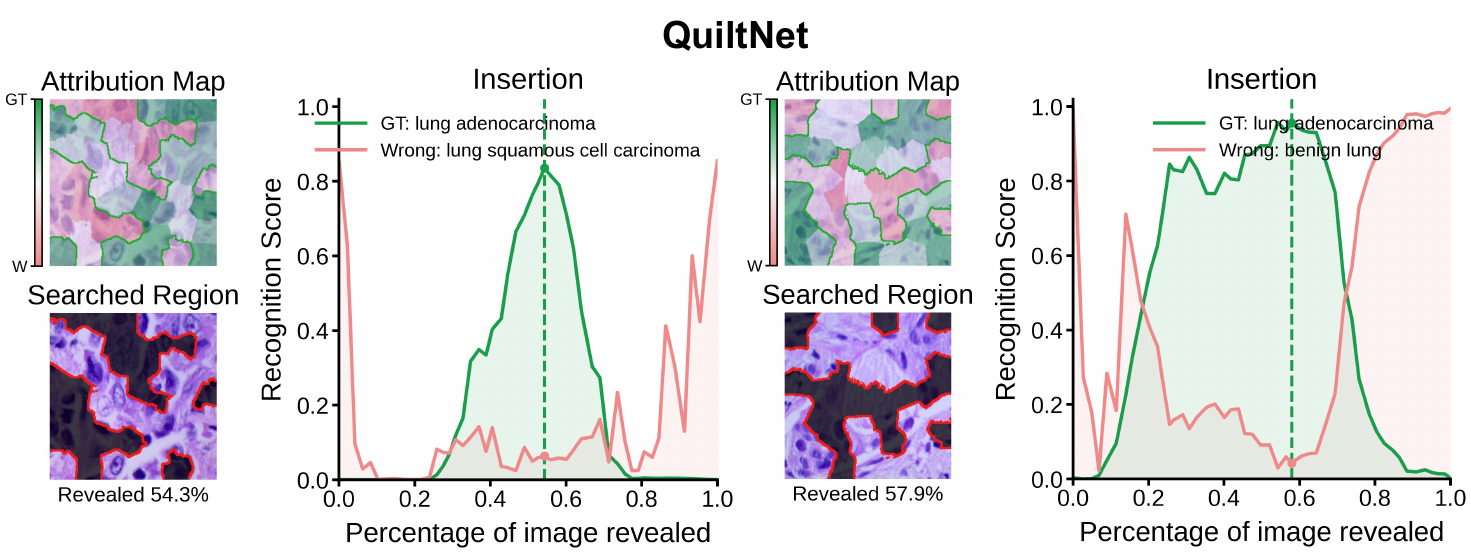} 
        \put (54.8,34.5) {\tiny \cite{ikezogwo2024quilt}}
    \end{overpic}
    \caption{\revision{Visualization of failure attribution for medical foundation model QuiltNet on representative misclassified samples.}}  
    \label{vis:debug_medical}  
\end{figure}

\subsubsection{\revision{Failure Attribution for Medical Foundation Model}}

We also evaluated our method’s ability to attribute errors on the medical multimodal foundation model, QuiltNet~\cite{ikezogwo2024quilt}. Table~\ref{discover_mistake_multimodal_medical} presents the attribution results. Compared to the HSIC-Attribution method, our approach improves the average highest confidence by 39.2\% and the Insertion AUC score by 137.5\%, demonstrating strong attribution capabilities in the medical imaging domain. Fig.~\ref{vis:debug_medical} shows images misclassified as Benign Lung by QuiltNet. Our method identifies key regions that trigger a strong response from the model for the correct category, thus interpreting the sources of the prediction errors.

\subsubsection{\revision{Failure Attribution for Fine-graininess Recognition Model}}

Finally, we evaluated our method’s ability to detect errors on fine-grained datasets across three CNN models. Table~\ref{discover_mistake_cub} presents the results for misclassified samples from the CUB-200-2011~\cite{welindercaltech} validation set, where our method demonstrates significant improvements. Compared to the state-of-the-art HSIC-Attribution method, in the global search interval (0-100\%), the average highest confidence identified by our method increased by 217.5\% on ResNet-101, 152.2\% on MobileNetV2, and 123.9\% on EfficientNetV2-M, when using SAM as the sub-region division method. Similarly, the Insertion AUC score saw substantial improvements, increasing by 178.7\% on ResNet-101, 101.7\% on MobileNetV2, and 139.4\% on EfficientNetV2-M. In addition, compared to the previous version of our method, we achieved improvements of 75.4\%, 36.9\%, and 81.5\% in the average highest confidence on the three CNN models, and 127.4\%, 71.5\%, and 120.7\% in the Insertion AUC score, respectively. See Appendix~\ref{cub_debug_vis} for the visualization results.

\begin{table}[!h]
    \caption{Evaluation of \revision{failure attribution for incorrect predictions across convolutional neural network backbones on the CUB-200-2011 dataset.}}
    \label{discover_mistake_cub}
    \begin{center}
    \resizebox{0.48 \textwidth}{!}{
        \begin{tabular}{cl|ccccc}
        \toprule
        & \multicolumn{1}{c|}{} & \multicolumn{4}{c}{Average highest confidence ($\uparrow$)} &                             \\
        \multirow{-2}{*}{Backbones} & \multicolumn{1}{c|}{\multirow{-2}{*}{Methods}}       & (0-25\%)                  & (0-50\%)                  & (0-75\%)                  & (0-100\%)                 & \multirow{-2}{*}{Insertion ($\uparrow$)} \\ \midrule
        & Grad-CAM++~\cite{chattopadhay2018grad}  & 0.1988 & 0.2447 & 0.2544 & 0.2647 & 0.1094 \\
        & Score-CAM~\cite{wang2020score} & 0.1896 & 0.2323 & 0.2449 & 0.2510 & 0.1073 \\
        & Kernel SHAP~\cite{lundberg2017unified} & 0.0083 & 0.0247 & 0.0574 & 0.2642 & 0.1008\\
        & RISE~\cite{petsiuk2018rise} & 0.2406 & 0.3032 & 0.3316 & 0.3807 & 0.2238 \\
        & HSIC-Attribution~\cite{novello2022making} & 0.1709 & 0.2091 & 0.2250 & 0.2493 & 0.1446 \\
        & Previous work~\cite{chen2024less} & 0.2430 & 0.3519 & 0.3984 & 0.4513 & 0.1772  \\
        & \cellcolor[HTML]{EFEFEF}\textsc{LiMA} (w/ SLICO) & \cellcolor[HTML]{EFEFEF}0.3683 & \cellcolor[HTML]{EFEFEF}\textbf{0.5016} & \cellcolor[HTML]{EFEFEF}0.5501 & \cellcolor[HTML]{EFEFEF}0.6003 & \cellcolor[HTML]{EFEFEF}0.3694    \\
        \multirow{-8}{*}{ResNet-101~\cite{he2016deep}} & \cellcolor[HTML]{EFEFEF}\textsc{LiMA} (w/ Segment Anything) & \cellcolor[HTML]{EFEFEF}\textbf{0.3874} & \cellcolor[HTML]{EFEFEF}0.4957 & \cellcolor[HTML]{EFEFEF}\textbf{0.5975} & \cellcolor[HTML]{EFEFEF}\textbf{0.7916} & \cellcolor[HTML]{EFEFEF}\textbf{0.4030}    \\ \midrule
        & Grad-CAM++~\cite{chattopadhay2018grad}  & 0.1584 & 0.2820 & 0.3223 & 0.3462 & 0.1284 \\
        & Score-CAM~\cite{wang2020score} & 0.1574 & 0.2456 & 0.2948 & 0.3141 & 0.1195 \\
        & Kernel SHAP~\cite{lundberg2017unified} & 0.0525 & 0.1367 & 0.2321 & 0.4142 & 0.2511 \\
        & RISE~\cite{petsiuk2018rise} & 0.2340 & 0.3101 & 0.3355 & 0.4280 & 0.2573 \\
        & HSIC-Attribution~\cite{novello2022making} & 0.1648 & 0.2190 & 0.2415 & 0.2914 & 0.1635 \\
        & Previous work~\cite{chen2024less} & 0.2460 & 0.4142 & 0.4913 & 0.5367 & 0.1922    \\
        & \cellcolor[HTML]{EFEFEF}\textsc{LiMA} (w/ SLICO) & \cellcolor[HTML]{EFEFEF}0.3466 & \cellcolor[HTML]{EFEFEF}0.4686 & \cellcolor[HTML]{EFEFEF}0.5462 & \cellcolor[HTML]{EFEFEF}0.6278 & \cellcolor[HTML]{EFEFEF}\textbf{0.3448}    \\
        \multirow{-8}{*}{MobileNetV2~\cite{sandler2018mobilenetv2}} & \cellcolor[HTML]{EFEFEF}\textsc{LiMA} (w/ Segment Anything) & \cellcolor[HTML]{EFEFEF}\textbf{0.5339} & \cellcolor[HTML]{EFEFEF}\textbf{0.6098} & \cellcolor[HTML]{EFEFEF}\textbf{0.6611} & \cellcolor[HTML]{EFEFEF}\textbf{0.7349} & \cellcolor[HTML]{EFEFEF}0.3297  \\ \midrule
        & Grad-CAM++~\cite{chattopadhay2018grad} & 0.2338 & 0.2549 & 0.2598 & 0.2659 & 0.1605 \\
        & Score-CAM~\cite{wang2020score} & 0.2126 & 0.2327 & 0.2375 & 0.2403 & 0.1572 \\
        & Kernel SHAP~\cite{lundberg2017unified} & 0.0123 & 0.0670 & 0.1489 & 0.3357 & 0.0787 \\
        & RISE~\cite{petsiuk2018rise} & 0.2958 & 0.3339 & 0.3473 & 0.3697 & 0.2274 \\
        & HSIC-Attribution~\cite{novello2022making} & 0.2418 & 0.2561 & 0.2615 & 0.2679 & 0.1611 \\
        & Previous work~\cite{chen2024less} & 0.2616 & 0.3117 & 0.3235 & 0.3306 & 0.1748 \\
        & \cellcolor[HTML]{EFEFEF}\textsc{LiMA} (w/ SLICO) & \cellcolor[HTML]{EFEFEF}0.3670 & \cellcolor[HTML]{EFEFEF}0.4575 & \cellcolor[HTML]{EFEFEF}0.4799 & \cellcolor[HTML]{EFEFEF}0.4859 & \cellcolor[HTML]{EFEFEF}0.3518    \\
        \multirow{-8}{*}{EfficientNetV2-M~\cite{tan2021efficientnetv2}}   & \cellcolor[HTML]{EFEFEF}\textsc{LiMA} (w/ Segment Anything) & \cellcolor[HTML]{EFEFEF}\textbf{0.4138} & \cellcolor[HTML]{EFEFEF}\textbf{0.5035} & \cellcolor[HTML]{EFEFEF}\textbf{0.5543} & \cellcolor[HTML]{EFEFEF}\textbf{0.5999} & \cellcolor[HTML]{EFEFEF}\textbf{0.3857}    \\ \bottomrule
        \end{tabular}
    }
    \end{center}
\end{table}

\revision{On the evaluated CUB-200-2011 samples, \textsc{LiMA} shows a larger advantage over the baselines for misclassified inputs than for correctly classified inputs. This observation is consistent with the context-dependent response transitions analyzed in Section~\ref{method_analysis}.}

\subsection{Ablation Study}\label{ablation_study}

In this section, we provide a comprehensive ablation study of the various components of our approach, conducted on the ImageBind model and the ImageNet dataset, and the sub-region division algorithm is SLICO.

\subsubsection{Ablation of the Bidirectional Greedy Search Algorithm}

\revision{Section~\ref{alg_fast} discusses the computational complexity and candidate-reuse mechanism of the proposed bidirectional greedy search without asserting a submodular approximation bound. Table~\ref{ablation_search} empirically evaluates the resulting quality--efficiency trade-off. Although naive forward greedy ranking achieves the strongest faithfulness scores, it requires substantially more execution time. Bidirectional search reduces the measured attribution time by reusing forward evaluations to construct a restricted reverse candidate pool. The pool size $n_p$ controls this trade-off: larger values increase search time but generally improve attribution quality. For the default $N=49$ and $n_p=8$, the pseudocode performs 805 candidate-objective evaluations, compared with 1,225 for a complete naive forward ranking. These algorithmic evaluation counts are distinguished from external API queries, whose number can depend on the model interface and caching strategy. In our ImageBind--ImageNet experiment on an NVIDIA 4090 GPU, $n_p=8$ provides a practical balance and reduces the measured attribution time from 32 seconds to 20 seconds, or 62.5\% of the naive greedy runtime.}

\begin{table}[h]
    \caption{Ablation study on the bidirectional greedy search algorithm for ImageBind on the ImageNet dataset.}
    \label{ablation_search}
    \begin{center}
    \resizebox{0.48 \textwidth}{!}{
        \begin{tabular}{cc|cccc}
        \toprule
        \multirow{2}{*}{\begin{tabular}[c]{c}Bi-directional\\ subset\end{tabular}} & \multirow{2}{*}{\begin{tabular}[c]{l}Candidate pool\\ size $n_p$\end{tabular}} & \multicolumn{3}{c}{Faithfulness metrics} & \multirow{2}{*}{Exec. time (s)} \\ 
         & & Deletion ($\downarrow$) & Insertion ($\uparrow$) & $\mu$Fidelity ($\uparrow$) &  \\ \midrule
        \XSolidBrush & 0 & \textbf{0.1260} & \textbf{0.8099} & \textbf{0.3669} & 32 \\
        \Checkmark & 1 & 0.2048 & 0.6445 & 0.3344 & 16 \\
        \Checkmark & 4 & 0.1309 & 0.7303 & 0.3508 & 18 \\
        \Checkmark & 8 & 0.1280 & 0.7463 & 0.3526 & 20 \\
        \Checkmark & 12 & 0.1271 & 0.7517 & 0.3551 & 22 \\
        \Checkmark & 16 & 0.1269 & 0.7542 & 0.3471 & 25 \\
        \bottomrule
        \end{tabular}
    }
    \end{center}
\end{table}

\subsubsection{\revision{Ablation of the Composite Objective Function}}

\revision{We analyze the empirical contribution of each score in the composite objective on both correctly and incorrectly predicted samples.}

\revision{\textbf{Impact on correctly predicted samples:} As shown in Table~\ref{ablation_score_function}, a single criterion is insufficient to achieve the strongest attribution faithfulness. Combining complementary criteria generally improves the Insertion--Deletion trade-off, and removing objective components changes the resulting scores. Consistency and collaboration have the largest empirical influence, whereas confidence and effectiveness act as weaker auxiliary criteria; accordingly, the first two terms receive larger balancing coefficients. These ablations support the practical objective design but are not interpreted as evidence of submodularity.}

\begin{table}[h]
    \caption{revision{Ablation study of the composite objective components for ImageBind on the ImageNet dataset.}}
    \label{ablation_score_function}
    \begin{center}
        \resizebox{0.48\textwidth}{!}{
            \begin{tabular}{cccc|cc}
            \toprule
            \multicolumn{4}{c|}{\revision{Objective Function}}                                        & \multirow{3}{*}{Deletion ($\downarrow$)} & \multirow{3}{*}{Insertion ($\uparrow$)} \\ 
            Cons. Score & Colla. Score & Conf. Score & Eff. Score &   &   \\
            (Equation~\ref{consistency_function}) & (Equation~\ref{collaboration_function}) & (Equation~\ref{confidence_function}) & (Equation~\ref{r_function}) &   & \\ \midrule
            \Checkmark   & \XSolidBrush & \XSolidBrush & \XSolidBrush & 0.2567 &   0.6324 \\
            \XSolidBrush & \Checkmark & \XSolidBrush & \XSolidBrush & 0.1031 & 0.5598 \\
            \XSolidBrush & \XSolidBrush & \Checkmark & \XSolidBrush & 0.3416 & 0.3609 \\
            \XSolidBrush & \XSolidBrush & \XSolidBrush & \Checkmark & 0.3267 & 0.3909 \\
            \midrule
            \Checkmark & \Checkmark & \XSolidBrush & \XSolidBrush & 0.1290 & 0.7403 \\
            \XSolidBrush & \Checkmark & \Checkmark & \XSolidBrush & \textbf{0.1022} & 0.5520 \\
            \XSolidBrush & \XSolidBrush & \Checkmark & \Checkmark & 0.2861 & 0.4465 \\
            \midrule
            \XSolidBrush & \Checkmark & \Checkmark & \Checkmark & 0.1115 & 0.5595 \\
            \Checkmark & \XSolidBrush & \Checkmark & \Checkmark & 0.2570 & 0.6432 \\
            \Checkmark & \Checkmark & \XSolidBrush & \Checkmark &  0.1293 & 0.7436 \\
            \Checkmark & \Checkmark & \Checkmark & \XSolidBrush & 0.1333 & 0.7442 \\
            \Checkmark & \Checkmark & \Checkmark & \Checkmark &  0.1280  & \textbf{0.7463} \\
            \bottomrule
            \end{tabular}
        }
    \end{center}\vspace{-10pt}
\end{table}

\textbf{Impact on incorrectly predicting samples:} We primarily focus on the impact of the consistency score and collaboration score on the attribution of misclassified samples. As shown in Table~\ref{discover_mistake_ablation}, removing any score function leads to a decrease in both the average highest confidence and the Insertion AUC score. This indicates that even for attribution on misclassified samples, both the consistency and collaboration scores are essential. Maximum performance is only achieved when both scores are used together.

\begin{table}[h]
    \caption{\revision{Ablation study of the composite objective components for incorrectly predicted samples using ImageBind on ImageNet.}}
    \label{discover_mistake_ablation}
    \begin{center}
    \resizebox{0.48\textwidth}{!}{
        \begin{tabular}{cc|ccccc}
        \toprule
        Cons. Score & Colla. Score & \multicolumn{4}{c}{Average highest confidence ($\uparrow$)} & \multirow{2}{*}{Insertion ($\uparrow$)} \\ 
        (Equation~\ref{consistency_function}) & (Equation~\ref{collaboration_function}) & (0-25\%)  & (0-50\%) & (0-75\%)  & (0-100\%) \\ \midrule
        \XSolidBrush & \Checkmark  & 0.0821 & 0.3075 & 0.5869 & 0.6113 & 0.2691 \\
        \Checkmark & \XSolidBrush  & 0.5033 & 0.6087 & 0.6183 & 0.6392 & 0.2949 \\
        \Checkmark & \Checkmark    & \textbf{0.5351} & \textbf{0.6816} & \textbf{0.7445} & \textbf{0.7557} & \textbf{0.4681} \\  \bottomrule
        \end{tabular}
    }
    \end{center}\vspace{-10pt}
\end{table}

\subsubsection{Ablation on Importance Assessment}

In this section, we perform an ablation study on the importance assessment method. The baseline for comparison assumes that the importance difference between adjacent sequence elements in the subset is consistent and equal to 1. We use $\mu$Fidelity as the evaluation metric to assess the rationality of the importance score assignment. As shown in Table~\ref{ablation_importance_assessment}, we found that when evaluating the importance of each element in the subset, the $\mu$Fidelity score improves relative to the baseline. This improvement demonstrates, to some extent, the validity of the importance scores assigned by this assessment strategy. Fig.~\ref{ablation_importance_assessment}A shows an attribution result, highlighting the search region at the inflection point of the Insertion curve, considered the most important. Fig.~\ref{ablation_importance_assessment}B and~\ref{ablation_importance_assessment}C illustrate different importance assessment strategies, showing that the proposed method helps users intuitively identify key regions.

\begin{table}[h]
    \caption{Ablation study on importance assessment for ImageBind on the ImageNet dataset.}
    \label{ablation_importance_assessment}
    \begin{center}
    \resizebox{0.48\textwidth}{!}{
        \begin{tabular}{l|ccc}
        \toprule
        \multicolumn{1}{c|}{\multirow{2}{*}{Strategies}} & \multicolumn{3}{c}{$\mu$Fidelity ($\uparrow$)}                         \\ 
        \multicolumn{1}{c|}{}                            & CLIP (ViT-L) & ImageBind (Huge) & LanguageBind (Large) \\ \midrule
        Consistent subset difference & 0.3208 & 0.3378   & 0.3889 \\
        \rowcolor{gray!20}
        Importance Assessment & \textbf{0.3334} & \textbf{0.3551} & \textbf{0.4096} \\ \bottomrule
        \end{tabular}
    }
    \end{center}
\end{table}

\begin{figure}[h]
    \centering  
    \includegraphics[width=0.48\textwidth]{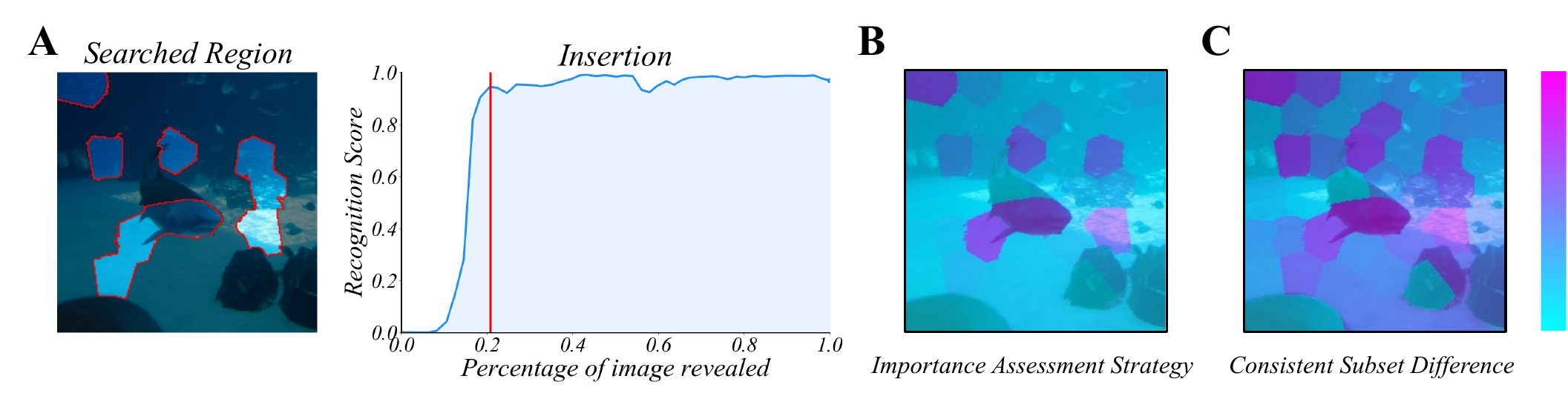} 
    \caption{An example of attribution results. \textbf{A.} Search region at the Insertion curve inflection point. \textbf{B.} Visualization using our strategy. \textbf{C.} Visualization using the baseline.}  
    \label{ablation:value}  
\end{figure}

\subsubsection{Ablation on Division Sub-region Number}

The sub-region division algorithm plays a key role in determining the quality of the search space elements. In addition to the choice of algorithm, the number of sub-regions, denoted as $|V|$, is also an important factor. A higher number of sub-regions results in a more fine-graininess division, but it also increases the attribution time. Therefore, this section examines the impact of the number of sub-regions on both performance and computation time. Similarly, we separately analyze the samples correctly and incorrectly predicted by the model.

\textbf{Impact on correctly predicting samples:} For samples correctly predicted by the model, the relationship between attribution performance and the number of sub-region divisions $|V|$ is shown in Table~\ref{ablation_on_size_imagenet}. We observed that as $|V|$ increases, attribution performance improves progressively, but the computation time rises significantly. Our method outperforms others when the number of sub-regions is set to 49. Future work should focus on improving the efficiency and accuracy of attribution when dealing with larger numbers of sub-regions.

\begin{table}[h]
    \caption{Ablation study on the effect of sub-region division number $|V|$ for ImageBind on the ImageNet dataset.}
    \label{ablation_on_size_imagenet}
    \begin{center}
        \resizebox{0.48\textwidth}{!}{
        \begin{tabular}{c|cccc}
        \toprule
        Sub-region Number $|V|$ & Deletion ($\downarrow$) & Insertion ($\uparrow$) & $\mu$Fidelity & Execution time (s) \\ \midrule
        49 & 0.1280 & 0.7463 & 0.3551 & 20 \\
        64 & 0.1230 & 0.7506 & 0.3473 & 39 \\
        77 & 0.1216 & \textbf{0.7532} & 0.3407 & 59 \\ 
        100 & \textbf{0.1199} & 0.7492 & 0.3291 & 158 \\
        \bottomrule
        \end{tabular}
        }
    \end{center}
\end{table}

\textbf{Impact on incorrectly predicting samples:} 
Next, for the samples the model predicted incorrectly, as shown in Table~\ref{ablation_on_size_error}, our results indicate that dividing the image into more fine-grained sub-regions leads to higher average confidence scores (0-100\%) and Insertion AUC scores. This further demonstrates that finer granularity in sub-region division results in better attribution performance.

\begin{table}[h]
    \caption{Ablation study on the effect of sub-region division number $|V|$ in incorrect sample attribution for ImageBind on the ImageNet dataset.}
    \label{ablation_on_size_error}
    \begin{center}
    \resizebox{0.48 \textwidth}{!}{
        \begin{tabular}{c|ccccc}
        \toprule
        \multirow{2}{*}{Sub-region Number $|V|$} & \multicolumn{4}{c}{Average highest confidence ($\uparrow$)} & \multirow{2}{*}{Insertion ($\uparrow$)} \\ 
        & (0-25\%)  & (0-50\%) & (0-75\%)  & (0-100\%) \\ \midrule 
        49  & 0.5351 & 0.6816 & 0.7445 & 0.7557 & 0.4681 \\
        64  & 0.5579 & 0.7100 & 0.7689 & 0.7803 & 0.4934 \\
        77  &  0.5731 & 0.7194 & 0.7765 & 0.7884 & 0.5004 \\
        100  & 0.5987 & 0.7418 & 0.7958 & 0.8100 & 0.5144 \\
        121  & \textbf{0.6167} & \textbf{0.7599} & \textbf{0.8094} & \textbf{0.8247} & \textbf{0.5285} \\
        \bottomrule
        \end{tabular}
    }
    \end{center}
\end{table}

%% file: revision_sec/7-appendix.tex
\section{\texorpdfstring{\revision{Coefficient Sensitivity Analysis}}{Coefficient Sensitivity Analysis}}
\label{supp:coefficient_sensitivity}

\revision{We evaluate the empirical sensitivity of \textsc{LiMA} to the four coefficients in its composite objective by varying one coefficient at a time around the default setting $(\lambda_1,\lambda_2,\lambda_3,\lambda_4)=(20,5,0.05,0.01)$. The study uses a randomly sampled subset of 100 inputs and covers ImageNet and VGG-Sound, image and audio inputs, and CLIP, ImageBind, and ViT backbones. Table~\ref{tab:supp_coefficient_sensitivity} reports the resulting Insertion and Deletion scores.}

\revision{The results reveal an empirical trade-off rather than a single setting that is optimal for every metric. Increasing $\lambda_1$ generally favors Insertion, while smaller values can improve Deletion. Removing $\lambda_2$ substantially worsens both metrics across the evaluated settings, whereas larger $\lambda_2$ generally favors Deletion with smaller changes in Insertion. The \mbox{examined} variations of $\lambda_3$ and $\lambda_4$ produce comparatively smaller changes. We therefore use the default coefficients as one fixed operating point that balances the evaluated Insertion--Deletion trade-off, rather than as a universally optimal configuration.}

\revision{The zero-weight rows are component-removal ablations and are not treated as alternative tuning points. Across all examined nonzero, one-at-a-time coefficient perturbations, the largest relative decrease in Insertion among the four evaluated settings is approximately $3.0\%$. Deletion varies more noticeably, reflecting its trade-off with Insertion, and several perturbations improve rather than degrade it. These results show that reasonable coefficient variation does not cause a collapse in faithfulness, while completely removing either principal criterion can substantially degrade performance.}

\begin{strip}
    \setlength{\abovecaptionskip}{6pt}
    \centering
    \captionof{table}{Parameter sensitivity ablation results on 100 samples.
    The baseline is $(\lambda_1,\lambda_2,\lambda_3,\lambda_4)
    =(20,5,0.05,0.01)$. Higher Insertion and lower Deletion are better.
    The best result in each column is bolded.}
    \label{tab:supp_coefficient_sensitivity}
    \scriptsize
    \setlength{\tabcolsep}{3pt}
    \renewcommand{\arraystretch}{0.94}
    \resizebox{0.9\textwidth}{!}{%
    \begin{tabular}{l*{4}{cc}}
    \toprule
    & \multicolumn{6}{c}{ImageNet}
    & \multicolumn{2}{c}{VGGSound} \\
    \cmidrule(lr){2-7}\cmidrule(lr){8-9}
    Setting
    & \multicolumn{2}{c}{CLIP}
    & \multicolumn{2}{c}{ImageBind}
    & \multicolumn{2}{c}{ViT}
    & \multicolumn{2}{c}{ImageBind} \\
    & Insertion$\uparrow$ & Deletion$\downarrow$
    & Insertion$\uparrow$ & Deletion$\downarrow$
    & Insertion$\uparrow$ & Deletion$\downarrow$
    & Insertion$\uparrow$ & Deletion$\downarrow$ \\
    \midrule
    Baseline
    & 0.7981 & 0.1203
    & 0.7352 & 0.1220
    & 0.7063 & 0.0851
    & 0.3958 & 0.0583 \\
    \addlinespace[1pt]
    $\lambda_1=0$
    & 0.6250 & 0.0975
    & 0.5689 & 0.1048
    & 0.5123 & \textbf{0.0621}
    & 0.3320 & 0.0558 \\
    $\lambda_1=5$
    & 0.7865 & \textbf{0.0958}
    & 0.7161 & 0.1053
    & 0.6856 & 0.0699
    & 0.3840 & 0.0550 \\
    $\lambda_1=10$
    & 0.7904 & 0.1068
    & 0.7259 & 0.1108
    & 0.7010 & 0.0755
    & 0.3884 & 0.0554 \\
    $\lambda_1=40$
    & 0.8059 & 0.1290
    & \textbf{0.7391} & 0.1376
    & \textbf{0.7195} & 0.0984
    & 0.4030 & 0.0585 \\
    \addlinespace[1pt]
    $\lambda_2=0$
    & 0.7181 & 0.2529
    & 0.6318 & 0.2438
    & 0.5245 & 0.2225
    & 0.2227 & 0.1311 \\
    $\lambda_2=2.5$
    & \textbf{0.8060} & 0.1318
    & 0.7351 & 0.1358
    & 0.7110 & 0.0997
    & 0.3991 & 0.0590 \\
    $\lambda_2=10$
    & 0.7911 & 0.1031
    & 0.7262 & 0.1135
    & 0.6988 & 0.0773
    & 0.3975 & 0.0555 \\
    $\lambda_2=20$
    & 0.7879 & 0.0980
    & 0.7275 & \textbf{0.1022}
    & 0.6912 & 0.0697
    & 0.4000 & \textbf{0.0545} \\
    \addlinespace[1pt]
    $\lambda_3=0$
    & 0.7989 & 0.1231
    & 0.7341 & 0.1244
    & 0.7160 & 0.0863
    & \textbf{0.4279} & 0.0584 \\
    $\lambda_3=0.025$
    & 0.7993 & 0.1187
    & 0.7366 & 0.1234
    & 0.7146 & 0.0873
    & 0.4032 & 0.0592 \\
    $\lambda_3=0.1$
    & 0.7984 & 0.1208
    & 0.7248 & 0.1234
    & 0.7146 & 0.0869
    & 0.3894 & 0.0565 \\
    \addlinespace[1pt]
    $\lambda_4=0$
    & 0.7986 & 0.1181
    & 0.7314 & 0.1223
    & 0.7058 & 0.0847
    & 0.3804 & 0.0582 \\
    $\lambda_4=0.005$
    & 0.7987 & 0.1199
    & 0.7318 & 0.1216
    & 0.7029 & 0.0851
    & 0.3962 & 0.0571 \\
    $\lambda_4=0.02$
    & 0.7981 & 0.1178
    & 0.7334 & 0.1218
    & 0.6991 & 0.0874
    & 0.3979 & 0.0583 \\
    \bottomrule
    \end{tabular}%
    }
\end{strip}

\section{\revision{Deployment Scope and Limitations}}\label{supp-supp:deployment_scope}

\revision{The absolute-cost results in the main manuscript position \textsc{LiMA} as a high-faithfulness, moderate-query attribution method rather than the universally cheapest method. With $N=49$, standard greedy ranking requires up to 2,500 queries for the multimodal objective, whereas the corresponding \textsc{LiMA} upper bound is 1,660 queries, a reduction of 33.6\%. This operating point retains the strong faithfulness reported in the main experiments while using substantially fewer queries than full standard-greedy ranking. Nevertheless, fixed-budget perturbation baselines such as RISE and HSIC-Attribution can use fewer queries, and gradient methods can be faster when white-box backward access is available.}

\revision{\textsc{LiMA} is therefore most appropriate when attribution faithfulness is prioritized over minimum latency, for example in offline model auditing, batch analysis, research evaluation, or local/open-weight model inspection. Its cost can be less suitable for latency-critical applications, high-price proprietary APIs, or clinical deployment with strict query or response-time constraints. Batching can reduce local wall-clock latency but does not necessarily reduce the number of model evaluations billed by an external provider.}

\label{supp:deployment_scope}

\section{Additional visualizations}

In this section, we present more interpretable visualizations of our results and compare them with those from other methods.

\subsection{Additional Visualization on Foundation Model}\label{supp:vis-multimodal}

Fig.~\ref{vis:additional-multimodal} presents more interpretable attribution results for three multimodal foundation models (CLIP~\cite{radford2021learning}, ImageBind~\cite{girdhar2023imagebind}, and LanguageBind~\cite{zhu2024languagebind}) with ViT backbone on the ImageNet dataset.

\begin{figure*}[!htbp]
    \centering
    \includegraphics[width = 0.95\textwidth]{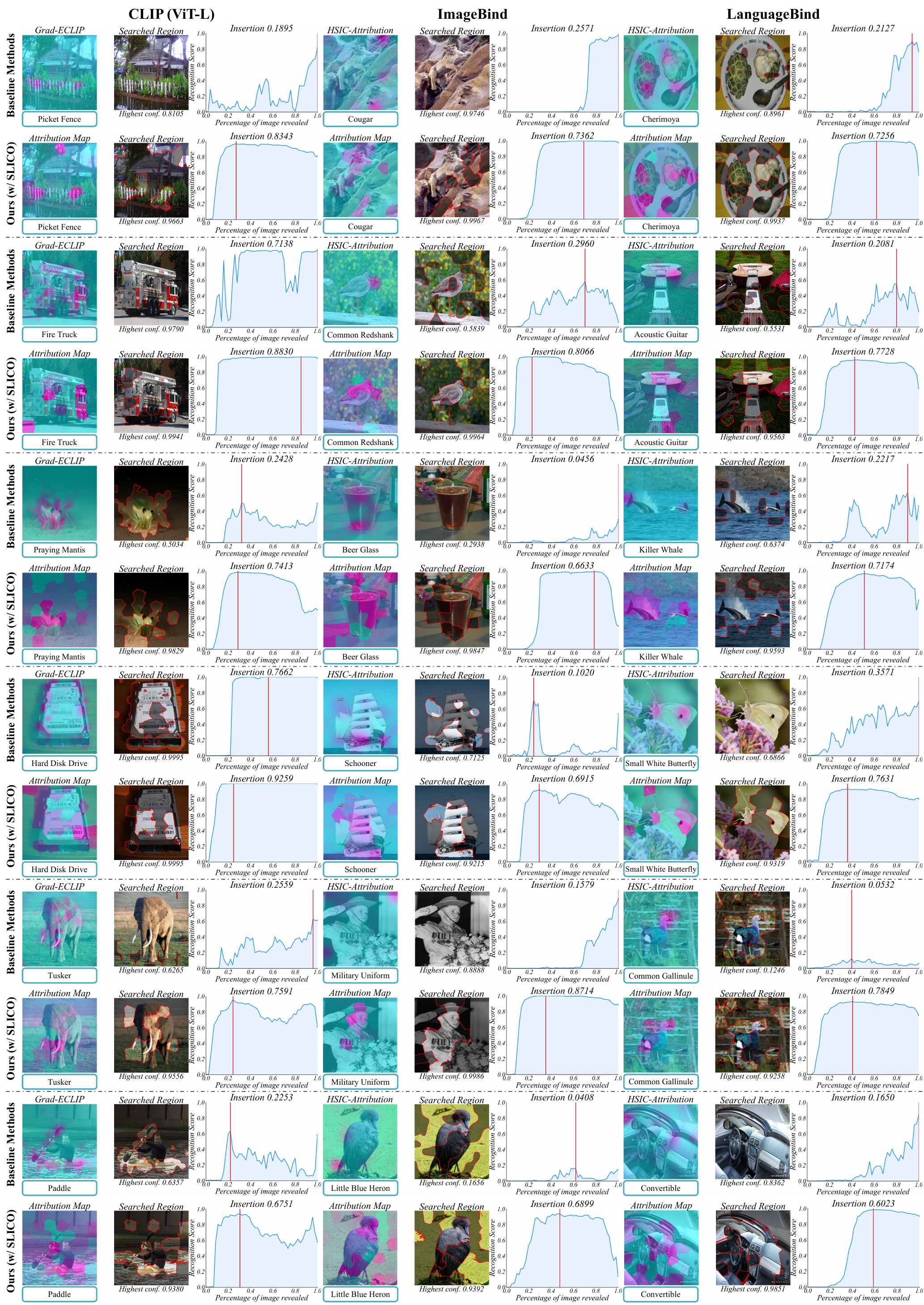}
    \caption{\revision{Additional attribution visualizations for different multimodal foundation models on the ImageNet dataset.}}
    \label{vis:additional-multimodal}
\end{figure*}

Fig.~\ref{vis:clip-resnet101} presents interpretable attribution results for CLIP (ResNet-101)~\cite{radford2021learning} on the ImageNet dataset. We compare with RISE~\cite{petsiuk2018rise} and HSIC-Attribution~\cite{novello2022making} methods, showing the advantages of our method.

\begin{figure*}[!htbp]
    \centering  
    \begin{overpic}[width=\textwidth,tics=8]{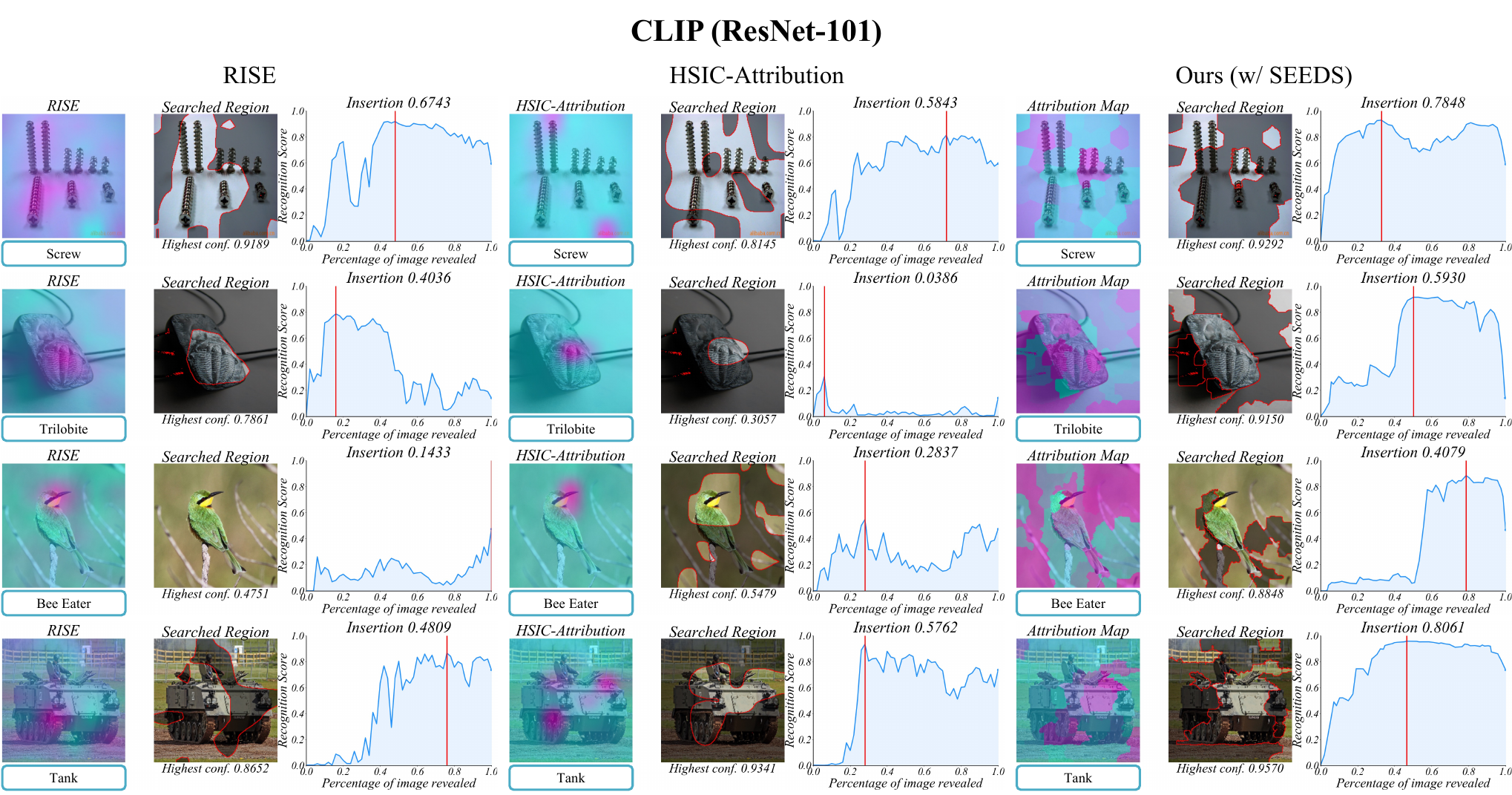} 
        \put (18.4,47.1) {\footnotesize \cite{petsiuk2018rise}}
        \put (56,47.1) {\footnotesize \cite{novello2022making}}        
    \end{overpic}  
    \caption{Additional interpretation visualization for CLIP (ResNet-101) on the ImageNet dataset.}  
    \label{vis:clip-resnet101}  
\end{figure*}

\begin{figure*}[!htbp]
    \centering  
    \begin{overpic}[width=\textwidth,tics=8]{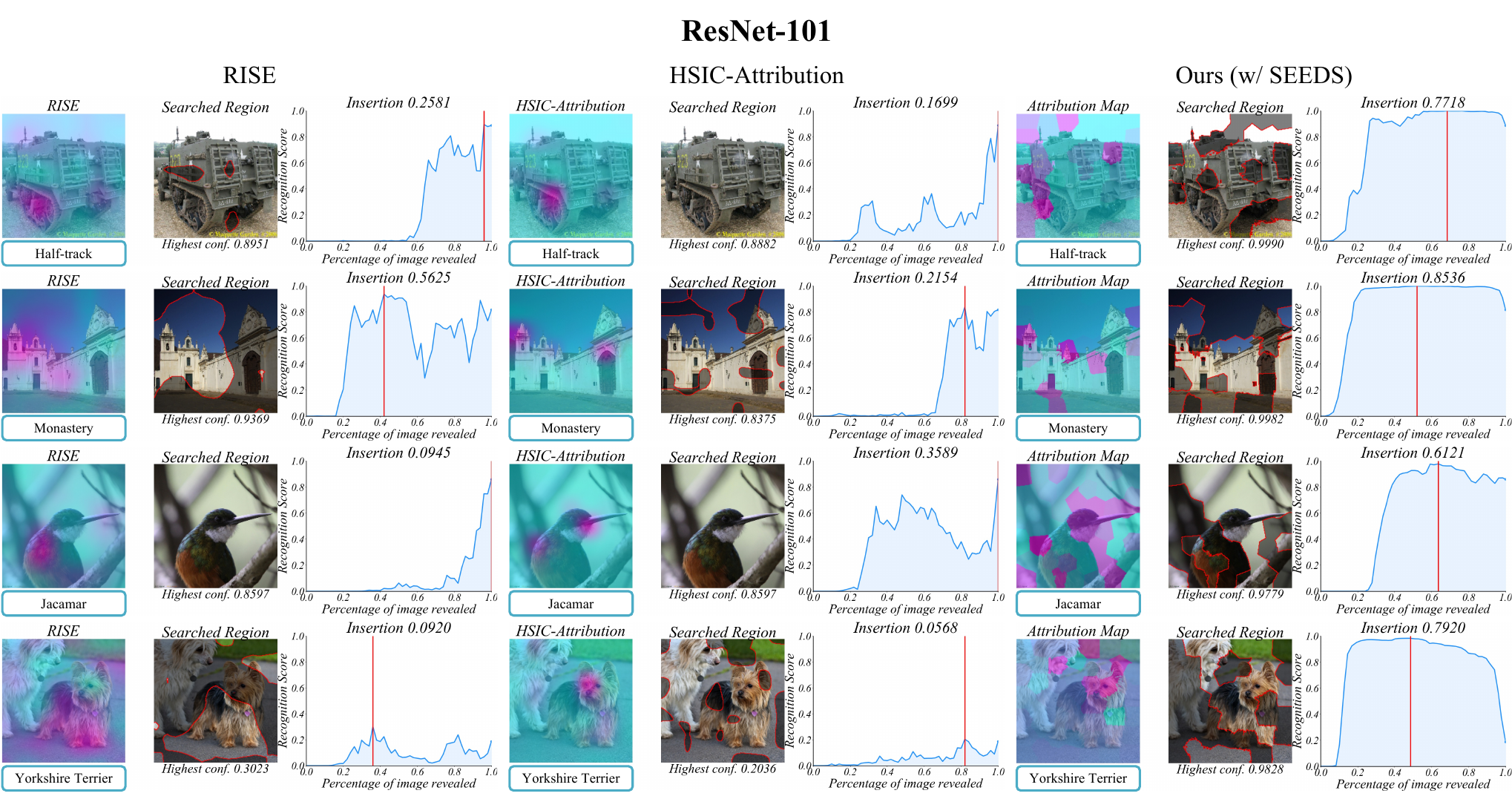} 
        \put (18.4,47.28) {\footnotesize \cite{petsiuk2018rise}}
        \put (56,47.2) {\footnotesize \cite{novello2022making}}        
    \end{overpic}
    \caption{Additional interpretation visualization for single-modal ResNet-101 on the ImageNet dataset.}  
    \label{vis:resnet101}  
\end{figure*}

\begin{figure*}[!htbp]
    \centering  
    \begin{overpic}[width=\textwidth,tics=8]{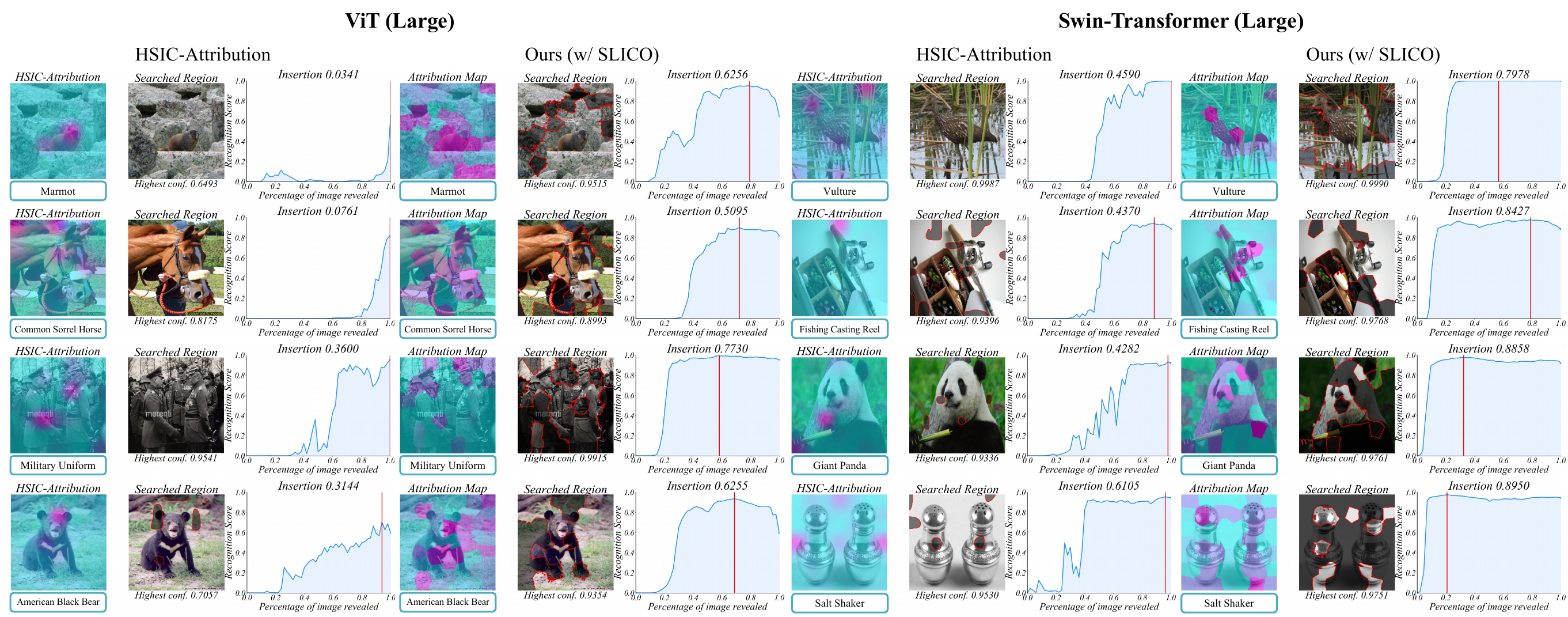} 
        \put (17.2,35.5) {\scriptsize \cite{novello2022making}}
        \put (67,35.5) {\scriptsize \cite{novello2022making}}        
    \end{overpic}
    \caption{Additional interpretation visualization for single-modal vision transformer (Large) and swin-transformer (Large) on the ImageNet dataset.}  
    \label{vis:vit}  
\end{figure*}

\begin{figure*}[!htbp]
    \centering  
    \begin{overpic}[width=\textwidth,tics=8]{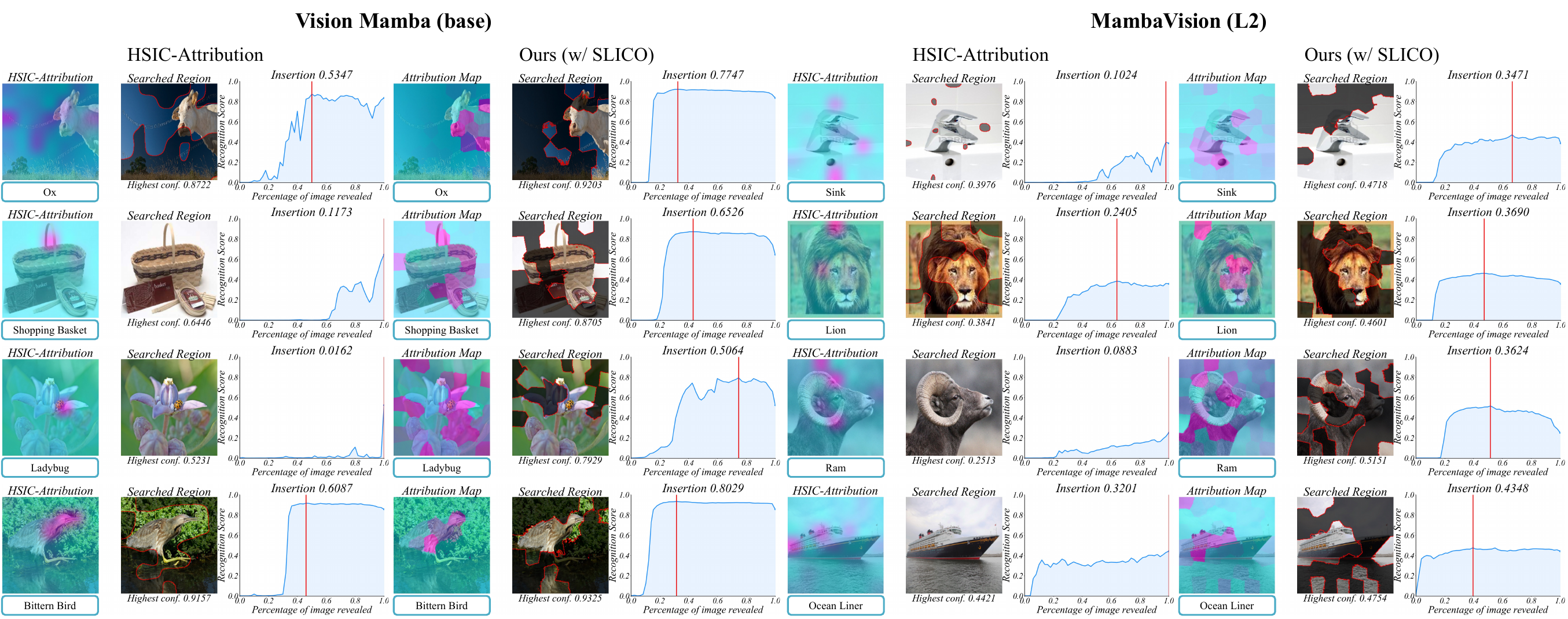} 
        \put (16.8,35.6) {\scriptsize \cite{novello2022making}}
        \put (67,35.6) {\scriptsize \cite{novello2022making}}        
    \end{overpic}
    \caption{Additional interpretation visualization for single-modal vision mamba (base) and mambavision (L2) on the ImageNet dataset.}  
    \label{vis:mamba}  
\end{figure*}

\subsection{Visualization on ResNet-101}\label{supp:vis-resnet}

Fig.~\ref{vis:resnet101} presents interpretable attribution results for single-modal ResNet-101~\cite{he2016deep} on the ImageNet dataset. We compare with RISE~\cite{petsiuk2018rise} and HSIC-Attribution~\cite{novello2022making} methods, showing the advantages of our method.

\subsection{Visualization on Vision Transformer}\label{supp:vis-vit}

Fig.~\ref{vis:vit} presents more interpretable attribution results for single-modal Vision Transformer and Swin Transformer~\cite{liu2021swin} models on the ImageNet dataset. We compare with HSIC-Attribution~\cite{novello2022making} methods, showing the advantages of our method.

\subsection{Visualization on Vision Mamba}\label{supp:vis-vim}

Fig.~\ref{vis:mamba} presents more interpretable attribution results for single-modal Vision Mamba (base)~\cite{zhu2024vision} and MambaVision (L2)~\cite{hatamizadeh2024mambavision} models on the ImageNet dataset. We compare with HSIC-Attribution~\cite{novello2022making} methods, showing the advantages of our method.

\section{Visualization of Model Mistake on CUB-200-2011}\label{cub_debug_vis}

We demonstrate the ability of our method to attribution images that are incorrectly predicted by the model on the CUB-200-2011 dataset. Fig.~\ref{vis:cub_vs_previous} shows the comparison between our method and HSIC-Attribution~\cite{novello2022making} and the previous version~\cite{chen2024less} on ResNet-101 model. It can be found that the current version of the method can have a flatter Insertion curve, which means that the attribution effect is better. In addition, we also show the attribution effects of different models on HSIC-Attribution and our method in Fig.~\ref{vis:cub_debug}. We include the sub-region division strategy based on superpixel segmentation and SAM.

\begin{figure*}[!htbp]
    \centering  
    \begin{overpic}[width=\textwidth,tics=8]{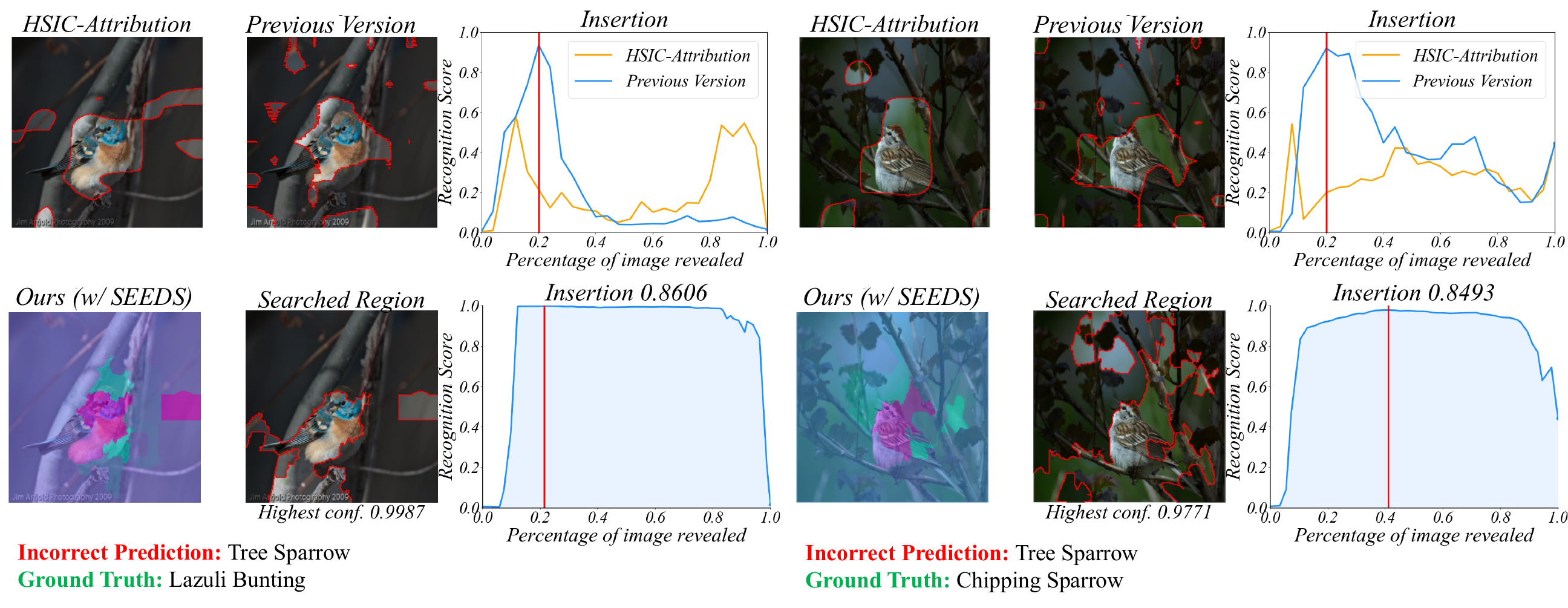} 
        \put (12,36.7) {\footnotesize \cite{novello2022making}}
        \put (26.4,36.7) {\footnotesize \cite{chen2024less}}
        \put (62.2,36.7) {\footnotesize \cite{novello2022making}}        
        \put (76.6,36.7) {\footnotesize \cite{chen2024less}}
    \end{overpic}
    \caption{\revision{Visualization of failure attribution on the CUB-200-2011 dataset. The first row shows the results of HSIC-Attribution and the previous version, and the second row shows the current method. The insertion curve reports the relationship between progressively revealed regions and the ground-truth response. The highlighted prefix provides model-level evidence associated with prediction reversal under this intervention; the visualization is not interpreted as formal causal identification.}}  
    \label{vis:cub_vs_previous}  
\end{figure*}

\begin{figure*}[!t]
    \centering  
    \begin{overpic}[width=\textwidth,tics=8]{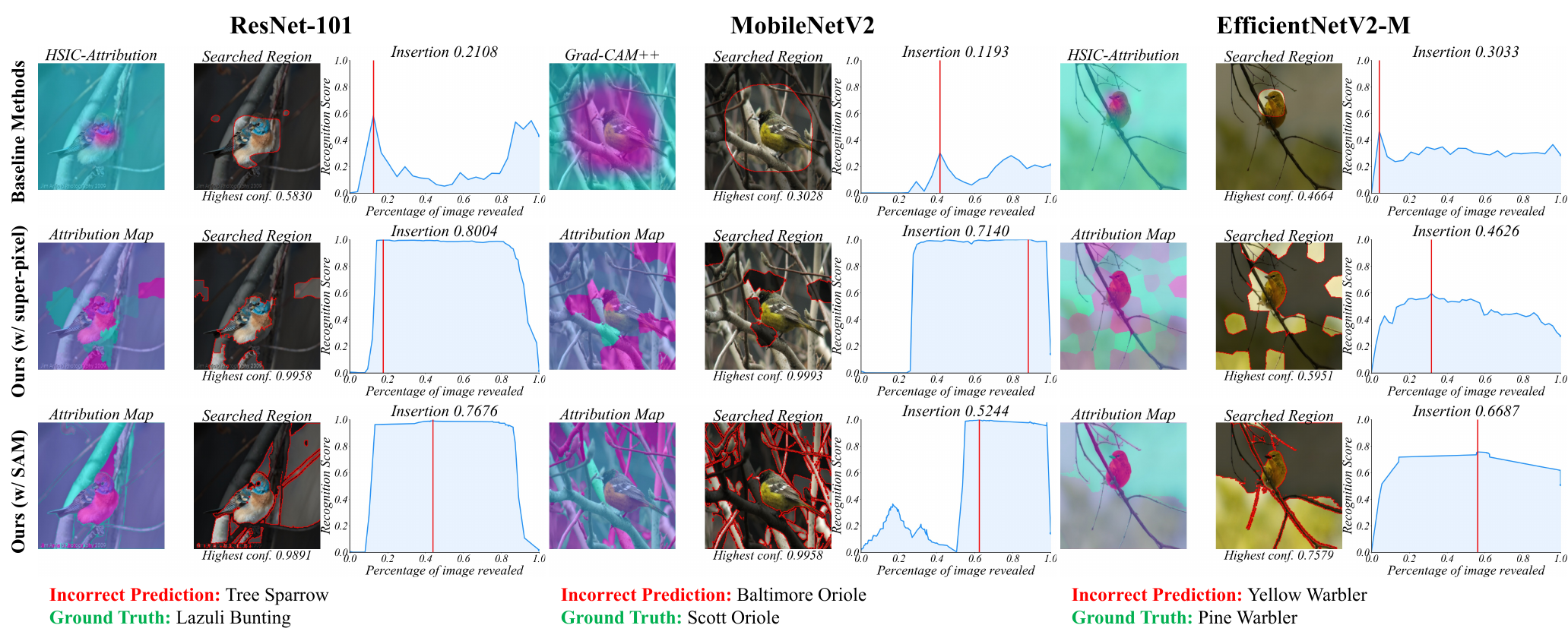} 
        \put (10,37.1) {\tiny \cite{novello2022making}}
        \put (42,37.1) {\tiny \cite{chattopadhay2018grad}}
        \put (75.2,37.1) {\tiny \cite{novello2022making}}
    \end{overpic}
    \caption{\revision{Visualization of failure attribution on the CUB-200-2011 dataset across different recognition models. The insertion curve reports the relationship between progressively revealed regions and the ground-truth response. The highlighted prefix provides model-level evidence associated with prediction reversal under this intervention; the visualization is not interpreted as formal causal identification.}}  
    \label{vis:cub_debug}  
\end{figure*}


\section{Sub-Region Division Algorithm based on Segment Anything Model}~\label{SAM_based_division}

Algorithm~\ref{alg:sam_division} demonstrates how to use Segment Anything (SAM)~\cite{kirillov2023segment} for sub-region division. As SAM produces overlapping sub-regions, we further split these overlapping regions to create more fine-grained sub-regions.

\begin{algorithm}[!t]
    \caption{Sub-region division algorithm based on Segment Anything Model~\cite{kirillov2023segment}}\label{alg:sam_division}
    \KwIn{Image $\mathbf{I} \in \mathbb{R}^{w \times h \times 3}$, Segment Anything Model $\texttt{SAM}(\cdot)$, delete threshold $\delta$.}
    \KwOut{An set $V$.}
    $V \gets \varnothing$ \Comment*[r]{Initiate the operation of sub-region division}
    $V_{M} = \texttt{SAM}(\mathbf{I})$\;
    \For{$i=1$ \KwTo $| V_{M} | - 1$}{
        \For{$j=i+1$ \KwTo $| V_{M} |$}{
            $\mathbf{M}_{\mathrm{inters.}} = V_{M}[i] \circ V_{M}[j]$\;
            \uIf {$\mathrm{sum}\left(\mathbf{M}_{\mathrm{inters.}}\right) \ne 0$}{
                \uIf {$\mathrm{sum}\left(V_{M}[i]\right) > \mathrm{sum}\left(V_{M}[j]\right)$}{
                    $V_{M}[i] = V_{M}[i] - \mathbf{M}_{\mathrm{inters.}}$\;
                }
                \uElse{
                    $V_{M}[j] = V_{M}[j] - \mathbf{M}_{\mathrm{inters.}}$\;
                }
            }
        }
    }
    \For{$i=1$ \KwTo $| V_{M} |$}{
        \uIf {$\mathrm{mean}\left( V_{M}[i] \right) > \delta$}{
            $V \gets V \cup \{V_{M}[i] \circ \mathbf{I} \}$\;
        }
    }
    $\mathbf{I}_{r} = \mathbf{I}_{r}$\Comment*[r]{A region that is not divided by $\texttt{SAM}(\cdot)$.}
    \For{$i=1$ \KwTo $| V_{M} |$}{
        $\mathbf{I}_{r} = \mathbf{I}_{r} - V_{M}[i]$\;
    }
    $V \gets V \cup \{\mathbf{I}_{r}\}$\;
    \Return $V$
\end{algorithm}

\section{\revision{Failure-Case Analysis of SAM-Based Sub-Regions}}
\label{supp:sam_failure_analysis}

\revision{We analyze the dependence of \textsc{LiMA} on the sub-region generator by holding the attribution objective and search procedure fixed and changing only the input partition between SLICO and SAM. Fig.~\ref{response:sam_failure_analysis} presents two complementary groups of diagnostic examples. For each method, the dashed vertical line marks the searched prefix shown beside its insertion curve, and the percentage below the searched-region image reports the revealed image area.}

\revision{Fig.~\ref{response:sam_failure_analysis}A contains cases in which SLICO yields a more compact and effective prefix. For the ruddy turnstone, Australian Kelpie, and marmot examples, SLICO reveals 26.4\%, 25.9\%, and 12.4\% of the image, whereas SAM reveals 83.6\%, 52.4\%, and 63.4\%, respectively. The mean revealed areas are therefore 21.6\% for SLICO and 66.5\% for SAM. In these examples, SAM produces an oversized region that includes the target together with surrounding content or combines several portions of the same semantic object. Semantic coherence does not guarantee \emph{decision homogeneity}: pixels within one semantic region may have different, or even opposing, effects on the model response. Because each generated region is an atomic candidate in the search, useful and detrimental evidence inside one large mask cannot be ranked independently. This explains both the delayed response and the non-monotonic portions of some SAM insertion curves. The overlap-resolution procedure in Algorithm~\ref{alg:sam_division} makes candidates disjoint, but it cannot guarantee that the resulting masks are sufficiently small or decision homogeneous.}

\revision{Fig.~\ref{response:sam_failure_analysis}B contains the complementary cases in which SAM provides a better partition. For the toaster, green mamba, and mobile phone examples, the SAM prefixes reveal 5.5\%, 29.2\%, and 19.7\%, whereas the SLICO prefixes reveal 96.0\%, 55.9\%, and 32.1\%. The corresponding means are 18.1\% and 61.3\%. Here the SAM masks more precisely follow semantic boundaries or isolate discriminative object parts. SLICO instead fragments the relevant semantic evidence across multiple low-level cells, so a larger collection of candidates must be accumulated before the model response rises. SAM therefore permits a much earlier response when its masks are both fine-grained and aligned with the model's discriminative evidence.}

\revision{These examples show that neither SAM nor SLICO uniformly dominates. The decisive property is whether the generated candidates are sufficiently fine and \emph{decision aligned} for the evaluated model, rather than semantic plausibility, boundary quality, or region size alone. Practical safeguards include removing near-duplicate masks, imposing a maximum mask-area threshold, subdividing oversized SAM masks with SLICO, or constructing a hybrid candidate pool and selecting its configuration on held-out data. The six examples are intended to diagnose the two mechanisms, not to estimate their prevalence over the full dataset.}

\begin{figure*}[!t]
    \centering
    \includegraphics[width=0.92\textwidth]{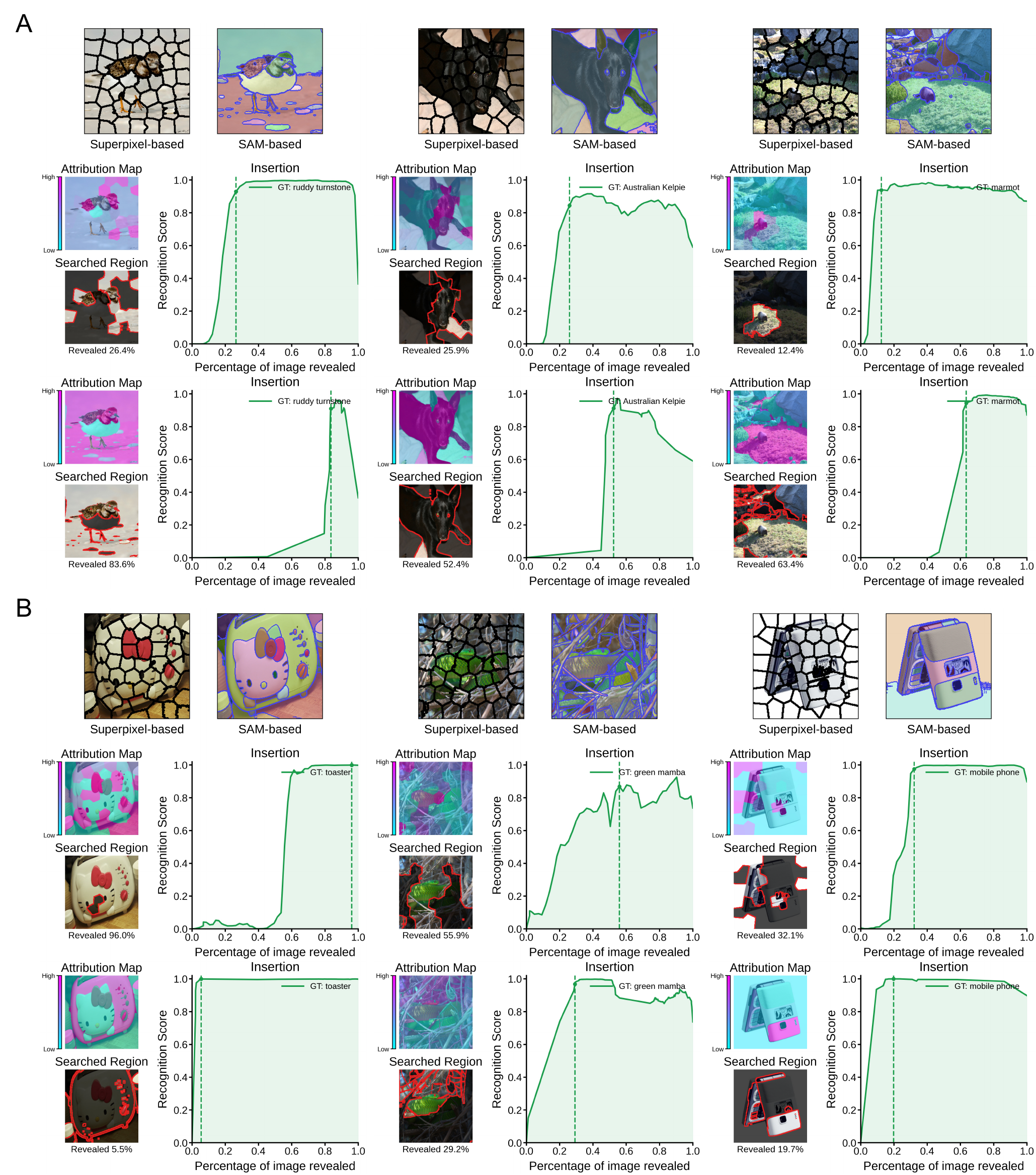}
    \caption{Controlled analysis of SLICO- and SAM-based sub-regions. \textbf{A.} oversized or decision-heterogeneous SAM masks delay recognition. \textbf{B.} boundary-aligned SAM masks isolate the object or its discriminative parts more effectively than low-level superpixels.}
    \label{response:sam_failure_analysis}
\end{figure*}